\documentclass{article}
\pdfoutput=1
\usepackage{url}
\usepackage{algorithm,algorithmic}
\usepackage{amssymb}
\usepackage{bbm}
\usepackage{mathtools}
\usepackage{caption}
\usepackage{subcaption}
\usepackage{placeins}
\usepackage{fancyhdr,amsmath,amsthm,amssymb,bm,url,enumerate,float, bbm} 
\usepackage{lastpage} 
\usepackage{extramarks} 
\usepackage{graphicx,caption} 

\usepackage{pgfplots}
\usepackage{pgfplotstable}
\usepackage{xcolor}
\usepackage{hyperref}
\usepackage{newfloat}

\usepackage{arxiv}
\usepackage[utf8]{inputenc} 
\usepackage[T1]{fontenc}    
\usepackage{booktabs}       
\usepackage{amsfonts}       
\usepackage{natbib}
\usepackage{doi}
\usepackage{caption} 
\newtheorem{lemma}{Lemma}

\newcommand\norm[1]{\left\lVert#1\right\rVert} 
\newcommand{\given}{\,\vert\,} 

\newtheorem{prop}{Proposition}
\newtheorem{corollary}{Case}[prop] 

\DeclareFloatingEnvironment[name={Supplementary Figure},fileext=lof]{suppfigure}

\pgfplotsset{compat=1.17}
\begin{document}

\title{
Optimal Ensemble Construction \\ for Multi-Study Prediction with Applications to COVID-19 Excess Mortality Estimation
}

\author{Gabriel Loewinger,$^{\dagger,1}$  Rolando Acosta Nunez,$^{\dagger}$ Rahul Mazumder$^{\ddag}$  and Giovanni Parmigiani$^{\dagger, \ast}$ \\[4pt]
$^{\dagger}$\textit{Department of Biostatistics, Harvard School of Public Health, Boston, MA} \\[2pt]
$^{\ddag}$\textit{MIT Sloan School of Management, Operations Research Center and MIT Center for Statistics, Cambridge, MA}
\\[2pt]
$^{\ast}$\textit{Department of Data Science,
Dana Farber Cancer Institute, Boston, MA}}

\renewcommand{\shorttitle}{Optimal Ensemble Construction}

\maketitle
\footnotetext[1]{To whom correspondence should be addressed. \textit{gloewinger@gmail.com}}

\hypersetup{
pdfkeywords={First keyword, Second keyword, More},
}
\begin{abstract}{
It is increasingly common to encounter prediction tasks in the biomedical sciences for which multiple datasets are available for model training. Common approaches such as pooling datasets and applying standard statistical learning methods can result in poor out-of-study prediction performance when datasets are heterogeneous. Theoretical and applied work has shown \textit{multi-study ensembling} to be a viable alternative that leverages the variability across datasets in a manner that promotes model generalizability. Multi-study ensembling uses a two-stage \textit{stacking} strategy which fits study-specific models and estimates ensemble weights separately. This approach ignores, however, the ensemble properties at the model-fitting stage, potentially resulting in a loss of efficiency. We therefore propose \textit{optimal ensemble construction}, an \textit{all-in-one} approach to multi-study stacking whereby we jointly estimate ensemble weights as well as parameters associated with each study-specific model via a unified optimization formulation. We establish that limiting cases of our approach yield existing methods such as multi-study stacking and pooling datasets before model fitting. We propose an efficient block coordinate descent algorithm to optimize the proposed loss function. 
We compare our approach to standard methods by applying it to a multi-country COVID-19 dataset for baseline mortality prediction. We show that when little data is available for a country before the onset of the pandemic, leveraging data from other countries can substantially improve prediction accuracy. Importantly, our approach outperforms multi-study stacking and other standard methods in this application. We further characterize the method's performance in data-driven simulations and other numerical experiments. Our method remains competitive with or outperforms multi-study stacking and other earlier methods across a range of between-study heterogeneity.
}
\end{abstract}


\newpage
\section{Introduction}
\label{sec1}

\subsection{Background}
Statistical methods to leverage information from different studies and sources are critical for generating prediction algorithms that are replicable across populations and experimental settings. The increasing availability of biomedical databases in, for example, neuroimaging (OpenNeuro \citep{openneuro}), genetics (e.g., GWAS Catalogue \citep{gwas}), HIV (Stanford HIV database \citep{Ramon}) and cancer genomics \citep{ovarian, geo} has resulted in increased interest in the utilization of data from multiple sources when training prediction algorithms. This is further motivated by the observation that prediction performance assessed within-study (e.g., with cross validation) is often far more optimistic than the performance evaluated out-of-study \citep{Ma, Chang, Bernau}. 
Despite promising initial progress, much remains to be investigated on how to combine information across datasets or ``studies'' in a manner that promotes model generalizability. A simple, but at times powerful approach is to merge datasets before model fitting. This can fail to capture, however, the heterogeneity across datasets and can therefore perform poorly when studies are heterogeneous. In this case different models for each study may lead to better fits, and their ensemble to better out-of-study predictions \citep{Patil, Guan}. To account for the heterogeneity, one approach is to pre-process the studies with data harmonization techniques before model fitting \citep{combatMultiCenter, ZhangY}. While this is generally helpful, it can be insufficient, and potentially remove genuine signal \citep{Nygaard, combatLimits}. 

\subsection{Related Literature}\label{subs:previous}
Multiple related sub-fields of transfer learning \citep{DA_TL_review} contain a rich body of methodologies geared towards the multi-study problem specifically as well as the ``dataset shift'' issue more generally. Broadly speaking, these methods aim to deal with variation in the distribution of the training and test dataset(s). ``Concept shift'' refers to changes in the conditional probability of labels given predictors, while ``covariate shift,'' describes discrepancies in the joint distribution of the covariates between training and test data. ``Hybrid shift'' describes the case when both concept and covariate shift occur \citep{DA_TL_review, Yang}.

Our motivating application can be described as a multi-source transfer learning problem \citep{multiSourceTL, DA_TL_review}. These describe settings where the modeler has a small sample of labeled training data drawn from the distribution of interest, but seeks to enhance prediction performance through borrowing labeled data from other related studies, or ``domains.'' Typically these other auxiliary data sources share common covariates and outcomes. The distribution of these data are thought to be similar to the distribution of the data of the target study. This is closely related to multi-source domain adaptation which proposes strategies that use information from multiple datasets, but also leverage the covariates of the test set to tailor the model to the target domain \citep{DA_TL_review, domainAdptReview}. The domain generalization literature focuses on leveraging multiple datasets in model training to improve model generalizability, to enhance prediction performance on a new, unknown, but related, ``domain'' \citep{domGen_review}. Our methods are inspired both by this vast literature in transfer learning as well as related methods in ``multi-study'' statistics that aim to draw upon multiple data sources in inference \citep{jordan, ibrahim2020}, supervised prediction \citep{Loewinger, Guan, Ramchandran, ren}, and unsupervised learning \citep{multiStudyFactor, factorAnaly}. 

Previous work in this area proposed multi-study ensembling in conjunction with a generalization of stacking, an ensemble weight estimation method \citep{Breiman2}, as a flexible strategy to aggregate information from different studies \citep{Patil, Ramchandran, Loewinger, ren, Guan}. The approach involves two separate stages: A) training one or more models on each study separately, and B) constructing an ensemble prediction rule that is a weighted average of the predictions from each of the study-specific models. The ensemble weights are estimated in step B through ``multi-study stacking,'' (MSS) by regressing the outcome of all training studies against the predictions of each of these models, using all available training data sets. Heuristically, this method rewards cross-study prediction performance by determining the ensemble weight associated with a model fit on one study based upon how well that model predicts across all the training data sets. 

To formalize, we assume we have $K$ training studies and we denote the outcome vector of training study $k$ as $\mathbf{y}_k \in \mathbb{R}^{n_k}$, and its design matrix as $\mathbb{X}_k \in \mathbb{R}^{n_k \times (p + 1)}$. In addition to~$p$ predictors, $\mathbb{X}_k$ includes a column of ones for an intercept. We let $N = \sum_{k=1}^K n_k$. 
Multi-study stacking as originally proposed consists of two stages. In stage~A) we separately fit models $\mathbf{\hat{Y}}_k (\cdot)$ on data from the $k^{th}$ study, for every $k$. In stage~B) we compute predictions made by these models at the observed covariates of study $k'$ as $\mathbf{\hat{Y}}_k (\mathbb{X}_{k'})$ for every $k'$. Define $\mathbf{y}$ and $\hat{\mathbb{X}}$ as follows:
\[ \mathbf{y} = \begin{bmatrix}
\mathbf{y}_{1}\\
\vdots \\
\mathbf{y}_{K}\\
  \end{bmatrix}_{N \times 1}~~~~~~~~
  \mathbb{\hat{X}} = \begin{bmatrix}
{\mathbf{\hat{Y}}_1} (\mathbb{X}_{1}) & {\mathbf{\hat{Y}}_2} (\mathbb{X}_{1}) &  ... & {\mathbf{\hat{Y}}_K} (\mathbb{X}_{1})\\
\vdots & \vdots  &\ddots  & \vdots\\
\mathbf{\hat{Y}}_1 (\mathbb{X}_{K}) & {\mathbf{\hat{Y}}_2} (\mathbb{X}_{K}) &  ... & {\mathbf{\hat{Y}}_K} (\mathbb{X}_{K})\\
  \end{bmatrix}_{N \times K} \]
We then fit a stacking regression to predict $\mathbf{y}$ given $\hat{\mathbb{X}}$ using, for example, non-negative least squares (NNLS) or Ridge regression, typically with an intercept term.
The resulting coefficients serve as the weights of the ensemble prediction and are denoted by the $p+1$ dimensional vector~$\hat{\boldsymbol{w}}$. At data point $\mathbf{x}$ the prediction rule is $\mathbf{\hat{Y}}_{Stack} \left ( \mathbf{x} \right ) = \hat{w_0} + \sum_{k=1}^K \hat{w_k} \mathbf{\hat{Y}}_k \left ( \mathbf{x} \right)$.
    
Theoretical work has described settings in which this type of ensembling is expected to outperform merging data sets before model fitting here labeled ``Trained-on-the-Merged dataset'' (ToM). For example \cite{Guan} compare the ToM approach to ensembling when the learner is either linear or Ridge regression, under a linear mixed effects data generating model where true regression coefficients vary with $k$. They show that when between-study heterogeneity, as measured by the variance of the random regression coefficients, is sufficiently high, ensembling is expected to outperform~ToM. 

Subsequent work generalized multi-study ensembling and formalized two approaches: {\it generalist stacking} (described above) for domain generalization settings, in which we aim to train models that generalize to an unseen study, and {\it specialist stacking} for multi-source transfer learning settings, in which we aim to tailor the ensemble to predict well on a target study for which labeled data is available at the time of model training \citep{ren}. For example, say we wish to engineer the ensemble to predict well on future data that follows the same distribution as data from training study $k^*$, where $k^* \in [K]$ and $[K]$ denotes the set $\{1,2,...,K\}$. This can be accomplished by regressing $\mathbf{y}_{k^*}$ onto $\hat{\mathbb{X}}_{k^*}$ where $\hat{\mathbb{X}}_{k^*} = \begin{bmatrix}
{\mathbf{\hat{Y}}_1} (\mathbb{X}_{k^*}) & {\mathbf{\hat{Y}}_2} (\mathbb{X}_{k^*}) &  ... & {\mathbf{\hat{Y}}_K} (\mathbb{X}_{k^*})
  \end{bmatrix}_{n_{k^*} \times K}$. By setting up the stacking regression in this manner, specialist stacking estimates ensemble weights appropriate for data similar to training set $k^*$. Similar architecture for stacking in transfer learning settings has been proposed previously \citep{stacktrans}. 
  
\cite{ren} further proposed a ``no data reuse'' (NDR) specialist stacking procedure that excludes data used in stage~A from the stage~B stacking. Briefly, this is accomplished by dividing training data for the target study into $L$ folds. Taking $L=2$ for ease of explanation, let the data in the two folds be $\{\mathbf{y}_{k^*}^1, \mathbb{X}_{k^*}^1\}$ and $\{\mathbf{y}_{k^*}^2, \mathbb{X}_{k^*}^2\}$. Then train model $\mathbf{\hat{Y}}_{k^*}^1()$ on $\{\mathbf{y}_{k^*}^1, \mathbb{X}_{k^*}^1\}$ and a separate model, $\mathbf{\hat{Y}}_{k^*}^2()$, on $\{\mathbf{y}_{k^*}^2, \mathbb{X}_{k^*}^2\}$. Then when fitting the stacking regression, data reuse is avoided by using $\mathbf{\hat{Y}}_{k^*}^2()$ to make predictions on $\{\mathbf{y}_{k^*}^1, \mathbb{X}_{k^*}^1\}$ and vice-versa. Specifically, the stacking procedure involves regressing $\mathbf{y}_{k^*}$ onto $\hat{\mathbb{X}}_{k^*}$ where \[ \mathbf{y}_{k^*} = \begin{bmatrix}
\mathbf{y}_{k^*}^1\\
\mathbf{y}_{k^*}^2\\
\end{bmatrix}_{n_{k^*} \times 1}~~~~~~~~
  \mathbb{\hat{X}}_{k^*} = \begin{bmatrix}
{\mathbf{\hat{Y}}_1} (\mathbb{X}_{k^*}^1) & {\mathbf{\hat{Y}}_2} (\mathbb{X}_{k^*}^1) & ... & {\mathbf{\hat{Y}}_{k^*}}^2 (\mathbb{X}_{k^*}^1) & ... & {\mathbf{\hat{Y}}_K} (\mathbb{X}_{k^*}^1)\\
{\mathbf{\hat{Y}}_1} (\mathbb{X}_{k^*}^2) & {\mathbf{\hat{Y}}_2}
(\mathbb{X}_{k^*}^2) & ... &  {\mathbf{\hat{Y}}_{k^*}}^1 (\mathbb{X}_{k^*}^2) & ... & {\mathbf{\hat{Y}}_K} (\mathbb{X}_{k^*}^2)\\
  \end{bmatrix}_{n_{k^*} \times K} . \]
In practice this is implemented with more than two folds. After the estimation of ensemble weights, a third stage is implemented in which a final model for study $k^*$, $\mathbf{{Y}}_{k^*}()$, is fit using $\{ \mathbf{y}_{k^*}, \mathbb{X}_{k^*} \}$ (i.e., all available data for that study). This model is then used in the final prediction rule, $\mathbf{\hat{Y}}_{Stack} \left ( \mathbf{x} \right ) = \hat{w}_0 + \sum_{k=1}^K \hat{w}_k \mathbf{\hat{Y}}_k \left ( \mathbf{x} \right)$

In this paper we propose a simpler version of no data reuse specialist stacking, which does not require splitting the data into folds, and is amenable to a simple optimization formulation. In stage A), this algorithm trains single-study learners on all studies except the target study $k^*$. In stage B), it only uses the target training data $\{\mathbf{y}_{k^*}, \mathbb{X}_{k^*}\}$ by regressing $\mathbf{y}_{k^*}$ onto the \[\begin{bmatrix}
{\mathbf{\hat{Y}}_1} (\mathbb{X}_{k^*}) &  ... & \mathbf{\hat{Y}}_{k^* - 1} (\mathbb{X}_{k^*}) & \mathbf{\hat{Y}}_{k^* + 1} (\mathbb{X}_{k^*}) & ... & {\mathbf{\hat{Y}}_K} (\mathbb{X}_{k^*})
\end{bmatrix}_{n_{k^*} \times {K-1}} \] matrix. Stage A) builds an ensemble that captures the heterogeneity of data distributions. Stage B) specializes the weights of the ensemble to perform well on the target study. 
We refer to this variant as "no-data-reuse specialist stacking" throughout the manuscript.
 
\subsection{Optimal Ensemble Construction: \\ A general framework for simultaneous model parameter and model weight estimation}

The multi-study ensembling procedures described above are flexible and easy to implement, but one potential shortcoming is that ensemble weights and study-specific model parameters are estimated separately in each stage. These approaches may therefore lose efficiency because the ensemble properties are ignored in stage A). The main thrust of this work is to integrate both stages into a single and explicit optimization framework.

We propose a method based on multi-study ensembling and stacking that seeks to improve prediction performance by \emph{jointly} optimizing model parameters and ensemble weights. We term this ``all-in-one'' approach as \textit{optimal ensemble construction} (OEC). We formalize this below as a convex combination of the sum of study-specific model losses (e.g., study-specific mean square errors) and the stacking regression loss function. 
As an initial foray into this area, our work investigates linear models. We show that in this case the OEC yields standard stacking and the ToM as special cases of our framework. 
While we focus on linear models in this work, we note that our framework is general and may be applied to other predictive models as well. 

\subsection{Motivation: Counterfactual COVID-19 mortality estimation}
Our work was initially motivated by the estimation of COVID-19 related mortality. International comparisons of COVID-19 specific death counts suffer from important limitations. For example, due to health systems heterogeneity, the criteria for what constitutes a COVID-19 death is different across countries \citep{karanikolos2020comparable}. In part, these criteria depend on the testing capacity of the jurisdiction, which can be sub-optimal in underdeveloped countries \citep{owidcoronavirus}. Hence, excess mortality is regarded as the gold-standard method of estimating the direct and indirect effects of the COVID-19 pandemic \citep{goldstandard} and is calculated by subtracting the number of deaths that would have been expected to occur in a locality (e.g., based on historical trends) from the number of observed deaths.

In many countries, infrastructure to monitor mortality before and during a disaster is not yet fully established \citep{Karlinsky, fottrell2009dying}. Such countries may have limited data to estimate baseline mortality before the onset of a disaster, such as a pandemic, and hence the resulting estimates of excess mortality may be unreliable. For instance, \citet{Karlinsky} contacted officials in a host of countries to collect international mortality data to produce a database for the comparison and assessment of public health strategies. Many officials indicated that no data or very limited data was available prior to the onset of COVID-19. Although only a minority of nations world-wide provided any public data, \citet{Karlinsky} quoted responses from health officials in countries such as Argentina, China, India, Liberia and Vietnam to emphasize that many countries do not have the resources to collect sufficient mortality data. The following quote from a Liberian health official speaks volumes about this ongoing challenge~\citep{Karlinsky}: 
\begin{quote}
\emph{``Unfortunately, we do not have a mechanism in place at the moment to capture routine mortality data in-country nation-wide...As you may also be aware, death or mortality registration or
reporting is yet a huge challenge in developing countries...''}
\end{quote}
Even among countries that collected pre-pandemic data, samples were often limited to as little as one year of pre-pandemic data (e.g., South Africa) \citep{Karlinsky,financial}. Even further, some countries only provided quarterly or monthly data (e.g., Iran, Taiwan) which may not provide the temporal resolution needed for precise assessment of public health responses.

This motivates the implementation of multi-study methods to borrow information across countries to create better baseline mortality estimates. Leveraging information from other countries' data may improve baseline mortality estimates on the target country because shared latent factors (e.g., geographic, health systems, weather patterns) may have similar impacts on mortality trends across countries. While these factors may not be comparable across all observed countries, there may exist clusters of countries across which borrowing information is useful. Identifying $a~priori$ which countries to borrow information from may not be reliable, as even natural heuristics such as geographic proximity may be unreliable. Thus we sought to develop methods that could identify which countries would provide auxiliary data that could improve the estimation of baseline mortality for a target country in a data-driven fashion. To this end, we compare the prediction capabilities of the proposed methods against conventional approaches by using a multi-country dataset of weekly death counts for 38 countries (\textit{i.e.}, $K=38$ studies) \citep{hmd}, where we aim to predict baseline mortality in a country of interest.

\section{Methods}
\label{sec2}
    
\subsection{Notation and Problem Statement}

We next build on the notation introduced earlier and specify distributional assumptions. We assume we have $K$ observed studies available for model training. The pair of outcomes and covariates for study $k$ is $\{ \mathbf{y}_k, \mathbb{X}_k \}$.
The merged dataset has outcome variable $\mathbf{y} \in \mathbb{R}^N$ and design matrix $\mathbb{X} \in \mathbb{R}^{N \times (p+1)}$ where, 
$N = \sum_k n_k$.
We assume that the outcome ${y}_{k,i}$ for observation $i$ in study $k$ is conditionally distributed as $y_{k,i} \given \mathbf{x}_{k,i} \sim f_{y_k \given x_k}(\mathbf{y}_k\given\mathbb{X}_k)$, and that marginally the covariates $\mathbf{x}_{k,i}$ are distributed as $\mathbf{x}_{k,i} \sim f_{x_k}(\mathbb{X}_k)$. We assume that the conditional distributions of the outcome given the covariates may differ across studies. We explore both settings where the marginal distributions of the covariates, $f_{x_k}(\mathbb{X}_k)$, differ across studies, and where they are approximately the same. 

In the two-stage multi-study stacking method discussed in Section~\ref{subs:previous}, we distinguish between the model {\it parameters} (e.g., the coefficients in a linear model) and the ensemble {\it weights} used to aggregate predictions from different models. We refer to statistical learning methods (e.g., linear models, as this is our focus here) to be fit on a single study as single study learners (SSL). We denote the model parameters associated with a fit to the $k^{th}$ study as ${\boldsymbol{\gamma}}_k\in \mathbb{R}^{p+1}$. For example, in the squared $\ell_2$ penalized linear regression setting, $\hat{{\boldsymbol{\gamma}}}_k$ is the solution to the following
\begin{equation}
    \min_{\boldsymbol{\gamma}_k} \norm{ \mathbf{y}_k - \mathbb{X}_k \boldsymbol{\gamma}_k}^2_2 + \lambda \norm{\mathbb{D}_{p+1} \boldsymbol{\gamma}_k}_2^2,
     \label{eq:ssm}
\end{equation}
where $\mathbb{D}_{r} \in \mathbb{R}^{r \times r}$ is a diagonal matrix with the (1,1)-th location equal to zero and all other diagonal entries equal to one to prevent regularization of the intercept. 

We denote the estimates of model parameters obtained from a fit to the merged dataset as $\hat{{\boldsymbol{\gamma}}}^{\text{ToM}}$. For example, in the squared $\ell_2$ penalized linear regression setting, ${\boldsymbol{\gamma}}^{\text{ToM}}$ is a solution to:
\begin{equation}
    \min_{\boldsymbol{\gamma}^{\text{ToM}}} \norm{ \mathbf{y} - \mathbb{X} \boldsymbol{\gamma}^{\text{ToM}}}^2_2 + \lambda \norm{\mathbb{D}_{p+1} \boldsymbol{\gamma}^{\text{ToM}}}_2^2 .
     \label{eq:tom}
\end{equation}
We refer to variations of two-stage multi-study methods as ``generalist'' or ``specialist'', depending on the goal. For specialist methods, we further indicate whether we used a ``no data reuse'' variation.
We abbreviate the multi-study stacking generalist as MSS$^{\text{G}}$, the specialist as MSS$^{\text{S}}$ and the no data reuse specialist as MSS$^{\text{SN}}$. We denote the corresponding ensemble (stacking) weights associated with a model fit to study $k$ as ${w}_k^{\text{G}}$, ${w}_k^{\text{S}}$ or ${w}_k^{\text{SN}}$.
We do not need to distinguish between model parameters for different MSS procedures because they are equal for all MSS methods; only ensemble weights, estimated in stage B), differ. We refer to the two-stage multi-study stacking methods collectively as MSS methods. Similarly, we collectively refer to the optimal ensemble approaches as OEC methods. 

\subsection{Optimal Ensemble Construction}

We now introduce the OEC counterparts of MSS$^{\text{G}}$, MSS$^{\text{S}}$ and MSS$^{\text{SN}}$, which we refer to as OEC$^{\text{G}}$, OEC$^{\text{S}}$ and OEC$^{\text{SN}}$, respectively. We consider a single SSL per study. To avoid confusion with the two-stage multi-study stacking of section~\ref{subs:previous}, we denote the model parameters associated with the SSL fit to data from study $k$ as $\boldsymbol{\beta}_k^{\text{G}}$, $\boldsymbol{\beta}_k^{\text{S}}$, $\boldsymbol{\beta}_k^{\text{SN}} \in \mathbb{R}^{p+1}$ and the ensemble weights by $\alpha^{\text{G}}_k$, $\alpha^{\text{S}}_k$ and $\alpha^{\text{SN}}_k$, respectively. We also include ensemble intercepts that we do not constrain to be non-negative which we denote $\alpha_0^{\text{G}}$, $\alpha_0^{\text{S}}$, $\alpha_0^{\text{SN}}$. We denote the corresponding vectors, which include both the study-specific ensemble weights and the intercept, by $\boldsymbol{\alpha}^{\text{G}}, \boldsymbol{\alpha}^{\text{S}} \in \mathbb{R}^{p+1}$, and $\boldsymbol{\alpha}^{\text{SN}}\in \mathbb{R}^{p}$. We denote $\mathbb{B}^{\text{G}}$, $\mathbb{B}^{\text{S}} \in \mathbb{R}^{p+1 \times K}$ and $\mathbb{B}^{\text{SN}} \in \mathbb{R}^{p \times K}$ to be matrices of study-specific model coefficients. For example, column $k$ of $\mathbb{B}^{\text{G}}$ is equal to $\boldsymbol{\beta}_k^{\text{G}}$. The SN parameters  $\boldsymbol{\alpha}^{\text{SN}}$ and $\mathbb{B}^{\text{SN}} \in \mathbb{R}^{p \times K}$ have dimension $p$ rather than $p+1$ as they do not contain components corresponding to study $k^*$, as we will further clarify later. Here we focus on the linear-linear OEC,
where the optimal coefficients for SSL and the ensemble weights are given by solutions to the following joint optimization formulations.

\paragraph{Generalist OEC~ (\text{OEC$^{\text{G}}$}):}
We begin with the generalist OEC loss function, defined in the linear-linear case as
    \begin{align*}
    &\underset{ \boldsymbol{\alpha}^{\text{G}} \geq \mathbf{0}, ~\alpha_0^{\text{G}} } {\mbox{min}} ~~~ \underset{\mathbb{B}^{\text{G}}} {\mbox{min }} ~ \left\{ \eta~ \left[ \frac{1}{2N} \norm{\boldsymbol{y} - \alpha_0^{\text{G}} \mathbbm{1} - \sum_{k=1}^K \alpha_k^{\text{G}} \mathbb{X} \boldsymbol{\beta}_k^{\text{G}} }_2^2  + \frac{\mu}{2} \norm{\mathbb{D}_{K+1} \boldsymbol{\alpha}^{\text{G}} }_2^2 \right] \right. + \notag \\
    & ~~~~~~~~~~~~ \qquad \qquad \quad
    (1-\eta) \left. \left[ \sum_{k=1}^K \frac{1}{2n_k} \norm{\boldsymbol{y}_k - \mathbb{X}_k \boldsymbol{\beta}_k^{\text{G}}}_2^2 + \frac{1}{2} \sum_{k=1}^K \lambda_k  \norm{\mathbb{D}_{p+1} \boldsymbol{\beta}_k^{\text{G}} }_2^2 \right ] \right\} \tag{1}
    \label{eq:gen}
\end{align*}
where dimensions are as specified above, and
where $\mu$ and $\lambda_k$'s are regularization parameters for the corresponding regression models. The top term in \eqref{eq:gen} is the loss for the (penalized) non-negative least squares regression used to calculate ensemble weights, and the bottom term in \eqref{eq:gen} is the sum of study-specific Ridge-regularized least squares loss functions. 
The overall objective function is a convex combination of these two terms---the weight $\eta$ controls whether the optimization procedure will seek better study-specific SSL fits (low $\eta$) or fits of the SSLs that better contribute to the ensemble as a whole (high $\eta$). 

\paragraph{Specialist OEC~ (\text{OEC$^{\text{S}}$}):}
The OEC$^{\text{S}}$ formulation is similar except that only data from the target study, $k^*$, is included in the portion of the objective function involving ensemble weights. The resulting optimization formulation is
    \begin{align}
    &\underset{ \boldsymbol{\alpha}^{\text{S}} \geq \mathbf{0}, ~\alpha_0^{\text{S}} } {\mbox{min}} ~~~ \underset{\mathbb{B}^{\text{S}}} {\mbox{min }} ~ \left\{ \eta~ \left[ \frac{1}{2n_{k^*}} \norm{\boldsymbol{y}_{k^*} - \alpha_0^{\text{S}} \mathbbm{1} - \sum_{k=1}^K \alpha_k^{\text{S}} \mathbb{X}_{k^*} \boldsymbol{\beta}_k^{\text{S}} }_2^2  + \frac{\mu}{2} \norm{\mathbb{D}_{p+1} \boldsymbol{\alpha}^{\text{S}} }_2^2 \right] \right. + \notag \\
    & ~~~~~~~~~~~~ \qquad \qquad \quad
    (1-\eta) \left. \left[ \sum_{k=1}^K \frac{1}{2n_k} \norm{\boldsymbol{y}_k - \mathbb{X}_k \boldsymbol{\beta}_k^{\text{S}}}_2^2 + \frac{1}{2} \sum_{k=1}^K \lambda_k  \norm{\mathbb{D}_{K+1} \boldsymbol{\beta}_k^{\text{S}} }_2^2 \right ] \right\} \tag{2}
    \label{eq:spec}
\end{align}

\paragraph{No Data Reuse Specialist OEC~(\text{OEC$^{\text{SN}}$}):}
Formulation~\eqref{eq:spec} can be modified to yield the OEC$^{\text{SN}}$ which does not include a study-specific learner for the target study, $k^*$ and determines parameters and weights by solving:
   \begin{align}
    &\underset{ \boldsymbol{\alpha}^{\text{SN}} \geq \mathbf{0}, ~\alpha_0^{\text{SN}} } {\mbox{min}} ~~~ 
    \underset{ 
    \boldsymbol{\mathbb{B}^{\text{SN}}}} {\mbox{min} }~~~ 
\left\{ \eta~ \left[ \frac{1}{2n_{k^*}} \norm{\boldsymbol{y}_{k^*} - \alpha_0^{\text{SN}} \mathbbm{1} - \sum_{k \neq k^*}^K \alpha_k^{\text{SN}} \mathbb{X}_{k^*} \boldsymbol{\beta}_k^{\text{SN}} }_2^2  + \frac{\mu}{2} \norm{\mathbb{D}_{K} \boldsymbol{\alpha}^{\text{SN}} }_2^2 \right] \right. + \notag \\
    & ~~~~~~~~~~~~ \qquad \qquad \quad
    (1-\eta) \left. \left[ \sum_{k \neq k^*}^K \frac{1}{2n_k} \norm{\boldsymbol{y}_k - \mathbb{X}_k \boldsymbol{\beta}_k^{\text{SN}}}_2^2 + \frac{1}{2} \sum_{k \neq k^*}^K \lambda_k  \norm{\mathbb{D}_{p} \boldsymbol{\beta}_k^{\text{SN}} }_2^2 \right ] \right\} \tag{3}
    \label{eq:zero}
\end{align}
The parameters $\boldsymbol{\beta}^{\text{SN}}_{k^*}$ and $\alpha^{\text{SN}}_{k^*}$ are not included in the objective function or the decision variables. 

Throughout, we denote by $\hat{\boldsymbol{\alpha}}$ and $\hat{\boldsymbol{\beta}}$, with appropriate subscripts and superscripts, solutions of these optimization problems. 
We obtain these estimates by applying the block coordinate descent method described in Supplemental Figure~\ref{suppFig:opt} to the corresponding
optimization problem. 

The loss function associated with the OEC$^{\text{SN}}$ differs from the no data reuse specialist stacking objective proposed in \cite{ren}. We implemented this modification because the algorithm proposed by \cite{ren} does not lend itself to an OEC form that rules out data reuse. While the approach does properly exclude data reuse when implemented in a two-stage procedure, a direct generalization as a joint optimization approach would not prevent reuse because the different folds of the target study data, and the corresponding model coefficients, would be tied through the shared estimation of the ensemble weight $w_{k^*}$. 


Throughout the paper, we distinguish between $generalist$ and $specialist$, algorithms. For this reason, when fitting a model or ensemble to make predictions on some new, unseen dataset, (denoted here as study $K+1$), we compare the performances of MSS$^{\text{G}}$, OEC$^{\text{G}}$, and ToM. When we aim to make predictions on a new observation in a target study for which we have limited training data (study $k^*$), we compare the MSS$^{\text{S}}$ with the OEC$^{\text{S}}$ and the MSS$^{\text{SN}}$ with the OEC$^{\text{SN}}$. We also compare these last four methods to SSM, a model fit only to the training data of the target study, to examine whether borrowing information is beneficial for prediction performance. 

The main methodological question we examine here is whether jointly optimizing ensemble weights and linear model parameters improves prediction performance above two-stage multi-study stacking. We examine this for generalist and specialist multi-study stacking. The possible ensembling methods can be described by a $2\times 2$ table: OEC/MSS $\times$ specialist/generalist. For specialist stacking, we also examine performance with and without data reuse. 
\subsection{Connections between OEC and earlier methods}  
The following observations provide insight into the properties of the OEC framework, as they show that, if both the SSL and the ensemble weight component of the loss are linear models with no regularization on model parameters or ensemble weights, then the ToM, SSM and two-stage multi-study stacking methods are special cases. We present the results in the OEC$^{\text{G}}$, OEC$^{\text{S}}$ and OEC$^{\text{SN}}$ cases specifically. Proofs can be found in Supplemental section \ref{proofs}.

The optimization problems corresponding to the OEC$^{\text{G}}$, OEC$^{\text{S}}$ and OEC$^{\text{SN}}$ with no regularization on model parameters or ensemble weights are:

\paragraph{\text{OEC$^{\text{G}}$} ~Without~ Regularization}
\begin{align}
& \min_{\substack{\boldsymbol{\alpha}^{\text{G}} \geq \mathbf{0}, ~\alpha_0^{\text{G}} \\\mathbb{B}^{\text{G}}  }} \left\{ \frac{\eta}{2N} \norm{\boldsymbol{y} - \alpha_0^{\text{G}} \mathbbm{1} - \sum_{k=1}^K \alpha_k^{\text{G}} \mathbb{X} \boldsymbol{\beta}_k^{\text{G}} }_2^2  +  
    (1-\eta) \sum_{k=1}^K \frac{1}{2n_k} \norm{\boldsymbol{y}_k - \mathbb{X}_k \boldsymbol{\beta}_k^{\text{G}}}_2^2 \right\} \tag{4} \label{eq:genProp}
    \end{align}
\paragraph{\text{OEC$^{\text{S}}$} ~Without~ Regularization}
\begin{align}
& \min_{\substack{\boldsymbol{\alpha}^{\text{S}} \geq \mathbf{0}, ~\alpha_0^{\text{S}} \\\mathbb{B}^{\text{S}}  }} \left\{ \frac{\eta}{2n_{k^*}} \norm{\boldsymbol{y}_{k^*} - \alpha_0^{\text{S}} \mathbbm{1} - \sum_{k=1}^K \alpha_k^{\text{S}} \mathbb{X}_{k^*} \boldsymbol{\beta}_k^{\text{S}} }_2^2  +  
    (1-\eta) \sum_{k=1}^K \frac{1}{2n_k} \norm{\boldsymbol{y}_k - \mathbb{X}_k \boldsymbol{\beta}_k^{\text{S}}}_2^2 \right\}  \tag{5} \label{eq:specProp}
    \end{align}

\paragraph{\text{OEC$^{\text{SN}}$} ~Without~ Regularization}
\begin{align}
&  \min_{\substack{\boldsymbol{\alpha}^{\text{SN}} \geq \mathbf{0}, ~\alpha_0^{\text{SN}} \\\mathbb{B}^{\text{SN}}  }} \left\{ \frac{\eta}{2n_{k^*}} \norm{\boldsymbol{y}_{k^*} - \alpha_0^{\text{SN}} \mathbbm{1} - \sum_{k \neq k^*} \alpha_k^{\text{SN}} \mathbb{X}_{k^*} \boldsymbol{\beta}_k^{\text{SN}} }_2^2  +  
    (1-\eta) \sum_{k \neq k^*} \frac{1}{2n_k} \norm{\boldsymbol{y}_k - \mathbb{X}_k \boldsymbol{\beta}_k^{\text{SN}}}_2^2 \right\}   \tag{6} \label{eq:snProp}
    \end{align}

Proposition \ref{eta1_minus_prop} provides insights into the optimal solutions of the OEC by characterizing how these solutions relate to parameter estimates from earlier methods. 
\begin{prop}\label{eta1_minus_prop}


\begin{corollary}\label{eta1_minus_prop_Gen} (OEC$^{\text{G}}$) Let $(\hat{\boldsymbol{\alpha}}^{\text{G}}(\eta), \hat{\mathbb{B}}^{\text{G}}(\eta))$ be optimal solutions to optimization problem \eqref{eq:genProp}.
Then as $\eta \rightarrow 1-$, 
$(\hat{\boldsymbol{\alpha}}^{\text{G}}(\eta), \hat{\mathbb{B}}^{\text{G}}(\eta))$ will converge to the solutions of:
$$\min_{\boldsymbol{\alpha}^{\text{G}}, \mathbb{B}^{\text{G}}}  \sum_{k = 1}^K \frac{1}{2n_k} \norm{\boldsymbol{y}_k - \mathbb{X}_k \boldsymbol{\beta}_k^{\text{G}}}_2^2 ~~~~~~~~~\text{s.t.}~~~  \frac{1}{2N} \norm{\boldsymbol{y} -\alpha_0^{\text{G}} \mathbbm{1} - \mathbb{X}  \sum_{k = 1}^K \alpha_k^{\text{G}} \boldsymbol{\beta}_k^{\text{G}}}_2^2 = f^*$$
where, $f^* = \min_{\boldsymbol{\alpha}^{\text{G}} \geq 0, \alpha_0^{\text{G}}, \mathbb{B}^{\text{G}}} \frac{1}{2N} \norm{\boldsymbol{y} -\alpha_0^{\text{G}} \mathbbm{1} - \mathbb{X}  \sum_{k = 1}^K \alpha_k^{\text{G}} \boldsymbol{\beta}_k^{\text{G}}}_2^2$.
\end{corollary}

\begin{corollary}\label{eta1_minus_prop_Spec} (OEC$^{\text{S}}$) Let $(\hat{\boldsymbol{\alpha}}^{\text{S}}(\eta), \hat{\mathbb{B}}^{\text{S}}(\eta))$ be optimal solutions to optimization problem \eqref{eq:specProp}.
Then as $\eta \rightarrow 1-$, 
$(\hat{\boldsymbol{\alpha}}^{\text{S}}(\eta), \hat{\mathbb{B}}^{\text{S}}(\eta))$ will converge to the solutions of:
$$\min_{\boldsymbol{\alpha}^{\text{S}}, \mathbb{B}^{\text{S}}}  \sum_{k = 1}^K \frac{1}{2n_k} \norm{\boldsymbol{y}_k - \mathbb{X}_k \boldsymbol{\beta}_k^{\text{S}}}_2^2 ~~~~~~~~~\text{s.t.}~~~  \frac{1}{2n_{k^*}} \norm{\boldsymbol{y}_{k^*} -\alpha_0^{\text{S}} \mathbbm{1} - \mathbb{X}_{k^*}  \sum_{k = 1}^K \alpha_k^{\text{S}} \boldsymbol{\beta}_k^{\text{S}}}_2^2 = f^*$$
where, $f^* = \min_{\boldsymbol{\alpha}^{\text{S}} \geq 0, \alpha_0^{\text{S}}, \mathbb{B}^{\text{S}}} \frac{1}{2n_{k^*}} \norm{\boldsymbol{y}_{k^*} -\alpha_0^{\text{S}} \mathbbm{1} - \mathbb{X}_{k^*}  \sum_{k = 1}^K \alpha_k^{\text{S}} \boldsymbol{\beta}_k^{\text{S}}}_2^2$.
\end{corollary}

\begin{corollary}\label{eta1_minus_prop_Zero} (OEC$^{\text{SN}}$) Let $(\hat{\boldsymbol{\alpha}}^{\text{SN}}(\eta), \hat{\mathbb{B}}^{\text{SN}}(\eta))$ be optimal solutions to optimization problem \eqref{eq:snProp}

Then as $\eta \rightarrow 1-$, 
$(\hat{\boldsymbol{\alpha}}^{\text{SN}}(\eta), \hat{\mathbb{B}}^{\text{SN}}(\eta))$ will converge to the solutions of:
$$\min_{\boldsymbol{\alpha}^{\text{SN}}, \mathbb{B}^{\text{SN}}}  \sum_{k \neq k^*} \frac{1}{2n_k} \norm{\boldsymbol{y}_k - \mathbb{X}_k \boldsymbol{\beta}_k^{\text{SN}}}_2^2 ~~~~~~~~~\text{s.t.}~~~  \frac{1}{2n_{k^*}} \norm{\boldsymbol{y}_{k^*} -\alpha_0^{\text{SN}} \mathbbm{1} - \mathbb{X}_{k^*}  \sum_{k \neq k^*} \alpha_k^{\text{SN}} \boldsymbol{\beta}_k^{\text{SN}}}_2^2 = f^*$$
where, $f^* = \min_{\boldsymbol{\alpha}^{\text{SN}} \geq 0, \alpha_0^{\text{SN}}, \mathbb{B}^{\text{SN}}} \frac{1}{2n_{k^*}} \norm{\boldsymbol{y}_{k^*} -\alpha_0^{\text{SN}} \mathbbm{1} - \mathbb{X}_{k^*}  \sum_{k = 1}^K \alpha_k^{\text{SN}} \boldsymbol{\beta}^{\text{SN}}_k}_2^2$.
\end{corollary}

\end{prop}

Case \ref{eta1_minus_prop_Gen} implies that in the generalist setting, when $\eta \rightarrow 1-$, optimal OEC$^{\text{G}}$ estimates will yield fitted values equal to those produced by the ToM linear model (without regularization): $\mathbb{X} \hat{\boldsymbol{\gamma}}^{\text{ToM}} = \hat{\alpha}_0^{\text{G}} \mathbbm{1} + \mathbb{X} \sum_{k} \hat{\alpha}_k^{\text{G}} \hat{\boldsymbol{\beta}}_k^{\text{G}}$ where $\hat{\boldsymbol{\gamma}}^{\text{ToM}} = \min_{\boldsymbol{\gamma}^{\text{ToM}}} \norm{\mathbf{y} - \mathbb{X} \boldsymbol{\gamma}^{\text{ToM}}}_2^2$. When $\mathbb{X}$ is of full rank, the equivalence can be stated in terms of the parameter estimates directly (and thus predictions made using these solutions): $\hat{\boldsymbol{\gamma}}^{\text{ToM}} = \hat{\alpha}_0^{\text{G}} + \sum_k \hat{\alpha}_k^{\text{G}} \hat{\boldsymbol{\beta}}^{\text{G}}$. Case \ref{eta1_minus_prop_Spec} states that in specialist settings, when $\eta \rightarrow 1-$, optimal OEC$^{\text{S}}$ solutions will yield fitted values equal to those produced by a linear model fit only to the target study, (i.e., study $k^*$): $\mathbb{X}_{k^*} \hat{\boldsymbol{\gamma}}_{k^*} = \hat{\alpha}_0^{\text{S}} \mathbbm{1} + \mathbb{X}_{k^*} \sum_{k} \hat{\alpha}_k^{\text{S}} \hat{\boldsymbol{\beta}}_k^{\text{S}}$ where $\hat{\boldsymbol{\gamma}}_{k^*} = \min_{\boldsymbol{\gamma}_{k^*}} \norm{\mathbf{y}_{k^*} - \mathbb{X}_{k^*} \boldsymbol{\gamma}_{k^*}}_2^2$. Similarly, Case \ref{eta1_minus_prop_Zero} states that  $\mathbb{X}_{k^*} \hat{\boldsymbol{\gamma}}_{k^*} = \hat{\alpha}_0^{\text{SN}} \mathbbm{1} + \mathbb{X}_{k^*} \sum_{k \neq k^*} \hat{\alpha}_k^{\text{SN}} \hat{\boldsymbol{\beta}}_k^{\text{SN}}$. When $\mathbb{X}_{k^*}$ is of full rank it follows that $\hat{\boldsymbol{\gamma}}_{k^*} = \hat{\alpha}_0^{\text{S}} + \sum_k \hat{\alpha}_k^{\text{S}} \hat{\boldsymbol{\beta}}^{\text{S}} = \hat{\alpha}_0^{\text{SN}} + \sum_k \hat{\alpha}_k^{\text{SN}} \hat{\boldsymbol{\beta}}^{\text{SN}}$.  

We next characterize the optimal OEC solutions when $\eta$ is taken to the opposite extreme. 
\begin{prop}\label{eta0_prop} 

\begin{corollary}\label{eta0_prop_Gen} 
   (OEC$^{\text{G}}$) 
    Let $(\hat{\boldsymbol{\alpha}}^{\text{G}}(\eta), \hat{\mathbb{B}}^{\text{G}}(\eta))$ be optimal solutions to optimization problem \eqref{eq:genProp}.
Then as $\eta \rightarrow 0+$, $(\hat{\boldsymbol{\alpha}}^{\text{G}}(\eta), \hat{\mathbb{B}}^{\text{G}}(\eta))$ will converge to a solution to the following:
$$\min_{ \boldsymbol{\alpha}^{\text{G}} \geq \boldsymbol{0}, \alpha_0^{\text{G}}, \mathbb{B}^{\text{G}}} \norm{\boldsymbol{y} - \alpha_0^{\text{G}} \mathbbm{1} -\mathbb{X} \sum_{k = 1}^K \alpha_k^{\text{G}} \boldsymbol{\beta}_k^{\text{G}}}_2^2 ~~~\text{s.t.}~~ \sum_{k =1}^K \frac{1}{2n_k} \norm{\boldsymbol{y}_k - \mathbb{X}_k \boldsymbol{\beta}_k^{\text{G}}}_2^2 = g^*$$
where, $g^* = \min_{\mathbb{B}^{\text{G}}} \sum_{k=1}^K \frac{1}{2n_k} \norm{\boldsymbol{y}_k - \mathbb{X}_k \boldsymbol{\beta}_k^{\text{G}}}_2^2$.
\end{corollary}
\begin{corollary}\label{eta0_prop_Spec} 
   (OEC$^{\text{S}}$) 
    Let $(\hat{\boldsymbol{\alpha}}^{\text{S}}(\eta), \hat{\mathbb{B}}^{\text{S}}(\eta))$ be optimal solutions to to optimization problem \eqref{eq:specProp}.
Then as $\eta \rightarrow 0+$, $(\hat{\boldsymbol{\alpha}}^{\text{S}}(\eta), \hat{\mathbb{B}}^{\text{S}}(\eta))$ will converge to a solution to the following:
$$\min_{ \boldsymbol{\alpha}^{\text{S}} \geq \boldsymbol{0}, \alpha_0^{\text{S}}, \mathbb{B}^{\text{S}}} \norm{\boldsymbol{y}_{k^*} - \alpha_0^{\text{S}} \mathbbm{1} - \mathbb{X}_{k^*} \sum_{k = 1}^K \alpha_k^{\text{S}} \boldsymbol{\beta}_k^{\text{S}}}_2^2 ~~~\text{s.t.}~~ \sum_{k =1}^K \frac{1}{2n_k} \norm{\boldsymbol{y}_k - \mathbb{X}_k \boldsymbol{\beta}_k^{\text{S}}}_2^2 = g^*$$
where, $g^* = \min_{\mathbb{B}^{\text{S}}} \sum_{k=1}^K \frac{1}{2n_k} \norm{\boldsymbol{y}_k - \mathbb{X}_k \boldsymbol{\beta}_k^{\text{S}}}_2^2$.
\end{corollary}
\begin{corollary}\label{eta0_prop_zero} 
   (OEC$^{\text{SN}}$) 
   Let $(\hat{\boldsymbol{\alpha}}^{\text{SN}}(\eta), \hat{\mathbb{B}}^{\text{SN}}(\eta))$ be optimal solutions to optimization problem \eqref{eq:genProp}.
Then as $\eta \rightarrow 0+$, $(\hat{\boldsymbol{\alpha}}^{\text{SN}}(\eta), \hat{\mathbb{B}}^{\text{SN}}(\eta))$ will converge to a solution to the following:
$$\min_{ \boldsymbol{\alpha}^{\text{SN}} \geq \boldsymbol{0}, \alpha_0^{\text{SN}}, \mathbb{B}^{\text{SN}}} \norm{\boldsymbol{y}_{k^*} - \alpha_0^{\text{SN}} \mathbbm{1} - \mathbb{X}_{k^*} \sum_{k \neq k^*}^K \alpha_k^{\text{SN}} \boldsymbol{\beta}_k^{\text{SN}}}_2^2 ~~~\text{s.t.}~~ \sum_{k \neq k^*} \frac{1}{2n_k} \norm{\boldsymbol{y}_k - \mathbb{X}_k \boldsymbol{\beta}_k^{\text{SN}}}_2^2 = g^*$$
where, $g^* = \min_{\mathbb{B}^{\text{SN}}} \sum_{k \neq k^*} \frac{1}{2n_k} \norm{\boldsymbol{y}_k - \mathbb{X}_k \boldsymbol{\beta}_k^{\text{SN}}}_2^2$.
\end{corollary}

\end{prop}
Proposition \ref{eta0_prop} shows that as $\eta \rightarrow 0+$, the optimal OEC solution yields fitted values equal to those of its MSS counterpart. For example, in the generalist case, as $\eta \rightarrow 0+$, we have $\hat{w}_0 \mathbbm{1} + \mathbb{X} \sum_k \hat{w}_k^{\text{G}} \hat{\boldsymbol{\gamma}} = \hat{\alpha}_0^{\text{G}} \mathbbm{1} + \mathbb{X} \sum_{k} \hat{\alpha}_k^{\text{G}} \hat{\boldsymbol{\beta}}_k^{\text{G}}$. If each $\mathbb{X}_k$ is full rank for $k \in [K]$, Proposition \ref{eta0_prop} states that MSS$^{\text{G}}$, MSS$^{\text{S}}$ and MSS$^{\text{SN}}$ are special cases of their corresponding OEC counterparts that arise when $\eta \rightarrow 0+$. For example, in the generalist case, if each of the $\mathbb{X}_k$ are full rank, then $\hat{\boldsymbol{w}}^{\text{G}} = \hat{\boldsymbol{\alpha}}^{\text{G}}$ and  $\hat{\boldsymbol{\gamma}}_k = \hat{\boldsymbol{\beta}}_k^{\text{G}} ~\forall k \in [K]$. 

Together, these propositions demonstrate that in the linear-linear setting, without regularization, the OEC encompasses important special cases and earlier multi-study methods: MSS$^{\text{G}}$ is a special case of OEC$^{\text{G}}$, arising when $\eta \rightarrow 0+$, while pooling datasets with the ToM approach is a special case arising when $\eta \rightarrow 1-$. Similarly, MSS$^{\text{S}}$ is a special case of the OEC$^{\text{S}}$, arising when $\eta \rightarrow 0+$, while a study-specific model is obtained when $\eta \rightarrow 1-$. These propositions examine behavior at the extremes, and suggest that, for intermediate values of $\eta$, the OEC can be viewed as a trade-off between MSS$^{\text{G}}$ and pooling datasets together before model fitting (i.e., the ToM) in the generalist case. In the specialist case, the propositions suggest that the OEC can be framed as a trade-off between MSS$^{\text{S}}$ and only fitting a model on the target study. In our motivating application, we visualize this trade-off in the no data reuse specialist case by showing how predicted values of the OEC smoothly vary between the predictions of the study-specific model and the MSS$^{\text{SN}}$ as a function of $\eta \in (0,1)$. Across the range of $\eta$, the OEC allows for a spectrum of solutions that range in both the manner and degree to which information is borrowed from other sources.

\section{Application: Counterfactual Mortality Prediction}

\subsection{Data}
We consider a data source recently developed by the Human Mortality Database, called the \textit{Short-term Mortality Fluctuations} (STMF) data series \citep{hmd}. The STMF data series provides weekly death counts and death rates for 38 countries, as outlined in Supplementary Table \ref{table:mortWeeks}. We denote the reported death rates for country $k$ at week $t$ by $\tilde{Y}_{t,k}$. STMF calculates it as the observed weekly death count, $C_{t,k}$ divided by an annual measure of population size $N_{l,k}$, that is: $\tilde{Y}_{t,k} = C_{t,k} / N_{l,k}$ when week $t$ occurs during year $l$. As a result, these death rates exhibit artificial discontinuities at the start of each calendar year, that are an artifact of changing the annual population size from $N_{l,k}$ to $N_{l+1,k}$. To avoid this, we followed the strategy outlined in \cite{us2012methodology}. We calculated the mid-year population size, $N_{l,k}$, (i.e., observed population counts taken on July 1 of year $l$) from the death rates and death counts provided by STMF. We then linearly interpolated population size by regressing the annual population estimates, $N_{l,k}$ against year, $l$, and used the regression fit to calculate weekly estimates of population size: $\hat{P}_{t,k} = \hat{\beta_0}^{\dagger} + \hat{\beta_1}^{\dagger} t$, where $\hat{\beta_0}^{\dagger}$ and $\hat{\beta_1}^{\dagger}$ were linear model coefficient estimates from this regression. Finally we calculated new annual death rates (per 1,000 individuals) as $Y_{k_t} = 1,000 \times 52 \times C_{t,k} / \hat{P}_{k,t}$.

We only used test countries in the Northern hemisphere since the model was based on seasonal effects modeled by calendar date. There are only three countries from the southern hemisphere in STMF (Australia, New Zeland and Chile) a set too small to fit a model borrowing information across countries, thereby resulting in poor performance. 
We did, however, allow the OEC methods to leverage information from countries in the Southern hemisphere when training models tailored to countries in the Northern hemisphere. For any given test year, we required countries to have at least 2 years of data at that point to serve as auxiliary sources (i.e., countries from which we can borrow information for a given test country/year). We started testing at 2003 because prior to 2003 there was little data and excessive variability year-to-year in the number of training countries.


\subsection{Study-specific models}
It is common to model expected death counts as a function of seasonal and secular trends in mortality. 
We modeled annual death rates, $Y_{k,t}$, with the following linear model:
\begin{equation}
    \mathbb{E}[{y}_{k,t} \given \mathbf{x}_{k,t}]
    = \gamma_{k,0} + \gamma_{k,1} t + \sum_{j=1}^2 \left[\gamma_{k,j + 1}\sin \left(\frac{2\pi j t}{52}\right) + \gamma_{k,j+3}\cos \left(\frac{2\pi j t}{52}\right)\right]
     \label{eq:lm}
\end{equation}
where ${\gamma}_{k,1}$ represents a linear effect of time to capture secular trends in mortality, and $\boldsymbol{\gamma}_{k,2:p}$ represents the effect of the Fourier basis function to capture seasonal changes. Model \eqref{eq:lm} can be represented in the familiar linear form, $\mathbb{E}[{y}_{k,t}\given \mathbf{x}_{k,t}] = \mathbf{x}_{k,t}^T \boldsymbol{\gamma}$, where the elements of $\mathbf{x}_{k,t}$ simply contain basis expansions of time. One could adopt, for example, a Poisson likelihood for the conditional distribution of the death counts given the covariates (i.e., basis expansions of time). However, weekly death counts were in the hundreds for all countries except Luxembourg and Iceland and most countries recorded weekly counts in the thousands. Since the counts were high, we expect a Gaussian to provide a good approximation to a Poisson distribution. 
Since we sought to borrow information across countries, we modeled rates instead of counts to make the outcome more comparable across the countries in the database.
We opted for model \eqref{eq:lm} assuming ${y}_{k,t}\given \mathbf{x}_{k,t}$ follows a Gaussian distribution instead of a Poisson for computational convenience. In Supplementary Figures \ref{suppFig:mortalityOECLin_vs_countryNoLin}-\ref{suppFig:mortalityOECnoLin_vs_stackNoLin} we show the results of sensitivity analyses where we omit the linear effect of time $\gamma_{k,1} t$ from model~\eqref{eq:lm}.

\subsection{Application: Problem Statement and Assessment of Methods}
Excess mortality in country $k$, at week $t$, can be defined as the observed death count minus the rate expected in the counterfactual world in which COVID-19 never existed. Estimating excess deaths therefore requires a prediction of mortality in the counterfactual world where a disaster did not occur. In practice, this prediction is the expected death counts in a particular week given historical data (i.e., all data before the onset of the disaster). Unfortunately, prediction performance in such settings is unverifiable. Hence, to assess if our methods are useful for counterfactual prediction, we trained models and tested their performance on historical data where no pandemic occurred, so that a ground-truth was known. Heuristically, the rationale is that if the baseline model has high prediction accuracy on historical data, we expect it would also predict counterfactual mortality~well. 

We tested predictors generated using the OEC$^{\text{S}}$ and the OEC$^{\text{SN}}$ against their MSS$^{\text{S}}$ and MSS$^{\text{SN}}$ counterparts. We also implemented a SSM (which we refer to interchangeably as a ``country-specific model'' here), in which we did not borrow information from other countries. In all approaches, we used model \eqref{eq:lm} as the single-study learner. 

We sought to evaluate prediction performance for all countries meeting the inclusion criteria in the database across a range of training/testing years. We assessed our question first in the STMF dataset using a time-series hold-one-country-out testing procedure. That is, we trained models on a subset of historical data (i.e., a subset of years), made predictions on held out historical data and computed performance metrics (e.g., RMSE) to quantify the quality of the predictions. We tailored our approach to assess whether our method improved performance when there exists little pre-COVID-19 training data for the target country. Specifically, for each target country and year (from 2003-2019) we artificially assumed only having one year of training data from that target country in the database. We then made predictions of mortality in that country for the following year and compared those predictions to what was observed.


Our approach is illustrated in Figure \ref{fig:cvFig1}. To fix ideas, say we are interested in making predictions for the Country~1 in 2014. We would use data from Country~1 in 2014 as a test set and trained a country-specific model on US data from only 2013 (i.e., simulating only having one year's worth of data). To analyze the performance when we borrowed information, we allowed the MSS$^{\text{S}}$, MSS$^{\text{SN}}$, OEC$^{\text{S}}$ and OEC$^{\text{SN}}$ to use data from any other country that had at least 100 weeks of training data prior to 2014 (e.g., 2011, 2012 and 2013). All data after the start of the test period was discarded and any country that did not have enough training data was not used as an auxiliary source. All approaches were compared based upon their prediction performance on US data from 2014. We repeated this procedure for every feasible year between 2003 and 2019 and for all countries (assuming the above inclusion criteria are met). The test period was set to an entire year to assess the performance of the method's estimation of both secular and seasonal trends as well as to mimic the long term predictions required for COVID-19 excess mortality estimation.


\begin{figure}[h]
	\centering
	\begin{subfigure}[t]{1\textwidth}
		\centering
		\includegraphics[scale = 0.3]{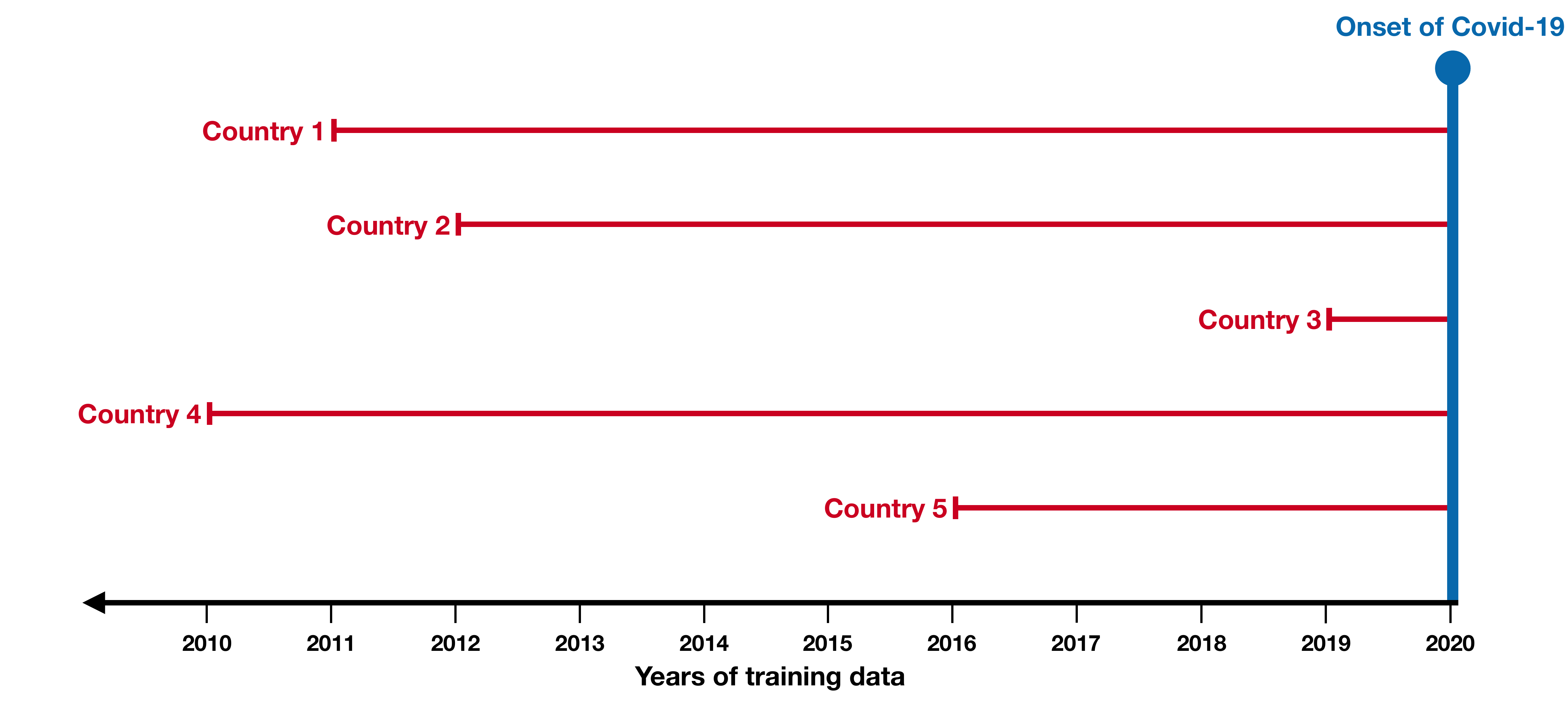}
		\caption{ Example observed data. Countries have different amounts of data before the onset of COVID-19 that could be used for training and validation.}
	\end{subfigure}
	\hfill
    \vspace*{10mm}
	\begin{subfigure}[t]{1\textwidth}
		\centering
		\includegraphics[scale = 0.3]{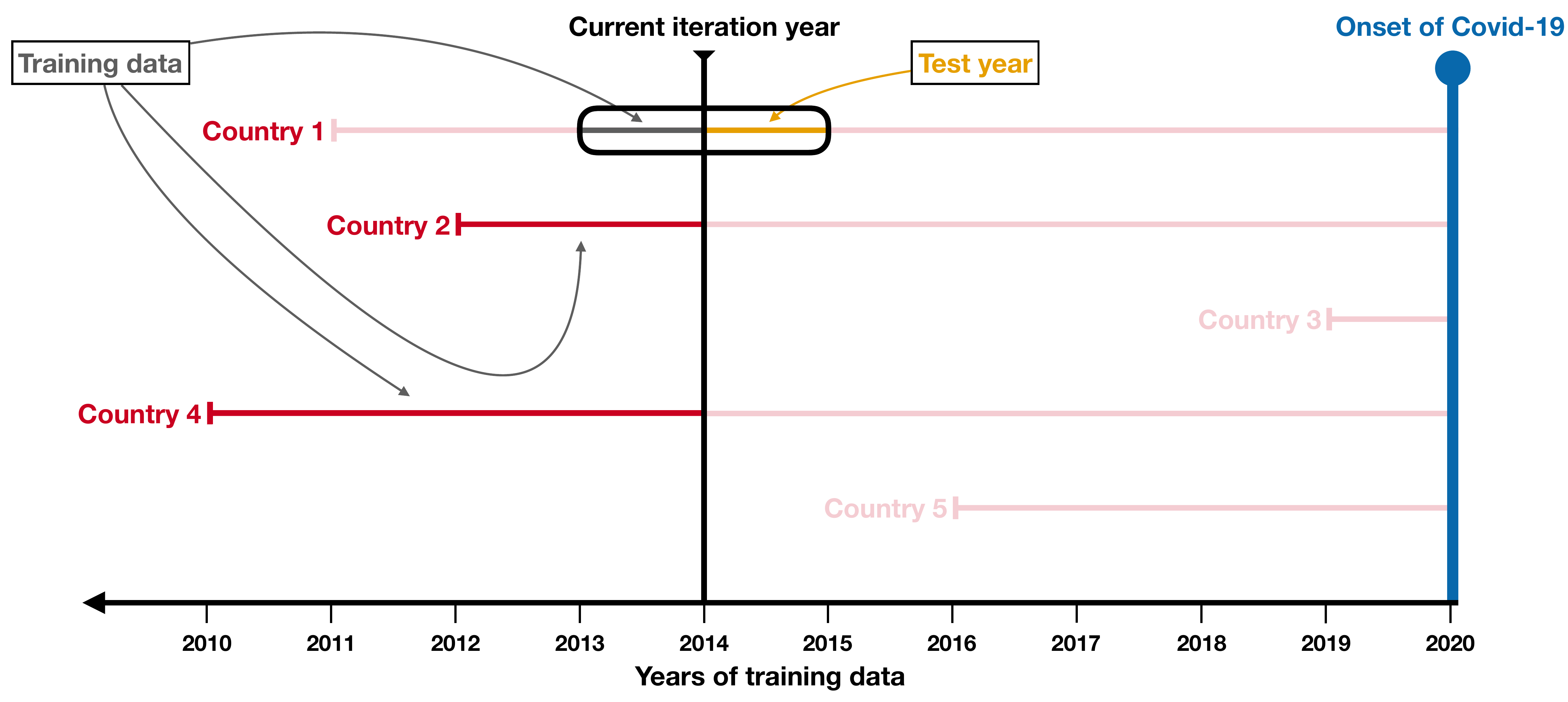}
		\caption{Architecture of leave-one-country-out prediction experiments.}
	\end{subfigure}
	\caption{Illustration of the validation scheme. Country 1 is used as a test country. Year 2013 from country 1 is utilized for training and year 2014 is taken as a test set; other data from country 1 is discarded for training or testing purposes. Country 2 and country 4 are used in multi-study versions. Data after the test year is discarded. Country 5 and Country 3 are not utilized since they do not have data before the test year 2014. This prediction is repeated for each year and country that meet the criteria for inclusion.}
	\label{fig:cvFig1}
\end{figure}

\subsection{Multi-study Modeling and Parameter Tuning}

 We ran analyses both with and without a Ridge (squared $\ell_2$) penalty term on the country-specific model parameters and the ensemble weight parameters. The inclusion of a penalty was kept consistent across all methods that we implemented. This was to ensure that any gains in performance achieved through the OEC could not have been attained through simple regularization.

We centered and scaled covariates prior to model fitting. We tuned $\eta$ in expression \eqref{eq:gen} using cross validation procedures specific to generalist and specialist settings. For the OEC$^{\text{S}}$ and OEC$^{\text{SN}}$, we tuned $\eta$ specifically for the target study. We divided the training set for study $k^*$ into $K$ folds (i.e., we set the number of folds to $K$, the number of training studies) and used each fold as a validation set. For OEC$^{\text{G}}$, we tuned $\eta$ with a $K$-fold cross validation procedure in which the folds were constructed to be study-balanced in order to encourage generalizability to an unseen study. The $\ell_2$ regularization parameter associated with ensemble weights, $\mu$, was tuned using a $K$-fold, study-balanced, cross validation scheme separately using the OEC$^{\text{G}}$ algorithm. We kept this tuning parameter fixed across all OEC algorithms. Similarly, we tuned the $\ell_2$ regularization parameter associated with the stacking regression again using a $K$-fold cross validation. We tuned study-specific model regularization parameters associated with the Ridge penalty using a within-study (i.e., within-country) $K$-fold cross validation scheme and kept these fixed across all methods implemented (i.e., ToM, SSM, and all OEC and MSS approaches). The $\ell_2$ regularization parameter associated with ToM algorithm was tuned using a hold-one-study-out cross validation (e.g., the $K^{th}$ fold merged all studies, $1,...,K - 1$, fit a single model and validated on study  $K$).

\subsection{Application Results: COVID-19}

We present results by test-year for both the the OEC$^{\text{S}}$ and OEC$^{\text{SN}}$ in Figure~\ref{fig:covid_box1} and Supplemental Figure~\ref{suppFig:mortalityTogether100_OLS}. We present results from models fit both with study-specific Ridge penalties (Supplemental Table~\ref{table:covid_Ridge}) and without them (Table~\ref{table:covid_OLS}). Table~\ref{table:covid_OLS} only contains a subset of years (2010-2019) for purposes of visual presentation but a version of the table including all years (2003-2019) can be found in Supplemental Table~\ref{table:Fullcovid_OLS}. Although we fit the generalist methods ToM, OEC$^{\text{G}}$ and MSS$^{\text{G}}$, their performance was uniformly poor. This is not surprising given that this application involves transfer learning. As a result, we do not present these results here.

The figures and tables demonstrate that borrowing information through two-stage stacking procedures (MSS$^{\text{S}}$ or MSS$^{\text{SN}}$) or their OEC counterparts (OEC$^{\text{S}}$ and OEC$^{\text{SN}}$) has strong benefits with or without regularization. Importantly, the OEC$^{\text{S}}$ substantially outperformed both a country-specific model and the MSS$^{\text{S}}$. The OEC$^{\text{SN}}$ outperformed the MSS$^{\text{SN}}$ but improvements were more modest than those comparing the OEC$^{\text{S}}$ and the MSS$^{\text{S}}$. The OEC$^{\text{SN}}$ and MSS$^{\text{SN}}$ substantially outperformed the OEC$^{\text{S}}$ and the MSS$^{\text{S}}$, respectively. We also present the results relative to a country-specific model to show the relative performance of the OEC and MSS methods in Supplemental Figure \ref{suppFig:mortalityTogether100_OLS}.

\begin{figure}[h]
	\centering
	\includegraphics[width=0.95\linewidth]{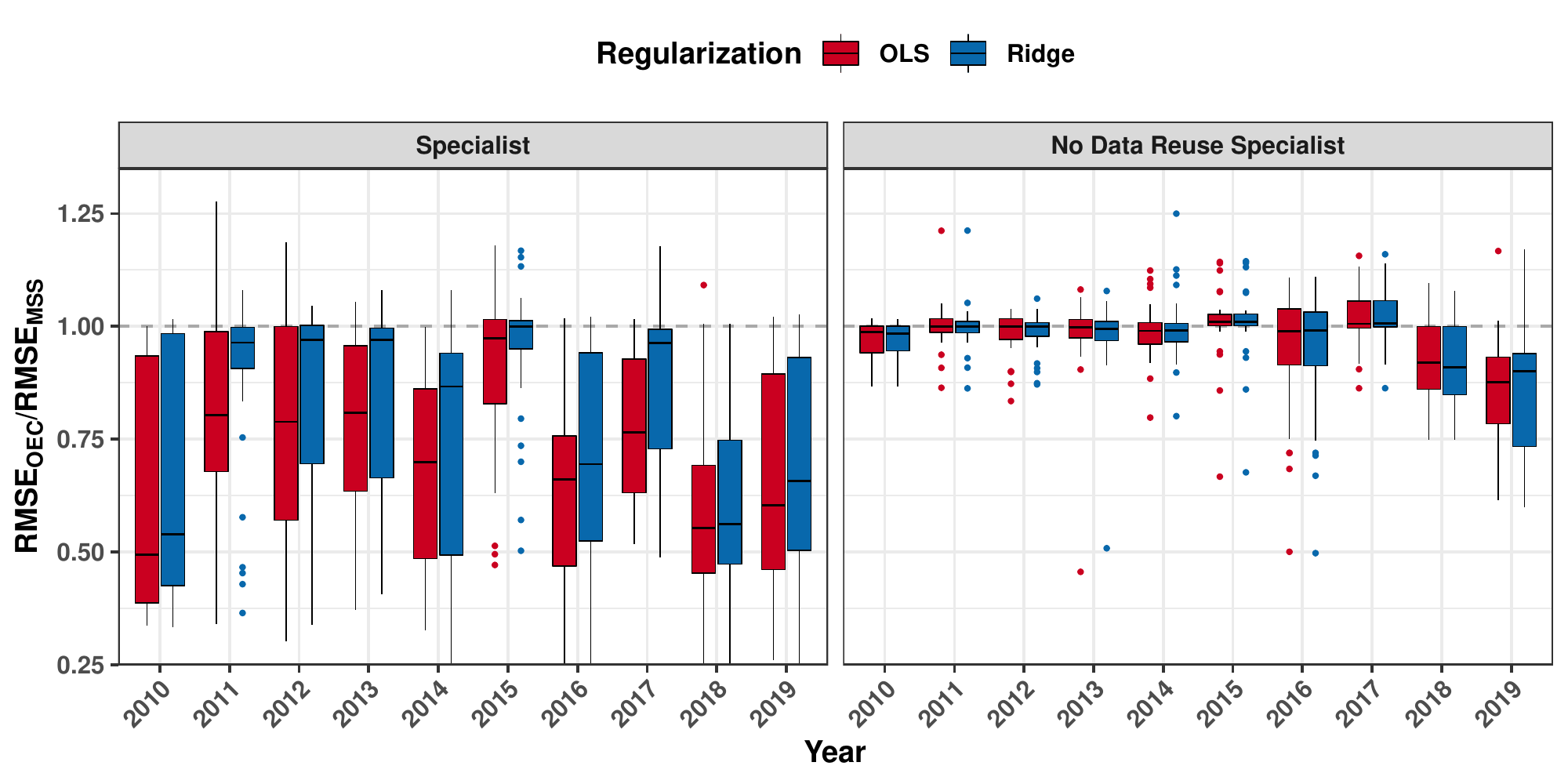}
	\caption{$RMSE_{OEC} / RMSE_{MSS}$. Performance of OEC methods relative to their multi-study stacking counterparts: values below 1.00 indicate OEC method produces superior performance relative to MSS version. The x-axis indicates the year of the last day in the test-set.}
	\label{fig:covid_box1}
\end{figure}

\begin{table}[h]
 \centering
\begin{minipage}[t]{1.0\linewidth} 
\centering
\begin{tabular}{l|rrrrrrrrrr}
\toprule
\toprule
\multicolumn{1}{c}{\bfseries Method} \vline & \multicolumn{10}{c}{\bfseries Test Year} \\
  & 2010 & 2011 & 2012 & 2013 & 2014 & 2015 & 2016 & 2017 & 2018 & 2019\\
\midrule
\midrule
MSS$^{\text{S}}$ & 0.98 & 0.92 & 0.97 & 0.91 & 0.92 & 0.94 & 0.95 & 0.90 & 0.97 & 0.92\\
OEC$^{\text{S}}$ & 0.60 & 0.72 & 0.76 & 0.70 & 0.62 & 0.84 & 0.63 & 0.72 & 0.59 & 0.61\\
MSS$^{\text{SN}}$ & 0.54 & 0.71 & 0.74 & 0.69 & 0.60 & 0.86 & 0.58 & 0.69 & 0.40 & 0.58\\
OEC$^{\text{SN}}$ & 0.53 & 0.71 & 0.74 & 0.68 & 0.60 & 0.86 & 0.57 & 0.70 & 0.38 & 0.52\\
\bottomrule
\end{tabular}
\end{minipage}
\hfill
\caption{Average $RMSE_{Method} / RMSE_{SSM}$. Performance of ensembling methods relative to a country-specific model that was fit without a country-specific Ridge penalty. Columns indicate test year.}
\label{table:covid_OLS}
\end{table}


When training data is limited, including a linear term for time to capture secular trends in a country-specific model may introduce identifiability issues. This is because the model may struggle to distinguish between seasonal and secular trends when trained on a single year of data, even with regularization. While we showed that borrowing information alleviates this issue, we wanted to ensure the multi-study approach was still beneficial compared to a simpler model with no linear effect of time. In the supplemental section we present results from analyses using a model with no linear term \eqref{eq:lmSecular}. We are not aware of past analyses that only model mortality with seasonal trends, but we explored these results as a sensitivity analysis to further characterize how the OEC methods improve performance over MSS methods. First, we compared an OEC with a linear trend to a country-specific model with no linear trend (Supplemental Figure \ref{suppFig:mortalityOECLin_vs_countryNoLin}). The OEC$^{\text{SN}}$ consistently outperformed the simpler country-specific model, but the OEC$^{\text{S}}$ and two-stage specialist stacking did not. This may suggest that the country-specific model struggled with identifiability problems and that the OEC$^{\text{S}}$ and MSS$^{\text{S}}$ may place too large an ensemble weight on that target country's model. The MSS$^{\text{SN}}$ and OEC$^{\text{SN}}$ avoid this problem by only using the target data to calculate ensemble weights. This is consistent with previous reports about limitations of the data re-use approach MSS$^{\text{S}}$ \citep{ren}. 

We next compared the OEC$^{\text{S}}$ and OEC$^{\text{SN}}$ models fit with no linear trend to their MSS counterparts. In both, the design matrix only included Fourier basis terms. Results from this analysis can be found in Supplemental Figure \ref{suppFig:mortalityOECnoLin_vs_stackNoLin}. In this simpler modeling scheme as well, the OEC$^{\text{S}}$ and OEC$^{\text{SN}}$ approaches outperformed their MSS counterparts, suggesting that the OEC approach can improve the estimation of both seasonal and secular trends. Importantly, the OEC$^{\text{SN}}$ and OEC$^{\text{S}}$ performed comparably. This may suggest that the OEC$^{\text{SN}}$ yields the greatest benefit over OEC$^{\text{S}}$ when the modeling scheme is susceptible to problems related to identifiability or over-fitting. 
Finally, we compared OEC$^{\text{SN}}$ and OEC$^{\text{S}}$ models fit with no linear term for time to a country-specific model without a linear term for time (Supplemental Figure \ref{suppFig:mortalityOECnoLin_vs_countryNoLin}). These results demonstrate that also in the simpler modeling case, there are strong benefits to borrowing information with the OEC approach. 

The present work seeks to develop methods to improve prediction performance when training data for a target study (country) is limited. We conducted a sensitivity analysis to characterize how prediction performance varies as a function of the quantity of training data available for the target study. Supplemental Figure \ref{suppFig:mortalityTrainingMonths} shows that OEC$^{\text{S}}$ was associated with superior prediction performance compared to both a SSM and a MSS$^{\text{S}}$ when the target study had as much as three years of training data, although the benefit monotonically decreased as number of training observations increases. Importantly, borrowing information was never detrimental, on average, in any of the settings explored.

We conclude the application with an example to demonstrate that the country-specific model and the MSS$^{\text{SN}}$ are special cases of the OEC$^{\text{SN}}$ that arise from limiting values of $\eta \in (0, 1]$ (Figure \ref{fig:convexCombo}). We selected the country and test year presented in the figure to show an example that yielded country-specific estimates and OEC$^{\text{SN}}$ estimates that were different enough to visually depict the effect of $\eta$ but did not overstate the utility of borrowing information with either MSS$^{\text{SN}}$ or OEC$^{\text{SN}}$. This test year exhibit a spike in mortality counts for three consecutive outlying weeks early in the year. This type of pattern occurred commonly across countries and years and is difficult for any modeling approach to capture. As $\eta \rightarrow 0$, the OEC predictions approach that of the MSS$^{\text{SN}}$. As $\eta \rightarrow 1$, the OEC$^{\text{SN}}$ predictions approach that of the country-specific model. 


\begin{figure}[H]
	\centering
		\centering	\includegraphics[width=1\linewidth]{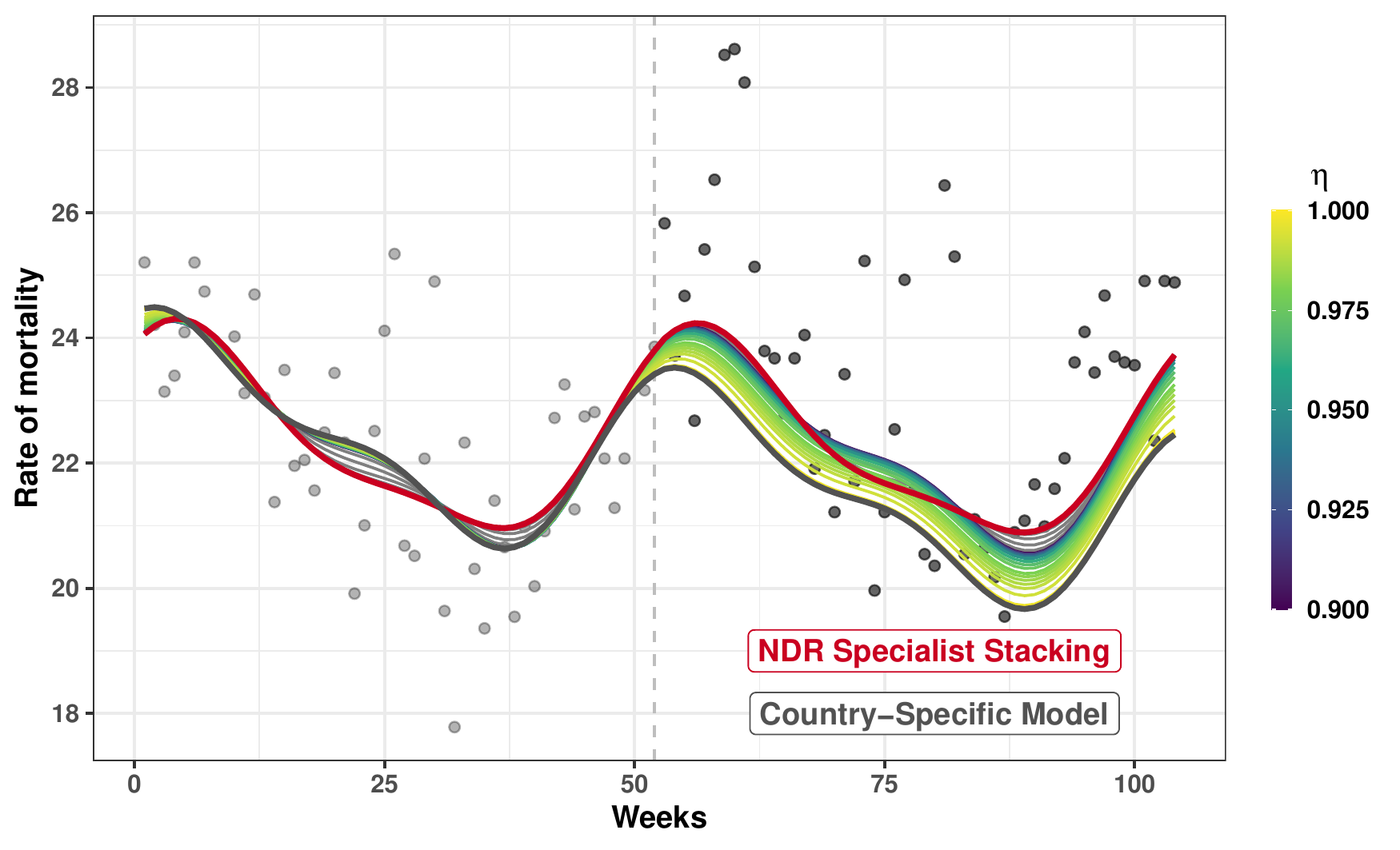}
\caption{Fitted values and predictions of OEC$^{\text{SN}}$ as a function of $\eta$. As $\eta \rightarrow 1$, the OEC$^{\text{SN}}$ predictions approach those of the country-specific model (SSM). When $\eta \rightarrow 0$, the OEC$^{\text{SN}}$ predictions approach those made by NDR Specialist Stacking (MSS$^{\text{SN}}$). We show $\eta \in [0,0.9]$ as a color gradient to zoom in on the most sensitive part of the tuning curve in this example. Gray lines indicate predictions for $\eta < 0.9$. Light (dark) gray points to the left (right) of the dotted vertical line are training (test) samples. Lines indicate fitted values (predictions) to the left (right) of the dotted vertical line.}
\label{fig:convexCombo}
\end{figure}

\section{Data-driven simulations based on COVID-19 Mortality Application}

\subsection{Design of Data-Driven Simulation Experiments}

We next conducted data-driven simulations to explore the COVID-19 application in a setting where we could compare predictions against a ground-truth. We simulated data from a linear mixed effects model: $$y_{k,t} = \mathbf{x}_{k,t}^T {\boldsymbol{\theta}}_k + \epsilon_{k,t}$$
where ${\boldsymbol{\theta}}_k \perp \!\!\! \perp \boldsymbol{\epsilon}_k$, ${\boldsymbol{\theta}}_k \sim N_{p}(\boldsymbol{\mu}_{\theta}, \sigma^2_{\theta} {\Sigma}_{\theta})$ and $\mathbf{x}_{k,t} \in \mathbb{R}^{p+1}$ which includes a column of ones for an intercept. $\mathbf{x}_{k,t}$ includes only basis expansions of time and was constructed as described above in the COVID-19 application. We chose the parameters of the random effects distribution based on estimates from the COVID-19 data. Denoting by $\hat{\boldsymbol{\gamma}}_k$ the coefficient estimates from an OLS fit to all observed data (i.e., all years prior to 2020) from country $k$, we define $\boldsymbol{\mu}_{\theta}= \bar{\hat{\boldsymbol{\gamma}}} = \frac{1}{K} \sum_k \hat{\boldsymbol{\gamma}_k}$ to be the sample mean of each coefficient, averaged across country-specific estimates. To explore the impact of heterogeneity in model coefficients on prediction performance we choose multiple scenarios for $\sigma^2_{\theta}$. For ${\Sigma}_{\theta}$, we used the empirical covariance matrix from the application by setting ${\Sigma}_{\theta} = \hat{\Sigma}_{\gamma}$ where $\hat{\Sigma}_{\gamma} = \frac{1}{K-1}\sum_k (\hat{\boldsymbol{\gamma}}_k \hat{\boldsymbol{\gamma}}_k^T - K \bar{\hat{\boldsymbol{\gamma}}} \bar{\hat{\boldsymbol{\gamma}}}^T)$. We assumed that the residuals of study $k$ are distributed as $\boldsymbol{\epsilon}_{k} \sim N_{n_k}(\mathbf{0}, {\sigma}^{2}_{\epsilon_k} \mathbb{I})$. We selected the variance of the residuals, ${\sigma}^{2}_{\epsilon_k}$, by sampling with replacement from the vector of estimated residual variances from each of the country-specific models $[$ $\hat{\sigma}^2_{\epsilon_1},...,\hat{\sigma}^2_{\epsilon_K} ]^T$. This allowed us to realistically emulate the variability in noise associated with the observed COVID-19 data from each of the countries.

The design matrix for study $k$, $\mathbb{X}_k$ is constructed as in the COVID-19 application from time (in weeks): it includes a linear effect for time and 4 terms that are Fourier basis expansions of time for seasonal trends. Including the column of ones for the intercept, $\mathbb{X}_k \in \mathbb{R}^{n_k \times 6}$. The test country was set to have 52 weeks of (weekly) training data ($n_{k^*} = 52$) and 52 weeks of test data. We drew $n_k$, the number of weeks of training data for $k \in [K] \setminus k^*$, from a discrete uniform: $n_k \sim \mbox{Unif}(104, 517)$. The parameters of the uniform were selected based upon the observed sample size in the real dataset and to ensure each study, for $k \in [K] \setminus k^*$, had at least twice as many weekly observations as the target country, $k^*$. 
Although the errors $\boldsymbol{\epsilon}_{k}$ are i.i.d, the basis functions used in the simulation scheme induces autocorrelation in the outcome $\mathbf{y}_k$, which we explored by inspection of the empirical autocorrelation function (ACF). We included a Ridge penalty in the stacking regression for the MSS methods. We included a Ridge penalty in the ensemble weighting component of the loss in the OEC methods as well. We assessed performance of the MSS and OEC methods with and without study-specific Ridge penalties. We present results with study-specific Ridge penalties in the supplement (Supplemental Figure \ref{fig:mortSims_Ridge}; Supplemental Table \ref{tab:mortSims_Ridge}). Results were comparable with and without the study-specific penalties. 
\subsection{Data-Driven Simulations Results}

In almost all settings explored, the OEC and MSS approaches substantially outperformed simply using a country-specific model. The OEC$^{\text{S}}$ appears to perform better compared to the MSS$^{\text{S}}$ when the magnitude of between-study heterogeneity in true model coefficients, ($\sigma^2_{\theta}$) is lower. Conversely, the OEC$^{\text{SN}}$ performed better relative to the MSS$^{\text{SN}}$ approach when the $\sigma^2_{\theta}$ was higher. Both the OEC$^{\text{S}}$ and OEC$^{\text{SN}}$ performed better relative to their MSS counterparts for lower $K$, although the effect of $K$ on performance varied as a function of the magnitude of $\sigma^2_{\theta}$. Importantly, the OEC$^{\text{S}}$ and OEC$^{\text{SN}}$ almost never yielded worse average performance than their MSS counterparts and they often conferred substantial benefit. In the rare instances where the OEC exhibited worse average performance than its MSS counterpart, it was only by roughly $1 \%$. These results demonstrate 1) borrowing information is almost uniformly beneficial and 2) the OEC approaches almost always outperformed using a MSS method.

\begin{figure}[!htbp]
	\centering
	\begin{subfigure}[t]{1\textwidth}
		\centering
		\includegraphics[width=1.0\linewidth]{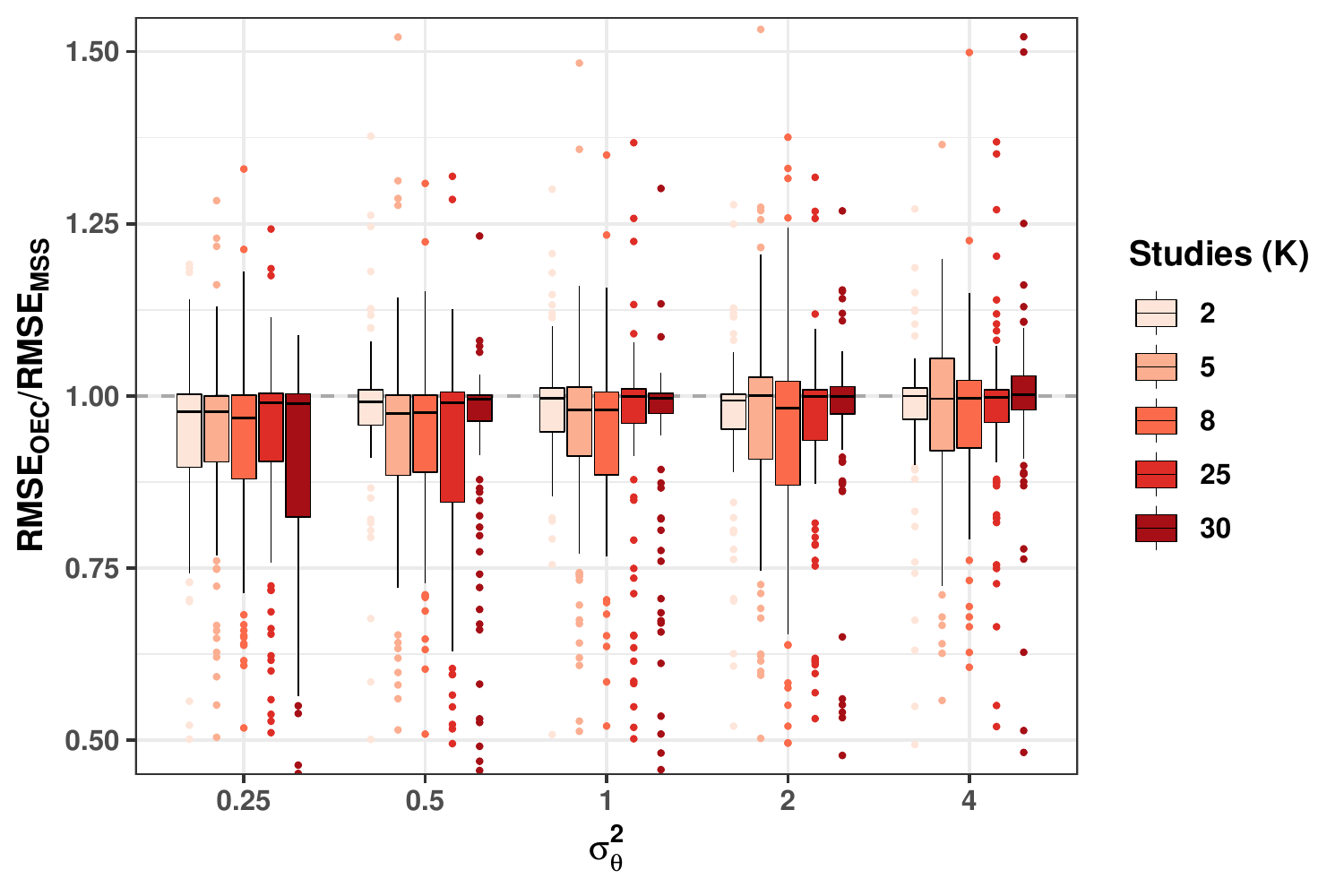}
		\caption{Specialist}
	\end{subfigure}
	\hfill
	\begin{subfigure}[t]{1\textwidth}
		\centering
		\includegraphics[width=1.0\linewidth]{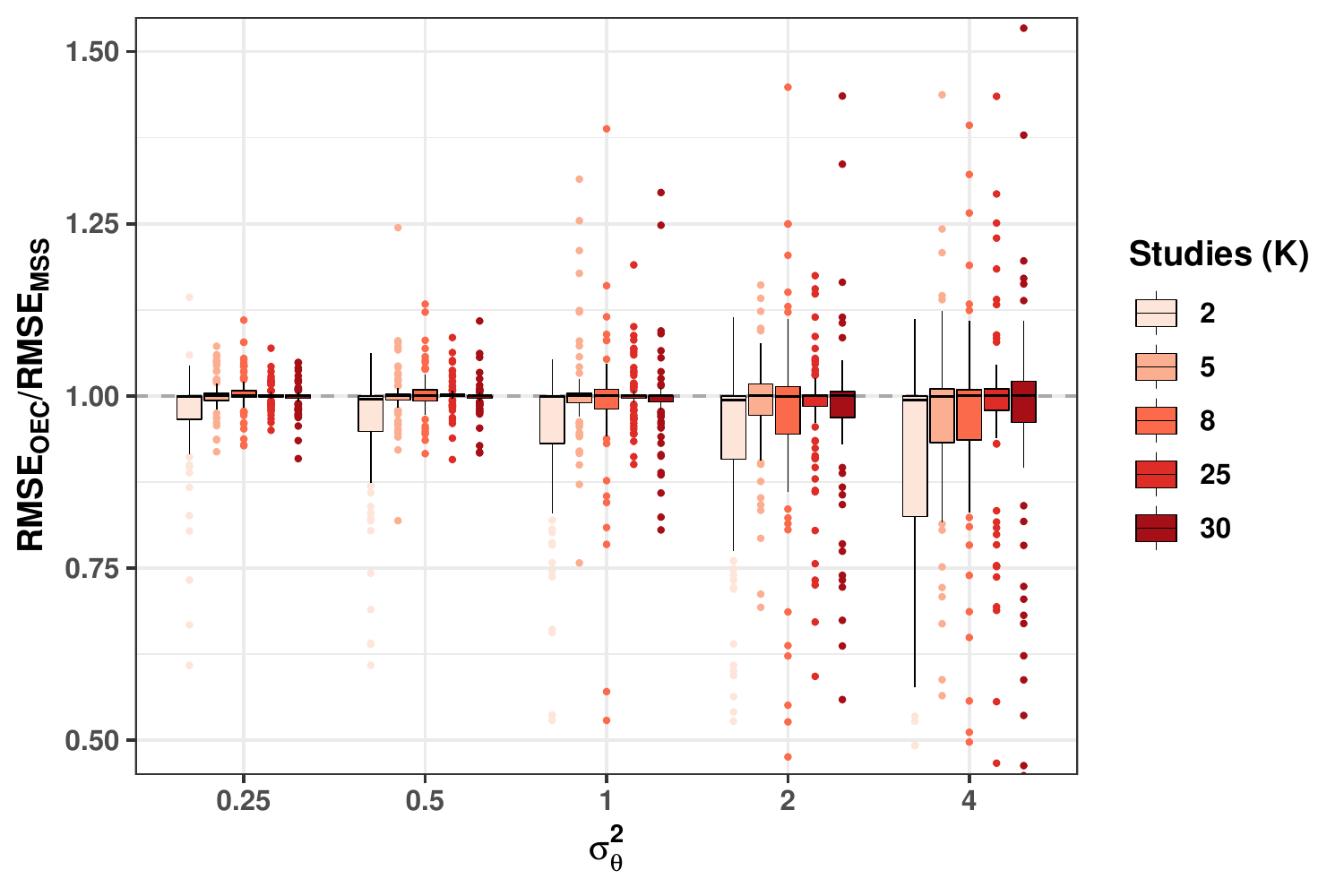}
		\caption{No Data Reuse}
	\end{subfigure}
	\caption{Performance of the OEC$^{\text{S}}$ and OEC$^{\text{SN}}$ compared to their MSS counterparts in data-driven simulations. Figures are zoomed in to easily visualize differences.}
\end{figure}

\begin{table}[h]
 \centering
\begin{minipage}[t]{1.0\linewidth} 
\centering
\begin{tabular}{r|rrrrr|rrrrr}
\toprule
\toprule
 \multicolumn{1}{c}{\footnotesize\bfseries Parameters} \vline & 
\multicolumn{5}{c}{\footnotesize\bfseries $\mathbf{OEC^{\text{S}}}$ vs. $\mathbf{MSS^{\text{S}}}$} \vline &
    
\multicolumn{5}{c}{\footnotesize\bfseries $\mathbf{OEC^{\text{SN}}}$ vs. $\mathbf{MSS^{\text{SN}}}$}  \\
  &
  \multicolumn{5}{c}{\footnotesize\bfseries $\sigma^2_{\theta}$} \vline  &
   \multicolumn{5}{c}{\footnotesize\bfseries $\sigma^2_{\theta}$}  \\
 $K$ & 0.25 & 0.5 & 1.0 & 2.0 & 4.0 & 0.25 & 0.5 & 1.0 & 2.0 & 4.0\\
\midrule
\midrule
2 & 0.94 & 0.99 & 0.98 & 0.97 & 0.97 & 0.97 & 0.95 & 0.94 & 0.92 & 0.89\\
5 & 0.93 & 0.94 & 0.95 & 0.97 & 0.99 & 1.00 & 1.00 & 1.00 & 1.00 & 0.98\\
8 & 0.93 & 0.93 & 0.94 & 0.94 & 0.99 & 1.00 & 1.00 & 0.99 & 0.97 & 0.98\\
25 & 0.94 & 0.92 & 0.95 & 0.95 & 0.98 & 1.00 & 1.00 & 1.00 & 0.98 & 0.98\\
30 & 0.90 & 0.95 & 0.95 & 0.97 & 1.00 & 1.00 & 1.01 & 0.99 & 0.98 & 0.98\\
\bottomrule
\end{tabular}
\end{minipage}
\caption{Average data-driven simulation performance: $RMSE_{OEC} / RMSE_{MSS}$. Columns indicate $\sigma^2_{\theta}$ and rows indicate $K$, number of studies. Monte carlo error was at most 0.002.}
\label{tab:mortSims1}
\end{table}

\begin{table}[h]
  \centering
\begin{minipage}[t]{0.75\linewidth} \centering
\begin{tabular}{rr|rrrr}
\toprule
\toprule
\multicolumn{2}{c}{\footnotesize\bfseries ~~~Parameters}  \vline  & 
     \multicolumn{1}{c}{\footnotesize\bfseries $\mathbf{OEC^{\text{S}}}$} 
     &
     \multicolumn{1}{c}{\footnotesize\bfseries $\mathbf{MSS^{
     \text{S}}}$} &
     \multicolumn{1}{c}{\footnotesize\bfseries $\mathbf{OEC^{\text{SN}}}$} 
      &
      \multicolumn{1}{c}{\footnotesize\bfseries $\mathbf{MSS^{
     \text{SN}}}$ } \\
  \multicolumn{1}{c}{\footnotesize\bfseries $K$} &
   \multicolumn{1}{c}{\footnotesize\bfseries ~~~~~$\sigma^2_{\theta}$}  \vline &
\multicolumn{4}{c}{\footnotesize  ~~~~~vs. Study-Specific Model   ~~~~~} \\
   
\midrule
\midrule
2 & 0.25 & 0.91 & 0.96 & 0.82 & 0.85\\
2 & 1.00 & 0.96 & 0.98 & 0.96 & 1.05\\
2 & 4.00 & 0.96 & 0.98 & 1.06 & 1.24\\
\addlinespace
8 & 0.25 & 0.86 & 0.94 & 0.80 & 0.80\\
8 & 1.00 & 0.88 & 0.94 & 0.81 & 0.82\\
8 & 4.00 & 0.90 & 0.92 & 0.89 & 0.93\\
\addlinespace
30 & 0.25 & 0.83 & 0.92 & 0.81 & 0.81\\
30 & 1.00 & 0.85 & 0.90 & 0.84 & 0.84\\
30 & 4.00 & 0.90 & 0.89 & 0.89 & 0.91\\
\bottomrule
\end{tabular}
\end{minipage}
\caption{Average data-drive simulation performance without regularization at different $K$ and $\sigma^2_{\theta}$. Performance of each method is $RMSE_{OEC^{\text{S}}} / RMSE_{SSM}$, where $SSM$ denotes a study-specific model, fit only on the target study. Monte carlo error was at most 0.0073.
}
\label{table:mortSims2_zero}
\end{table}

\section{General Simulations}

\subsection{Design of Simulation Experiments}

We next report on a second set of simulations whose goal is to characterize the performance of both the ``generalist'' and ``specialist'' algorithms, outside of a time series setting. We simulated datasets to characterize the effect of three features of multi-study settings on the performance of our proposed methods: 1) covariate-shift (heterogeneity in $f_{x_k}(\mathbb{X}_k)$ across studies), 2) concept-shift (heterogeneity in $f_{y_k \given x_k } (\mathbf{y}_k \given \mathbb{X}_k)$ across studies), and 3) study clusters across which (1) and (2) vary. We generated clusters as in \citep{Loewinger}: groups of studies that shared similar distributions, $f_{\mathbb{X}_k}(\mathbb{X}_k)$ and $f_{\mathbf{y}_k\given\mathbb{X}_k}(\mathbf{y}_k \given \mathbb{X}_k)$. To control both within and between-cluster heterogeneity in $f_{\mathbf{y}_k\given\mathbb{X}_k}(\mathbf{y}_k \given \mathbb{X}_k)$, we simulated $f_{y_k \given x_k }(\mathbf{y}_k \given \mathbb{X}_k)$ from a linear mixed effects model that included both cluster-specific and study-specific random effects:
$$\mathbf{y}_{k} = \mathbb{X}_k (\boldsymbol{\theta}_k + \boldsymbol{\delta}_{c}) + \boldsymbol{\epsilon}_{k},$$ 
where study $k$ is in cluster $c$, $\boldsymbol{\delta}_{c} \in \mathbb{R}^{p+1}$ is a cluster-specific random effect and ${\boldsymbol{\theta}}_k \in \mathbb{R}^{p+1}$ is a study-specific random effect. We independently drew $\boldsymbol{\delta}_{c} \given \boldsymbol{\mu}_{\delta} \sim N_{p+1}(\boldsymbol{\mu}_{\delta},  \sigma^2_{\delta} \mathbb{I})$ and $\theta_{k,j} \sim \mbox{Unif}(-\sigma^2_{{\delta}}/20,~ \sigma^2_{{\delta}}/20)$ so that that random effects varied across studies within a cluster by a degree proportional to the between-cluster heterogeneity. The vector of random effects ${\boldsymbol{\delta}}_c$ is centered at the fixed effects, $\boldsymbol{\mu}_{\delta} \in \mathbb{R}^{p+1}$, where we independently drew $\mu_{\delta,j} \sim \mbox{Unif}(-2,2)$. We simulated the random effects to be independent of the error term: ${\boldsymbol{\theta}}_k \perp \!\!\! \perp \boldsymbol{\epsilon}_k$ and ${\boldsymbol{\delta}}_c \perp \!\!\! \perp \boldsymbol{\epsilon}_k$. As above, $\mathbb{X}_k$ includes a column of ones so that the model includes an intercept. We conditionally drew $\boldsymbol{\epsilon}_{ki} \given \sigma^2_{\epsilon_k} \sim N_{n_k}(0, \sigma^2_{\epsilon_k} \mathbb{I}_{n_k})$ where $\sigma^2_{\epsilon_k} \sim \mbox{Unif}(1,~2)$, inducing heterogeneity between studies in the variance of the residuals. We drew $n_k \sim \mbox{Unif}(150,~300)$ (discrete uniform). $n_{k^*} = 50$ (the training set for the target study). $n_{test} = 100$ (test set for each iteration). We selected these values to ensure that  $n_k > n_{k^*}$. These sample size values were motivated by the ratio of $n_k / n_{k^*}$ observed in the application. We selected $K = 5$ since many multi-study settings encountered in practice have few training studies available. We selected $p = 20$ and set $10$ coefficients to 0, to model sparsity in the true model coefficients, as relevant in many prediction settings. 

We simulated the covariates so that the marginal distribution of the covariates, $f_{\mathbb{X}_k}(\mathbb{X}_k)$, differed across studies and clusters. We drew a vector of covariates for observation $i$ of study $k$ as $\mathbf{x}_{k,i} \given \boldsymbol{\mu}_{X_k} \sim N_p (\boldsymbol{\mu}_{X_k}, \Sigma_X)$ where $\Sigma_X$ is a covariance matrix that was randomly generated (i.e., varied across simulation iterations) but was held fixed across studies within an iteration and $\boldsymbol{\mu}_{X_k}$ was a study specific vector of covariate means. To induce covariate shift, we modeled the means of the covariates in study $k$ and cluster $c$ as 
$$\boldsymbol{\mu}_{X_k} = \boldsymbol{\zeta}_c + \boldsymbol{\tau}_k$$ where $ \boldsymbol{\zeta}_c, \boldsymbol{\tau}_k \in \mathbb{R}^{p}$. We drew cluster-specific covariate means, $\boldsymbol{\zeta}_c \given \tilde{\boldsymbol{\mu}} \sim  N_p(\tilde{\boldsymbol{\mu}}, \sigma^2_X \mathbb{I})$ and we independently drew $\boldsymbol{\tau}_k \sim \mbox{Unif}(-0.05, 0.05)$. Therefore $\sigma^2_X$ controls the magnitude of heterogeneity across clusters in the means of the covariates. We independently drew $\tilde{\mu}_j \sim N(5, 10)$. Thus studies differed within and across clusters in the covariate distributions. 

We simulated sets of studies both with and without study clusters. In the ``no cluster'' case, we drew each training and test study to be in a separate cluster ($C=6$). In the ``cluster'' case, we generated three clusters with two studies per cluster ($C=3$). The test study was also generated as belonging to one of the same three clusters.



    
We simulated 100 iterations for each set of simulation parameters, each iteration consisting of a set of training studies and a test set. Each iteration therefore produced an RMSE for each method. We show the distribution of these RMSEs across the 100 iterations in figures below.


We used both Ridge regression and OLS as the single-study learners. We tuned model parameters as in the COVID-19-driven simulations. We present figures and results without study-specific Ridge penalties in Supplemental Section \ref{sims23_supplement}.

\subsection{Results}

We present results comparing each OEC approach to the corresponding two-stage stacking approach, to investigate whether jointly training an ensemble in this framework is superior to training the ensemble with MSS. We included a subset of the results in Table \ref{table:generalSims_v2} and the full version in Supplemental Table \ref{table:generalSims_v2_supplement_cvCF}. We also present the results relative to a common baseline in Supplemental Table~\ref{table:generalSims_full_cvCF}.


\begin{figure}[h]
	\centering
	\begin{subfigure}[t]{1\textwidth}
		\centering
		\includegraphics[width=1.0\linewidth]{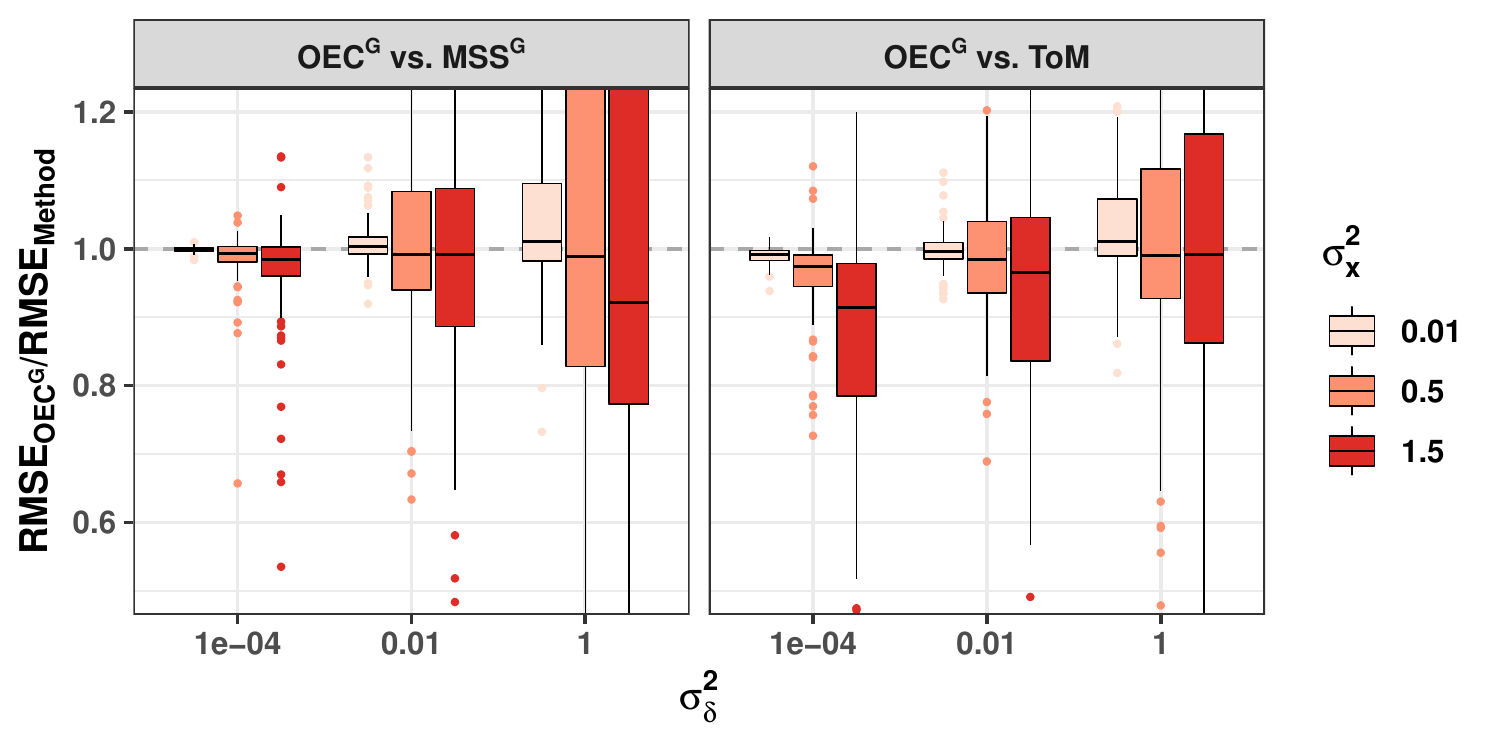}
		\caption{No clusters}
	\end{subfigure}
	\hfill
	\begin{subfigure}[t]{1\textwidth}
		\centering
		\includegraphics[width=1.0\linewidth]{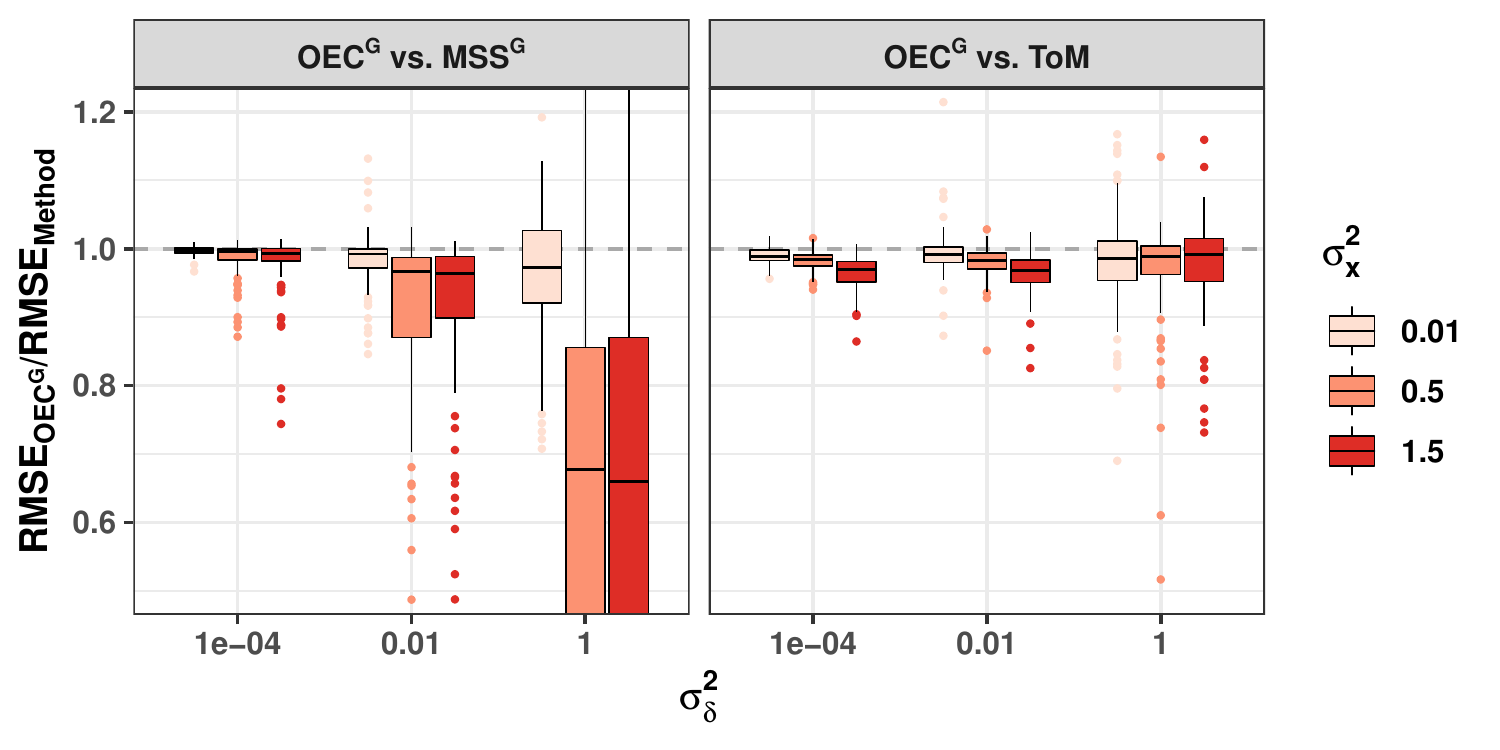}
		\caption{Clusters}
	\end{subfigure}
	\caption{Performance of the OEC$^{\text{S}}$ compared to MSS$^{\text{G}}$ and the ToM algorithm. Figures are zoomed in to easily visualize important differences.}
\end{figure}

\begin{table}[h]
 \centering
\begin{minipage}[t]{1.0\linewidth} 
\centering
\begin{tabular}{rrr|rr|rrrr}
\toprule
\toprule
\multicolumn{3}{c}{\footnotesize\bfseries ~~~Parameters} \vline &
\multicolumn{1}{c}{\footnotesize\bfseries $\mathbf{OEC^{\text{G}}}$} &
    \multicolumn{1}{c}{\footnotesize\bfseries  $\mathbf{MSS^{\text{G}}}$} \vline   &
     \multicolumn{1}{c}{\footnotesize\bfseries $\mathbf{OEC^{\text{S}}}$} 
     &
     \multicolumn{1}{c}{\footnotesize\bfseries  $\mathbf{MSS^{\text{S}}}$} &
     \multicolumn{1}{c}{\footnotesize\bfseries
     $\mathbf{OEC^{\text{SN}}}$} 
      &
      \multicolumn{1}{c}{\footnotesize\bfseries  $\mathbf{MSS^{\text{SN}}}$ } \\
 \multicolumn{1}{c}{\footnotesize\bfseries $C$} &
  \multicolumn{1}{c}{\footnotesize\bfseries $\sigma^2_X$} &
   \multicolumn{1}{c}{\footnotesize\bfseries $\sigma^2_{\delta}$}  \vline &
      \multicolumn{2}{c}{\footnotesize vs. ToM~~~~~}
\vline &
\multicolumn{4}{c}{\footnotesize  ~~~~~vs. Study-Specific Model   ~~~~~} \\
   
\midrule
\midrule
3 & 0.01 & 0.01 & 0.99 & 1.00 & 0.87 & 0.99 & 0.84 & 0.84\\
3 & 1.50 & 0.01 & 0.96 & 1.07 & 0.85 & 0.99 & 0.85 & 0.83\\
3 & 0.01 & 1.00 & 0.98 & 0.99 & 0.94 & 0.99 & 0.89 & 0.85\\
3 & 1.50 & 1.00 & 1.02 & 1.85 & 1.01 & 0.99 & 0.99 & 0.85\\
\addlinespace
6 & 0.01 & 0.01 & 1.00 & 0.99 & 0.89 & 0.98 & 0.86 & 0.85\\
6 & 1.50 & 0.01 & 0.97 & 1.00 & 0.89 & 0.98 & 0.86 & 0.85\\
6 & 0.01 & 1.00 & 1.06 & 0.99 & 1.01 & 0.99 & 1.03 & 1.75\\
6 & 1.50 & 1.00 & 1.12 & 1.15 & 1.03 & 0.99 & 1.06 & 1.74\\
\bottomrule
\end{tabular}
\end{minipage}
\caption{Average simulation performance with ($C = 3$) and without ($C = 6$) clustering at varying degrees of covariate-shift ($\sigma^2_{X}$) and variance of random effects ($\sigma^2_{\delta}$). Each section indicates the performance of the method (e.g., OEC$^{\text{G}}$) relative to a baseline of the ToM algorithm or a study-specific model respectively (e.g., RMSE$_{\text{OEC}^{\text{G}}}$ / RMSE$_{\text{ToM}}$, RMSE$_{\text{OEC}^{\text{S}}}$ / RMSE$_{\text{MSS}^{\text{S}}}$). Monte Carlo error was at most 0.003.}
\label{table:generalSims_v2}
\end{table}

The OEC$^{\text{G}}$ outperformed the MSS$^{\text{G}}$ across all values of both $\sigma^2_X$ and $\sigma^2_{\delta}$, when there was clustering in the studies. The OEC$^{\text{G}}$ outperformed the ToM  across all values of $\sigma^2_X$ and $\sigma^2_{\delta}$ except for very high levels of $\sigma^2_{\delta}$, where the two algorithms were comparable. When there was no clustering, the OEC$^{\text{G}}$ outperformed both the ToM and the MSS$^{\text{G}}$ at lower values of $\sigma^2_{\delta}$ and was comparable to both algorithms at higher values of $\sigma^2_{\delta}$.

The OEC$^{\text{S}}$ exhibited superior performance compared to the MSS$^{\text{S}}$ in most settings explored. However, when at high levels of $\sigma^2_{\delta}$ and when the data exhibited no clustering, the algorithms performed comparably. The OEC$^{\text{SN}}$ tended to perform better when there was no clustering and there was higher variance of the random effects. When the data exhibited clustering, the MSS$^{\text{SN}}$ performed comparably or slightly better than its OEC counterpart in all settings explored.
    
While covariate shift appeared to influence the relative performance of the methods in the generalist setting, it did not appear to explain much of the variability in performance of the OEC in the specialist case (i.e., OEC$^{\text{S}}$ and OEC$^{\text{SN}}$). At a fixed value of $\sigma^2_{\delta}$, the performance of the OEC$^{\text{G}}$ relative to the MSS$^{\text{G}}$ varied as a function of $\sigma^2_X$. It appeared the relative performance of the OEC$^{\text{G}}$ improved as a function of $\sigma^2_X$. However, it appears the OEC$^{\text{S}}$ performed better relative to its MSS counterpart when there were lower levels of covariate shift.

These results demonstrate the benefits of an all-in-one approach. The OEC$^{\text{G}}$ often outperformed both MSS$^{\text{G}}$ and the ToM. While MSS$^{\text{G}}$ often performed well compared to the ToM, it suffers from cases where it is vastly outperformed by the ToM. For example, when $\sigma^2_{\delta}$ and $\sigma^2_{X}$ were high and there was clustering in the studies ($C = 3$), the generalist approach exhibited an average RMSE about $80\%$ higher than the ToM. This echos the frequent empirical observation as well as insights from analytical work comparing ToM and ensembling \citep{ren, Guan}. The OEC$^{\text{G}}$ approach however, does not suffer from cases where it is vastly outperformed by the standard benchmarks; indeed, the OEC$^{\text{G}}$ was comparable to the ToM in that simulation setting. The ToM modestly outperformed the OEC$^{\text{G}}$ on average when $\sigma^2_{\delta}$ was high and the studies did not exhibit clustering, but this appeared to be driven by outliers.

The performance of the OEC$^{\text{S}}$ and OEC$^{\text{SN}}$ further demonstrate the utility of jointly estimating model parameters and ensemble weights. The OEC$^{\text{S}}$ strongly outperforms the study-specific model and the MSS$^{\text{S}}$ (by as much as about $15\%$). While the OEC$^{\text{SN}}$ performs comparably or slightly worse than the MSS$^{\text{SN}}$ method in some settings, it importantly never performs worse than a study-specific model. When studies do not cluster and $\sigma^2_X$ and $\sigma^2_{\delta}$ are high, the MSS$^{\text{SN}}$ method exhibits an RMSE about $75\%$ higher than a simple study-specific model. In these cases, OEC$^{\text{SN}}$ still remains superior to or competitive with the study-specific model and the MSS$^{\text{SN}}$. 

Taken together, these simulations demonstrate that the performance of the OEC was often comparable or superior to MSS in both generalist and specialist settings. Additionally, while MSS methods sometimes exhibited very poor performance compared to benchmark approaches, the OEC methods consistently outperformed these methods in these cases, highlighting the robustness of these approaches.

\begin{figure}[H]
	\centering
	\begin{subfigure}[t]{1\textwidth}
		\centering
		\includegraphics[width=1.0\linewidth]{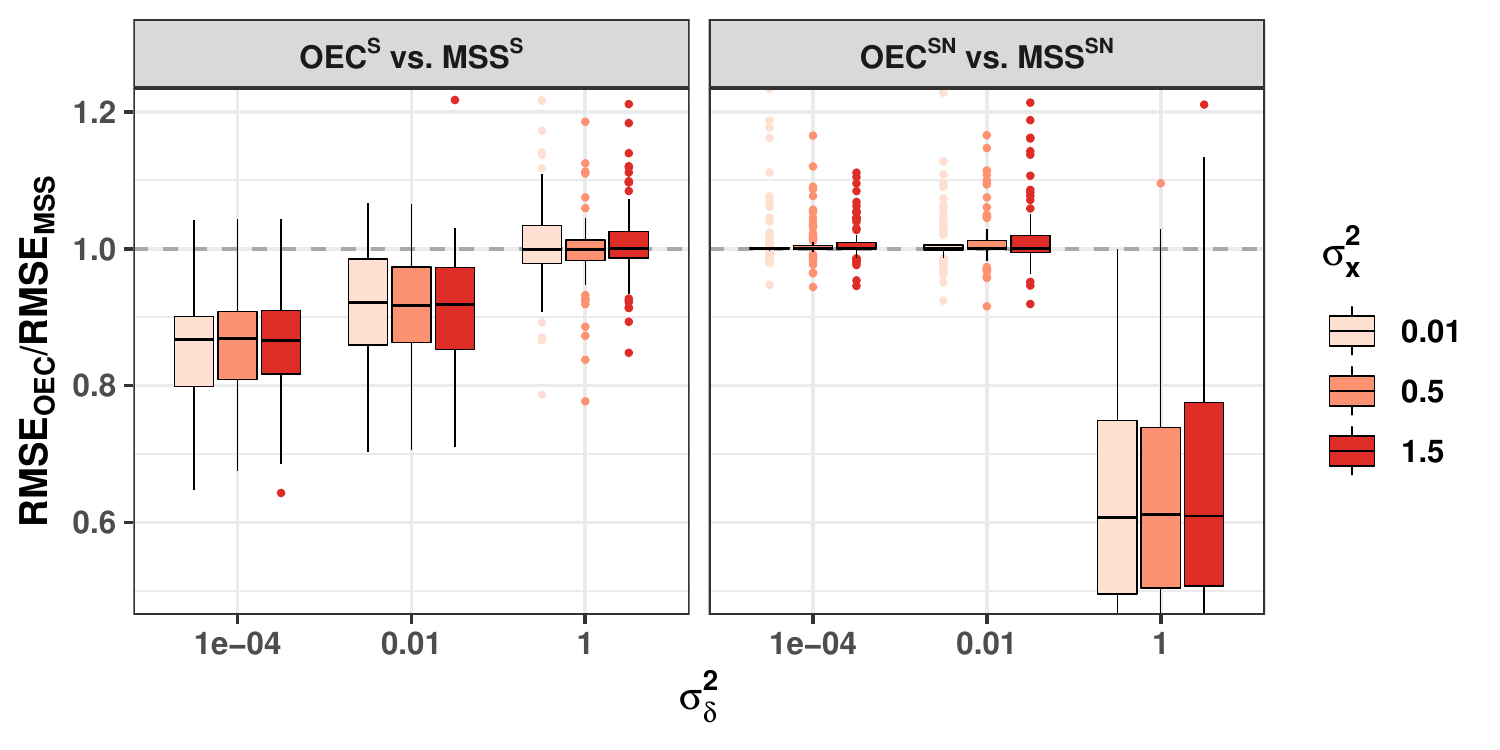}
		\caption{No clusters}
	\end{subfigure}
	\hfill
	\begin{subfigure}[t]{1\textwidth}
		\centering
		\includegraphics[width=1.0\linewidth]{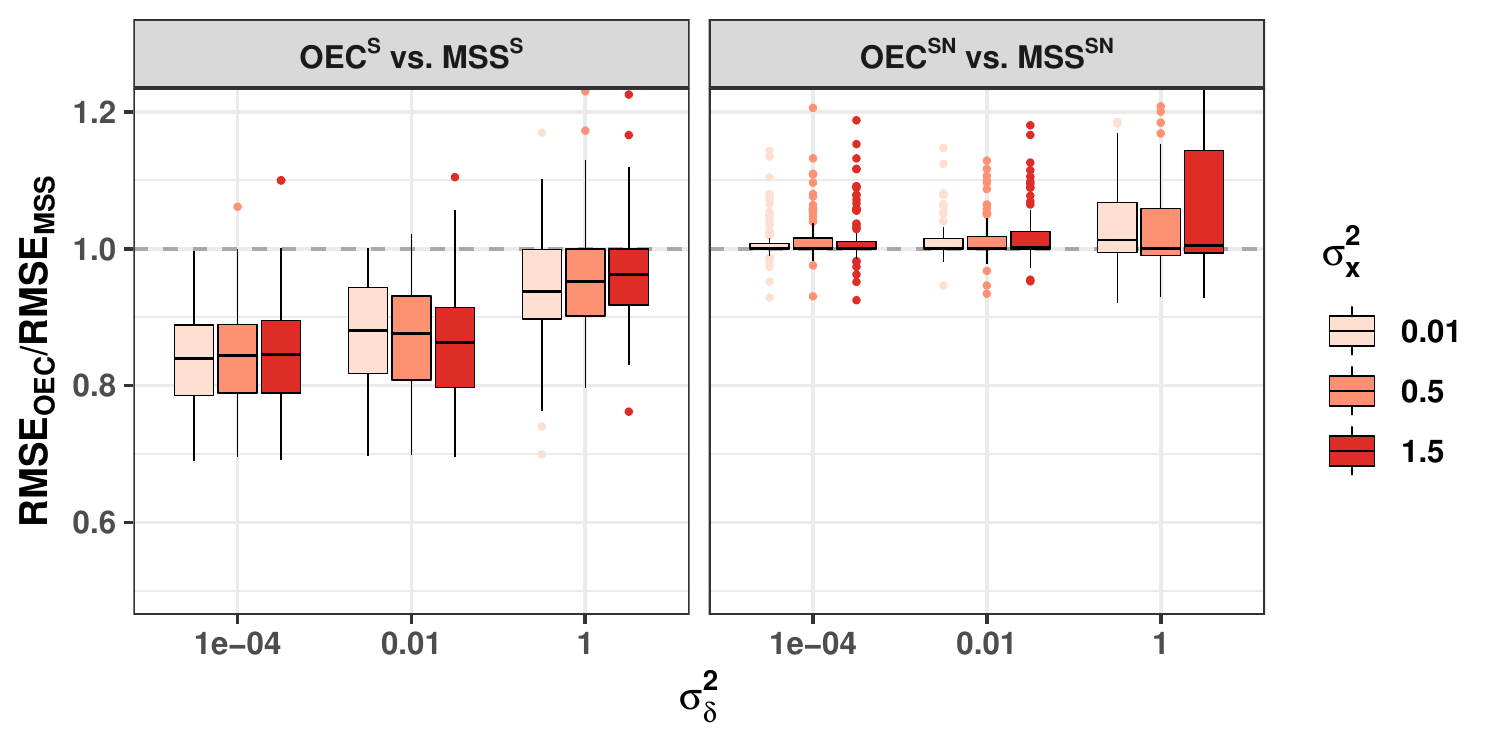}
		\caption{Clusters}
	\end{subfigure}
	\caption{ Performance of the OEC$^{\text{S}}$ and OEC$^{\text{SN}}$ compared to their MSS counterparts. Figures are zoomed in to easily visualize differences.}
\end{figure}

\section{Discussion}
\label{discussion}
We propose and evaluate Optimal Ensemble Construction (OEC), a flexible approach that can be used to construct ensemble learners in domain generalization and transfer learning. OEC generalizes two-stage multi-study stacking in cases where the individual learners are finitely parameterized, by specifying an explicit loss function. OEC improved prediction performance of both specialist multi-study stacking (for transfer learning) and generalist multi-study stacking (for domain generalization). We observed the most consistent gains in the transfer learning setting.

In our application, we showed how leveraging external sources of data through the OEC can be used to improve the accuracy of excess mortality estimation when countries do not have sufficient pre-disaster data. Even for countries that have sufficient training data, our method may prove useful when analyzing mortality (or any other vital statistics outcome) stratified by demographic indicators, or subregions of a country \citep{islam2021excess} since it is common to encounter small counts for the outcome when conducting analyses within substrata. In such cases, it may prove beneficial to borrow information from neighboring regions, or other demographic groups. More broadly, using multi-study methods in counterfactual estimation may be a promising area of future research in settings where data sources for counterfactual estimation are limited. The work also adds to the growing body of literature exploring the application of ensembling methods in COVID-19-related prediction problems. For example, ensembling techniques have proven useful for predicting cases \citep{covidEnsemble1, covidEnsemble2}. 

Recent work has proposed gradient boosted trees to directly estimate excess mortality during the COVID-19 pandemic based upon a large set of covariates such as COVID-19 case counts and demographic variables \citep{economist}. The authors proposed to use data from many countries during model training and motivated their methods by the need for excess mortality predictions in countries which have not made estimates available. They propose to pool data from different countries before fitting one global model to predict \textit{excess} mortality given a set of country-specific covariates. This differed from our method which generated an ensemble of country-specific models to predict \textit{baseline} mortality based upon historical mortality trends. 


Our method is not without its limitations. A disadvantage of combining the study-specific and ensemble weighting loss functions is that non-convexity arises from the resulting bilinear terms. As a result, the optimization approach implemented here is not guaranteed to converge to the global minimum and requires careful initialization. We found that initializing at MSS estimates yields consistently high prediction performance in practice and appeared to outperform other heuristics such as random restarts. Thus, the optimization procedure will always be more computationally expensive than multi-study stacking. However, in all the settings explored, the optimization procedure converged within a couple seconds and allowed for tuning over large grids of hyperparameter values. It is possible to obtain global solutions to the problems by using modern techniques in mixed integer programming~\citep{bertsimas2017certifiably}. Such techniques can deliver optimality certificates for the associated optimization problems but would likely be much slower compared to the methods we present here. 

The OEC method that we proposed here lend itself to future extensions. For example, we have introduced the OEC in the regularized linear-linear setting, but one could easily replace the linear models with other specifications (e.g., support vector machines) at either the SSL stage, the ensembling stage or both. Also, loss functions could be replaced by penalized negative log-likelihoods reflecting distributions other than the Gaussian. Similarly, while the propositions focus on the linear models case, analogous results can be shown for the broader class of generalized linear models. However, we keep the analysis focused on the linear case, empirically explored in the present work. In addition, some of the results above can be recast to account for models with regularization, but these generalizations require additional regularity conditions and may not provide further insight into the connections between the OEC and earlier multi-study methods. 

The methods presented here can complement many methods proposed in the rich literature of transfer learning, domain generalization and multi-task learning \citep{Zhang, Zhuang, Farahani}. A number of methods proposed in these bodies of work seek to jointly train ensembles using matrix decomposition methods of model parameters and matrix regularization schemes. Indeed, such methods can be used in conjunction with our framework since they focus on improving prediction performance through the model parameter estimation procedure, not the ensemble weighting scheme. In this sense, our work complements many cutting edge methods for multi-source data integration.

In summary we propose a flexible generalization of multi-study stacking that yields improvements in prediction performance in both ``generalist'' and ``specialsit'' implementations. We hope the present work will be beneficial both methodologically, by complementing existing methods, as well as in improving excess mortality estimation in low counts settings.


\section{Software and Reproducibility}
\label{sec5}

Code and instructions to reproduce analyses, figures and tables are available at: \url{https://github.com/gloewing/OEC}. This contains code to tune and fit penalized linear regression problems with all the methods assessed here. 


\section*{Acknowledgments}

GCL was supported by the NIH, F31DA052153; T32 AI 007358. RA was supported by the NIH, T32ES007142. GP was supported by NSF-DMS grants 1810829 and 2113707. RM acknowledges partial research support from NSF-IIS-1718258.
{\it Conflict of Interest}: None declared.

\bibliographystyle{biorefs}
\bibliography{refs}

\newpage
\section{Supplementary Material: Optimization Algorithm}

Denote $f(\mathbb{B}, \boldsymbol{w})$ as the optimization function for the OEC$^{\text{G}}$ method:
    \begin{align*}
    f(\mathbb{B}, \boldsymbol{w}) &= \underset{ \boldsymbol{\alpha}^{\text{G}} \geq \mathbf{0}, ~\alpha_0^{\text{G}} } {\mbox{min}} ~~~ \underset{\mathbb{B}^{\text{G}}} {\mbox{min }} ~ \left\{ \eta~ \left[ \frac{1}{2N} \norm{\boldsymbol{y} - \alpha_0^{\text{G}} \mathbbm{1} - \sum_{k=1}^K \alpha_k^{\text{G}} \mathbb{X} \boldsymbol{\beta}_k^{\text{G}} }_2^2  + \frac{\mu}{2} \norm{\mathbb{D}_{K+1} \boldsymbol{\alpha}^{\text{G}} }_2^2 \right] \right. + \notag \\
    & ~~~~~~~~~~~~ \qquad \qquad \quad
    (1-\eta) \left. \left[ \sum_{k=1}^K \frac{1}{2n_k} \norm{\boldsymbol{y}_k - \mathbb{X}_k \boldsymbol{\beta}_k^{\text{G}}}_2^2 + \frac{1}{2} \sum_{k=1}^K \lambda_k  \norm{\mathbb{D}_{p+1} \boldsymbol{\beta}_k^{\text{G}} }_2^2 \right ] \right\}
    \label{eq:gen}
\end{align*}

We describe the optimization algorithm for this case, with the understanding that the other cases will be similar.
\begin{figure}[H]
    \centering
    \includegraphics[scale = 0.65]{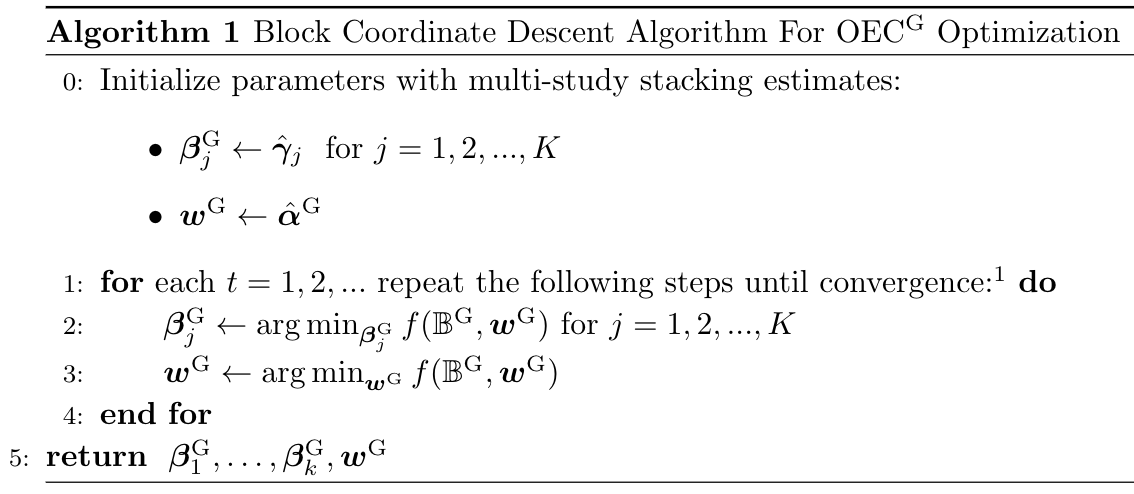}
    \caption{Block coordinate descent algorithm for optimization of the OEC loss. $^1$We terminate the algorithm when the relative difference in objective values across two successive updates in $(\boldsymbol{\beta}_{1}, \ldots, \boldsymbol{\beta}_{k}, \boldsymbol{w}^{\text{G}})$ is smaller than a threshold or the number of iterations is larger than 1000, whichever is sooner.}
    \label{suppFig:opt}
\end{figure}

\section{Supplementary Material: COVID-19 Figures}

\begin{table}[H]
\centering
\resizebox{1.05 \textwidth}{!}{\begin{tabular}{l|rrrrrrrrrrrrrrrrr}
\toprule
\toprule
\multicolumn{1}{c}{\bfseries Method} \vline & \multicolumn{17}{c}{\bfseries Test Year} \\
  & 2003 & 2004 & 2005 & 2006 & 2007 & 2008 & 2009 & 2010 & 2011 & 2012 & 2013 & 2014 & 2015 & 2016 & 2017 & 2018 & 2019\\
\midrule
\midrule
MSS$^{\text{S}}$  & 0.96 & 0.99 & 0.95 & 0.99 & 0.97 & 0.94 & 0.97 & 0.98 & 0.92 & 0.97 & 0.91 & 0.92 & 0.94 & 0.95 & 0.90 & 0.97 & 0.92\\
OEC$^{\text{S}}$ & 0.67 & 0.68 & 0.61 & 0.69 & 0.72 & 0.70 & 0.66 & 0.60 & 0.72 & 0.76 & 0.69 & 0.62 & 0.84 & 0.63 & 0.72 & 0.58 & 0.61\\
MSS$^{\text{SN}}$  & 0.63 & 0.62 & 0.51 & 0.64 & 0.66 & 0.66 & 0.63 & 0.54 & 0.71 & 0.74 & 0.69 & 0.60 & 0.86 & 0.58 & 0.69 & 0.40 & 0.58\\
OEC$^{\text{SN}}$ & 0.63 & 0.64 & 0.49 & 0.62 & 0.66 & 0.65 & 0.63 & 0.53 & 0.71 & 0.74 & 0.68 & 0.60 & 0.86 & 0.57 & 0.70 & 0.38 & 0.52\\
\bottomrule
\end{tabular}}
\caption{Average $RMSE / RMSE_{SSM}$ Performance of ensembling methods relative to a country-specific model (no country-specific Ridge penalty).  Columns indicate test year.}
\label{table:Fullcovid_OLS}
\end{table}

\begin{table}[H]
\centering
\resizebox{1.05 \textwidth}{!}{\begin{tabular}{l|rrrrrrrrrrrrrrrrr}
\toprule
\toprule
\multicolumn{1}{c}{\bfseries Method} \vline & \multicolumn{17}{c}{\bfseries Test Year} \\
  & 2003 & 2004 & 2005 & 2006 & 2007 & 2008 & 2009 & 2010 & 2011 & 2012 & 2013 & 2014 & 2015 & 2016 & 2017 & 2018 & 2019\\
\midrule
\midrule
MSS$^{\text{S}}$ & 0.99 & 1.01 & 1.01 & 1.03 & 0.99 & 1.00 & 0.99 & 1.01 & 1.00 & 1.00 & 1.00 & 1.02 & 0.99 & 1.00 & 1.01 & 1.03 & 1.01\\
OEC$^{\text{S}}$ & 0.77 & 0.88 & 0.71 & 0.82 & 0.87 & 0.83 & 0.79 & 0.67 & 0.88 & 0.86 & 0.86 & 0.75 & 0.94 & 0.72 & 0.88 & 0.64 & 0.71\\
MSS$^{\text{SN}}$ & 0.75 & 0.84 & 0.62 & 0.78 & 0.81 & 0.80 & 0.76 & 0.62 & 0.87 & 0.84 & 0.88 & 0.74 & 0.96 & 0.68 & 0.85 & 0.48 & 0.73\\
OEC$^{\text{SN}}$ & 0.75 & 0.86 & 0.60 & 0.76 & 0.82 & 0.79 & 0.76 & 0.61 & 0.87 & 0.84 & 0.87 & 0.73 & 0.96 & 0.66 & 0.88 & 0.44 & 0.63 \\
\bottomrule
\end{tabular}}
\caption{Average $RMSE / RMSE_{SSM}$ Performance of ensembling methods relative to a country-specific model with a study-specific Ridge penalty.  Columns indicate test year.}
\label{table:covid_Ridge}
\end{table}

\begin{table}[H]
\centering
\begin{tabular}{ll}
\toprule
\toprule
\textbf{Country} & \textbf{Weeks}\\
\midrule
\midrule
Austria & 217\\
Belgium & 261\\
Chile & 547\\
Denmark & 529\\
Ecuador & 179\\
France & 125\\
Germany & 223\\
Iceland & 171\\
Israel & 261\\
Italy & 96\\
Netherlands & 523\\
Norway & 534\\
Peru & 183\\
Portugal & 536\\
Portugal & 536\\
South Africa & 77\\
Spain & 229\\
Sweden & 280\\
Switzerland & 536\\
UK & 536\\
USA & 380\\
\bottomrule
\end{tabular}
\caption{Number of weeks in dataset for each country (including 2020 data)}
\label{table:mortWeeks}
\end{table}

\begin{figure}[h]
	\centering
	\includegraphics[width=0.95\linewidth]{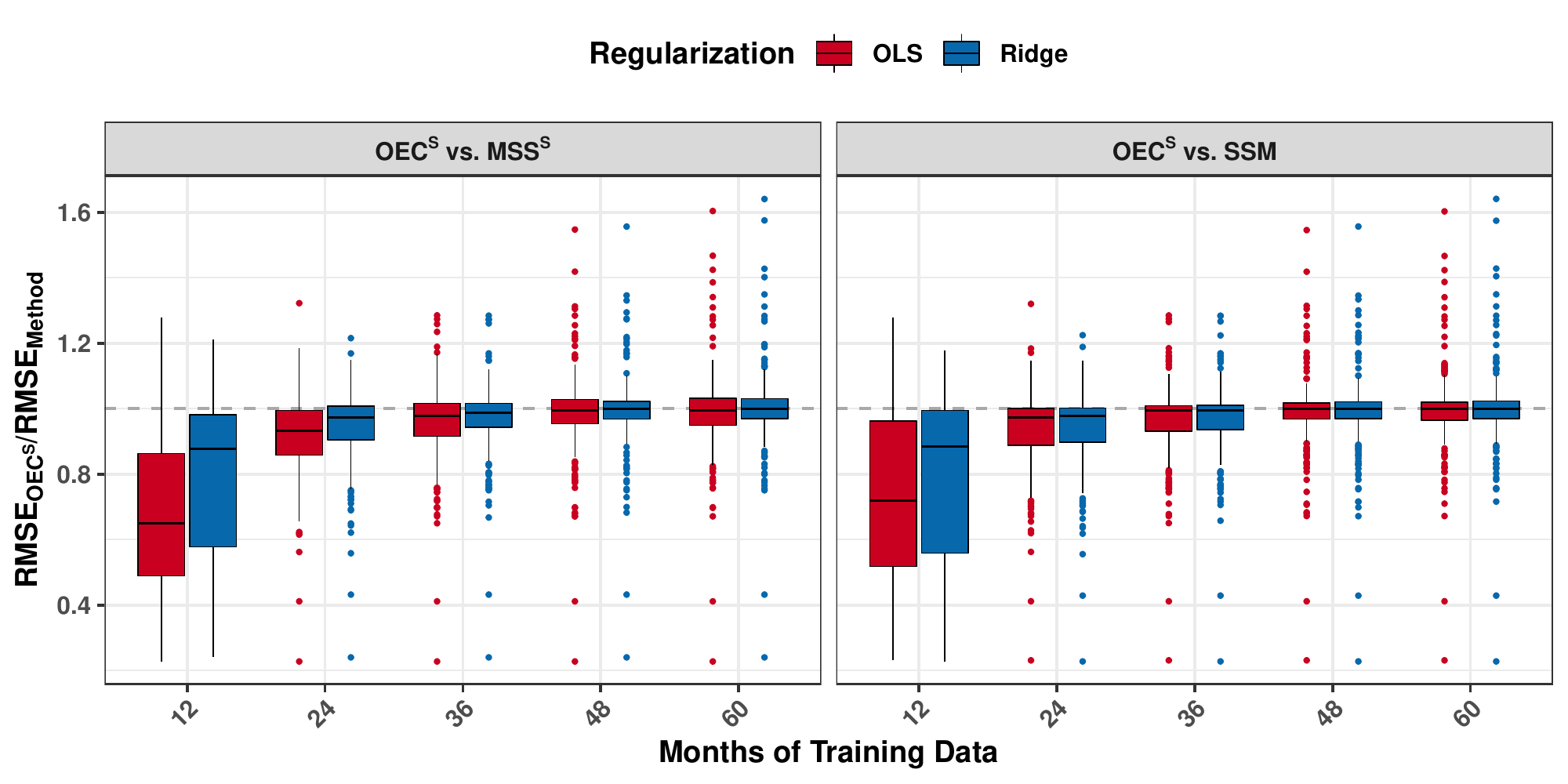}
	
	\caption{$RMSE_{OEC^S} / RMSE_{Method}$. Performance of the OEC$^{\text{S}}$ relative to the SSM and the MSS$^{\text{S}}$ as a function of number of months of training data for target study.}
	\label{suppFig:mortalityTrainingMonths}
\end{figure}

\begin{figure}[h]
	\centering
	\includegraphics[width=0.95\linewidth]{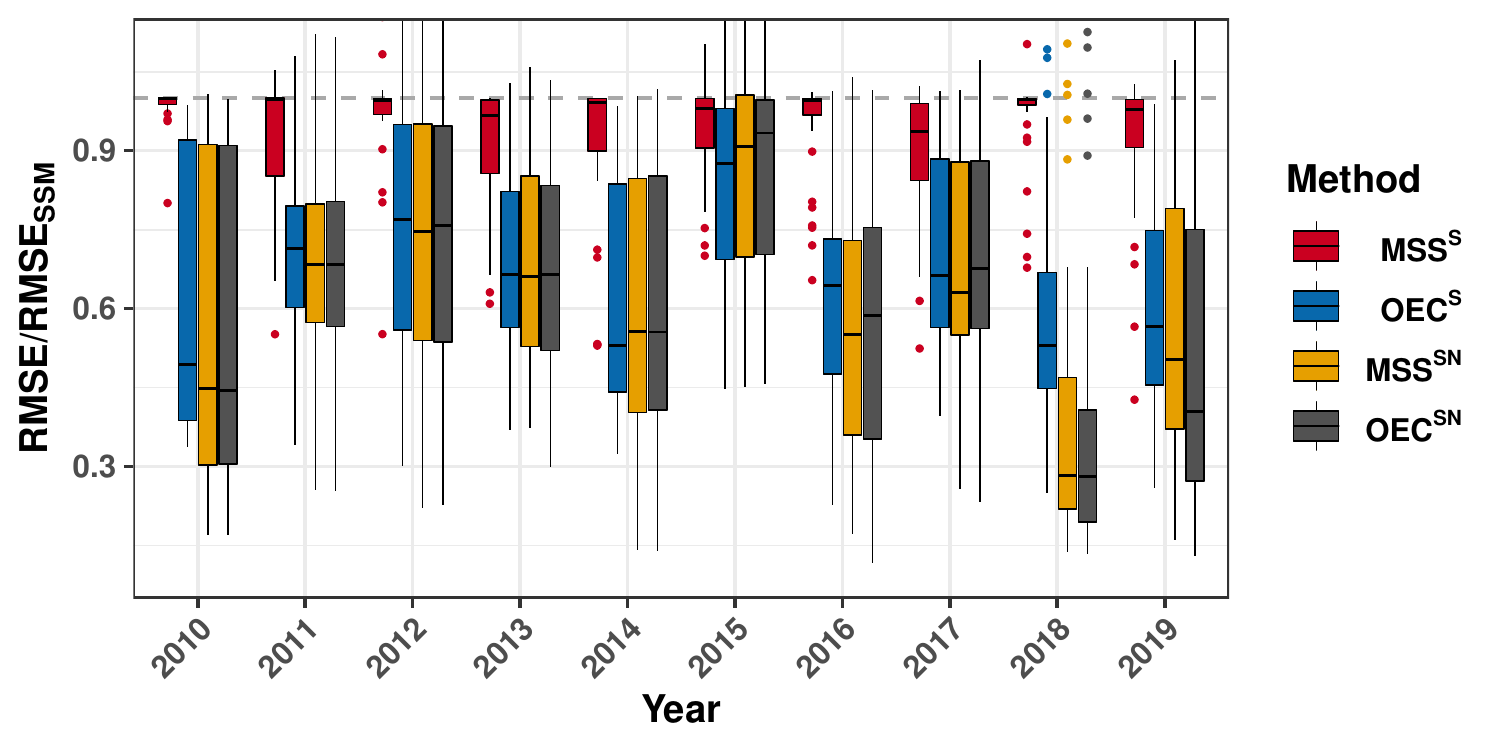}
\caption{$RMSE / RMSE_{SSM}$ Performance of MSS or OEC method relative to a country-specific model (SSM) using an OLS fit.}
	\label{suppFig:mortalityTogether100_OLS}
\end{figure}

\subsection{Comparison with model without linear term for time}

The model with only a seasonal trend and no secular trend (i.e., no linear term for time) is:
\begin{equation}
     Y_{k,t} = \gamma_{k,0} + \sum_{j=1}^2 \left[\gamma_{k,j }\sin \left(\frac{2\pi j t}{52}\right) + \gamma_{k,j+2}\cos \left(\frac{2\pi j t}{52}\right)\right]
    \label{eq:lmSecular}
\end{equation}

\begin{figure}[h]
	\centering
	\includegraphics[width=0.95\linewidth]{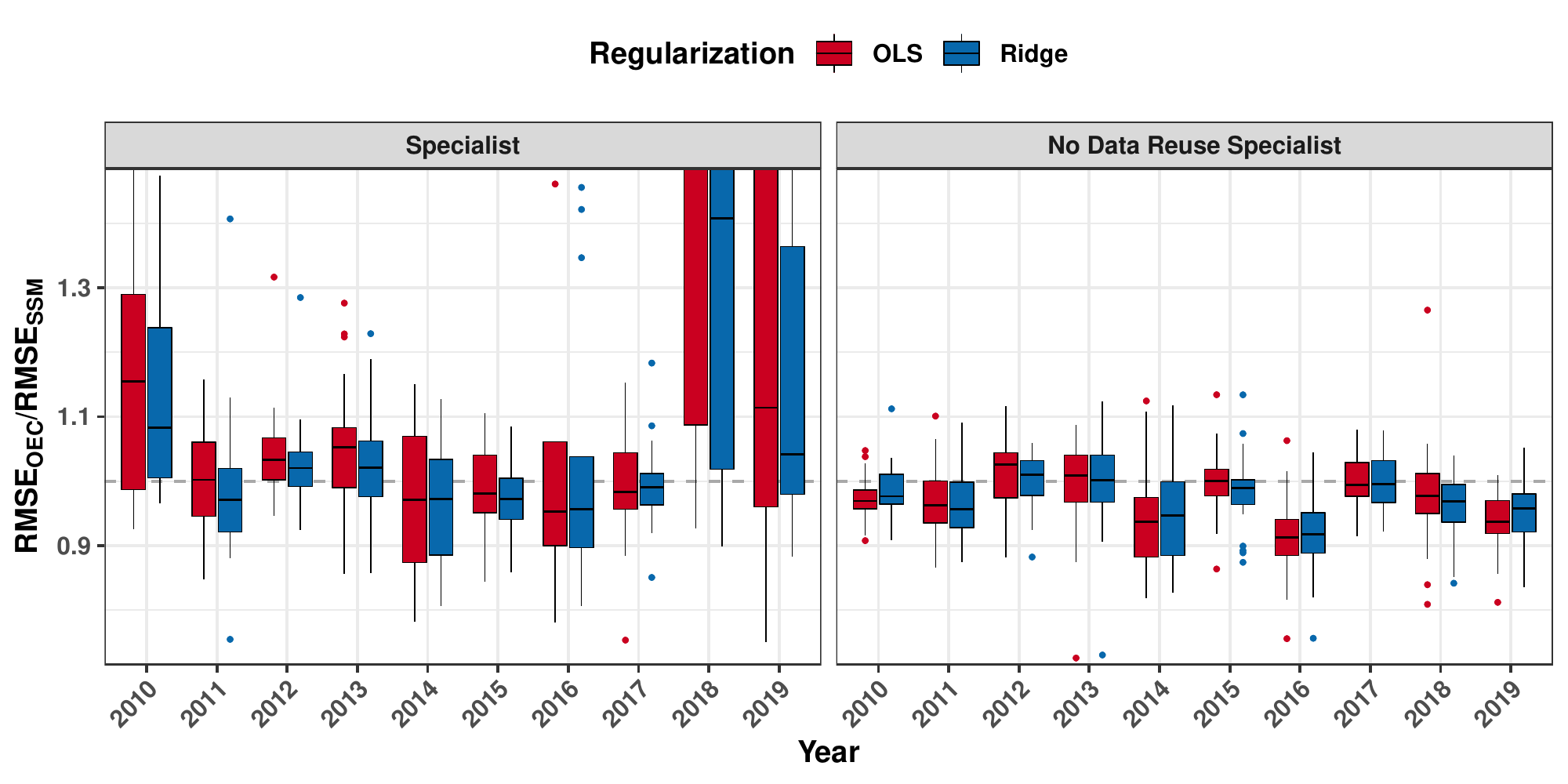}
	\caption{$RMSE_{OEC} / RMSE_{SSM}$ Performance of the OEC method fit with a linear term relative to a study-specific (country-specific) model with no linear term.}
		\label{suppFig:mortalityOECLin_vs_countryNoLin}
\end{figure}

\begin{figure}[h]
	\centering
	\includegraphics[width=0.95\linewidth]{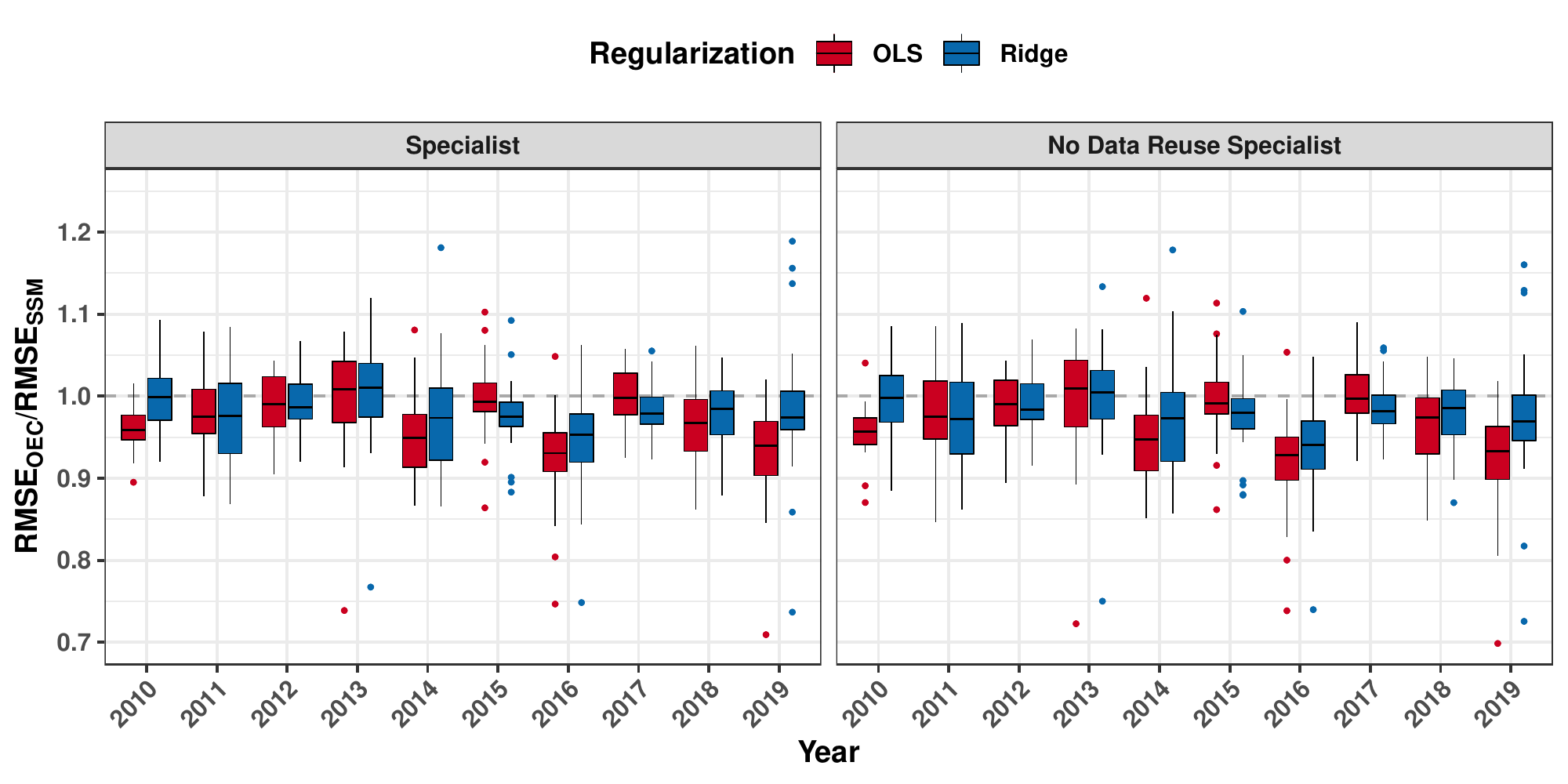}
	\caption{$RMSE_{OEC} / RMSE_{SSM}$ Performance of the OEC method fit with no linear term relative to a study-specific (country-specific) model with no linear term.}
		\label{suppFig:mortalityOECnoLin_vs_countryNoLin}
\end{figure}

\begin{figure}[!htbp]
	\centering
	\includegraphics[width=0.95\linewidth]{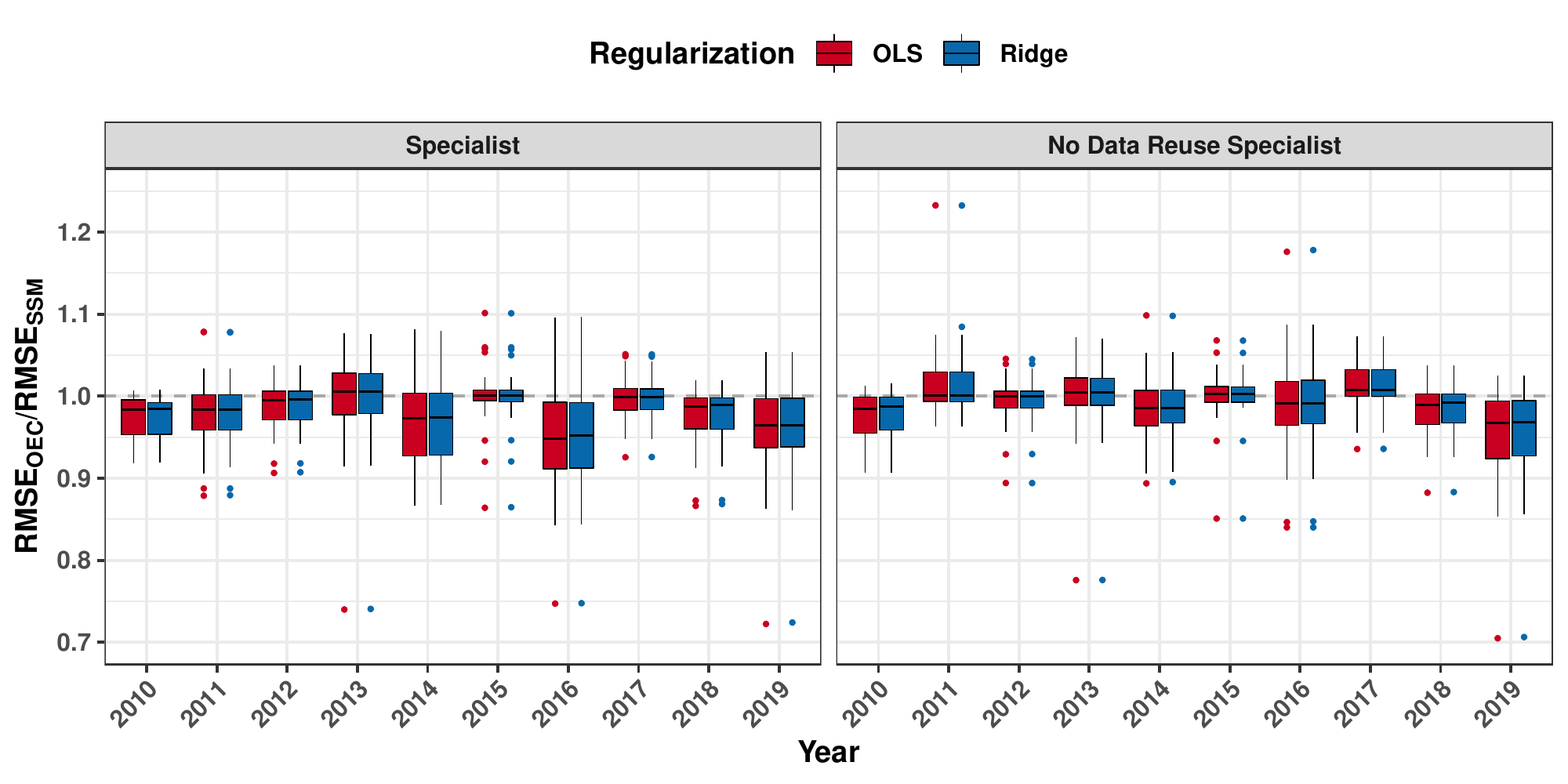}
	\caption{$RMSE_{OEC} / RMSE_{SSM}$. Performance of the OEC method fit with no linear term relative to stacking with no linear term.}
		\label{suppFig:mortalityOECnoLin_vs_stackNoLin}
\end{figure}

\FloatBarrier
\section{Supplementary Material: Data-Driven Simulations}

\begin{figure}[!htbp]
	\centering
	\begin{subfigure}[t]{0.9\textwidth}
		\centering
		\includegraphics[width=1.0\linewidth]{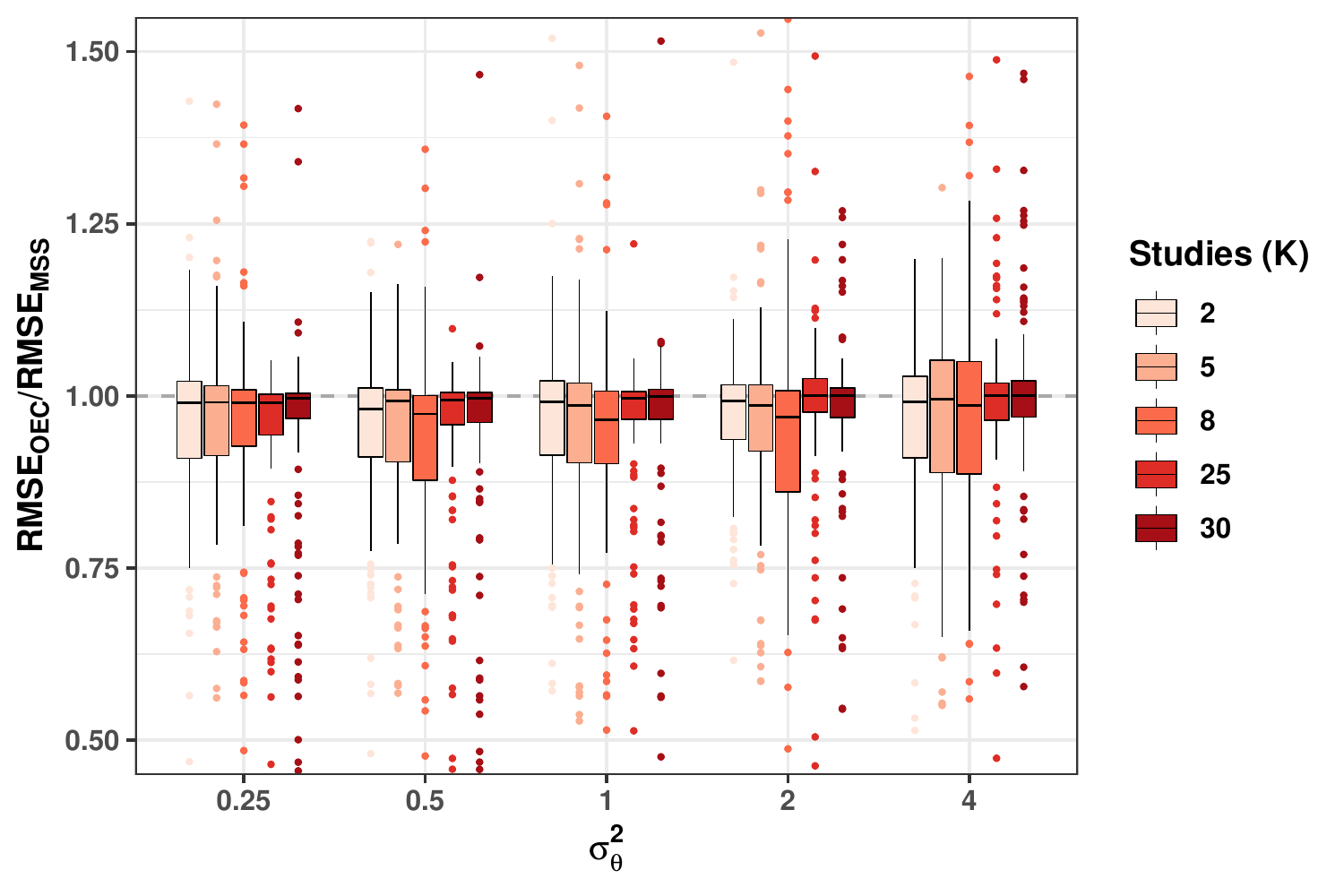}
		\caption{Specialist}
	\end{subfigure}
	\hfill
	\begin{subfigure}[t]{0.9\textwidth}
		\centering
		\includegraphics[width=1.0\linewidth]{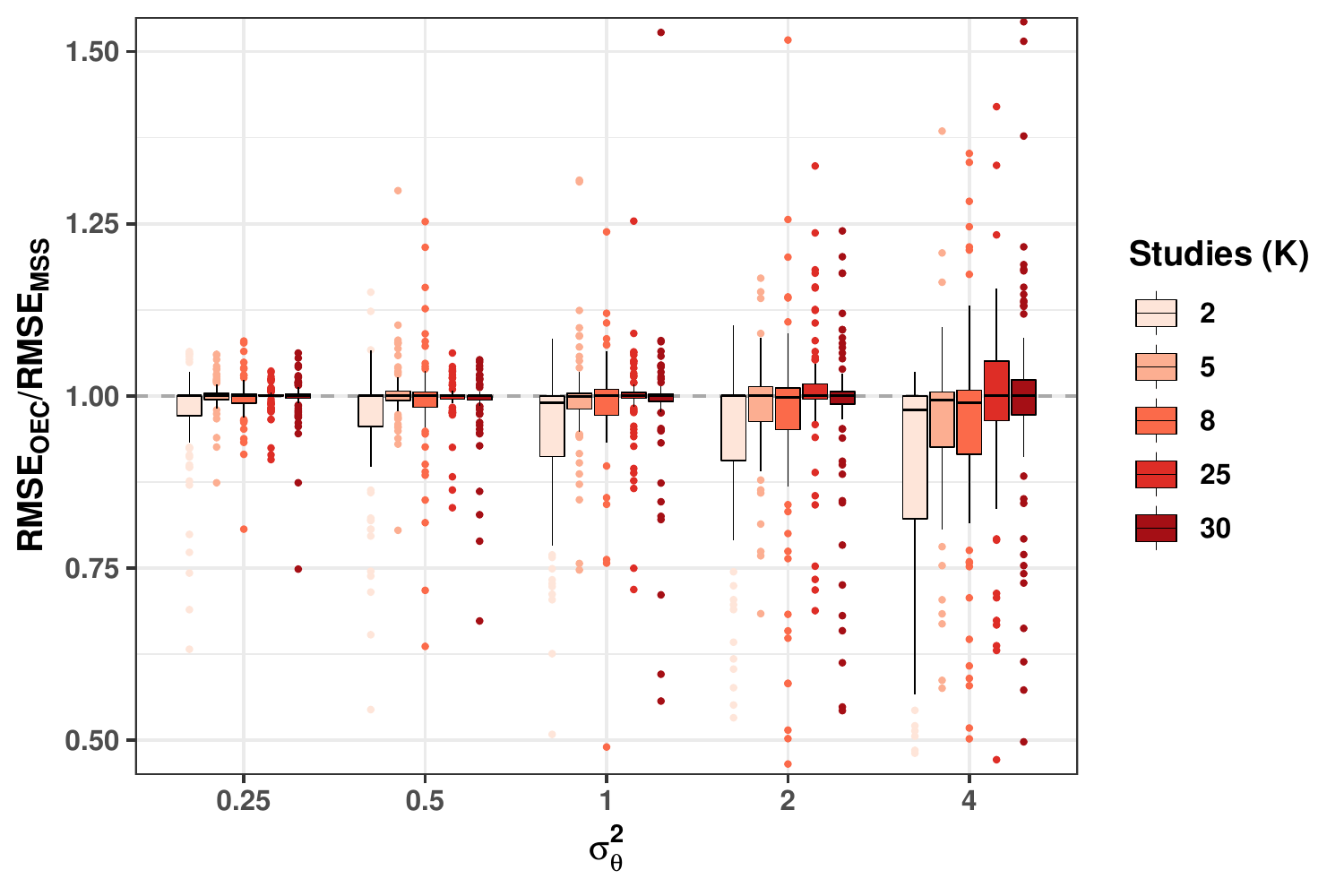}
		\caption{No Data Reuse Specialist}
	\end{subfigure}
	\caption{Performance of the OEC$^\text{S}$ and OEC$^\text{SN}$ compared to their MSS counterparts in data-driven simulations using a Ridge penalty on each study-specific model. Figures are zoomed in to easily visualize differences.}
	\label{fig:mortSims_Ridge}
\end{figure}

\begin{table}[h]
 \centering
\begin{minipage}[t]{1.0\linewidth} 
\centering
\begin{tabular}{r|rrrrr|rrrrr}
\toprule
\toprule
 \multicolumn{1}{c}{\bfseries\footnotesize Parameters} \vline & 
\multicolumn{5}{c}{\bfseries\footnotesize $\mathbf{OEC^{\text{S}}}$ vs. $\mathbf{MSS^{\text{S}}}$} \vline &
    
\multicolumn{5}{c}{\bfseries\footnotesize $\mathbf{OEC^{\text{SN}}}$ vs. $\mathbf{MSS^{\text{SN}}}$}  \\
  &
  \multicolumn{5}{c}{\bfseries\footnotesize $\sigma^2_{\theta}$} \vline  &
   \multicolumn{5}{c}{\bfseries\footnotesize $\sigma^2_{\theta}$}  \\
$K$ & 0.25 & 0.5 & 1.0 & 2.0 & 4.0 & 0.25 & 0.5 & 1.0 & 2.0 & 4.0\\
\midrule
\midrule
2 & 0.97 & 0.95 & 0.96 & 0.98 & 0.96 & 0.98 & 0.96 & 0.93 & 0.92 & 0.88\\
5 & 0.96 & 0.94 & 0.96 & 0.96 & 0.98 & 1.00 & 1.00 & 1.00 & 1.00 & 0.98\\
8 & 0.96 & 0.94 & 0.95 & 0.97 & 1.01 & 1.00 & 0.99 & 0.99 & 0.96 & 0.98\\
25 & 0.93 & 0.94 & 0.96 & 0.98 & 0.99 & 1.00 & 1.00 & 1.00 & 1.01 & 0.98\\
30 & 0.94 & 0.95 & 0.96 & 0.98 & 1.00 & 1.00 & 0.99 & 0.99 & 0.98 & 1.00\\
\bottomrule
\end{tabular}
\end{minipage}
\caption{Average data-driven simulation performance: $RMSE_{OEC} / RMSE_{MSS}$. Columns indicate $\sigma^2_{\theta}$ and rows indicate $K$, number of studies. Monte carlo error was at most 0.003.}
\label{tab:mortSims_Ridge}
\end{table}

\begin{table}[H]
  \centering
\begin{minipage}[t]{0.75\linewidth} \centering
\begin{tabular}{rr|rrrr}
\toprule
\toprule
\multicolumn{2}{c}{\footnotesize\bfseries ~~~Parameters}  \vline  & 
     \multicolumn{1}{c}{\footnotesize\bfseries $\mathbf{OEC^{\text{S}}}$} 
     &
     \multicolumn{1}{c}{\footnotesize\bfseries $\mathbf{MSS^{\text{S}}}$} &
     \multicolumn{1}{c}{\footnotesize\bfseries $\mathbf{OEC^{\text{SN}}}$} 
      &
      \multicolumn{1}{c}{\footnotesize\bfseries $\mathbf{MSS^{\text{SN}}}$ } \\
  \multicolumn{1}{c}{\footnotesize\bfseries $K$} &
   \multicolumn{1}{c}{\footnotesize\bfseries ~~~~~$\sigma^2_{\theta}$}  \vline &
\multicolumn{4}{c}{\footnotesize  ~~~~~vs. Study-Specific Model   ~~~~~} \\
   
\midrule
\midrule
2 & 0.25 & 0.91 & 0.96 & 0.82 & 0.85\\
2 & 0.50 & 0.96 & 0.98 & 0.89 & 0.97\\
2 & 1.00 & 0.96 & 0.98 & 0.96 & 1.05\\
2 & 2.00 & 0.94 & 0.97 & 0.92 & 1.04\\
2 & 4.00 & 0.96 & 0.98 & 1.06 & 1.24\\
\addlinespace
5 & 0.25 & 0.89 & 0.95 & 0.80 & 0.79\\
5 & 0.50 & 0.89 & 0.96 & 0.80 & 0.80\\
5 & 1.00 & 0.89 & 0.95 & 0.83 & 0.83\\
5 & 2.00 & 0.92 & 0.95 & 0.87 & 0.87\\
5 & 4.00 & 0.94 & 0.95 & 0.93 & 0.95\\
\addlinespace
8 & 0.25 & 0.86 & 0.94 & 0.80 & 0.80\\
8 & 0.50 & 0.86 & 0.93 & 0.77 & 0.77\\
8 & 1.00 & 0.88 & 0.94 & 0.81 & 0.82\\
8 & 2.00 & 0.87 & 0.93 & 0.83 & 0.87\\
8 & 4.00 & 0.90 & 0.92 & 0.89 & 0.93\\
\addlinespace
25 & 0.25 & 0.84 & 0.90 & 0.80 & 0.80\\
25 & 0.50 & 0.80 & 0.89 & 0.77 & 0.77\\
25 & 1.00 & 0.80 & 0.86 & 0.78 & 0.78\\
25 & 2.00 & 0.80 & 0.85 & 0.79 & 0.81\\
25 & 4.00 & 0.84 & 0.85 & 0.83 & 0.85\\
\addlinespace
30 & 0.25 & 0.83 & 0.92 & 0.81 & 0.81\\
30 & 0.50 & 0.85 & 0.91 & 0.83 & 0.83\\
30 & 1.00 & 0.85 & 0.90 & 0.84 & 0.84\\
30 & 2.00 & 0.86 & 0.88 & 0.85 & 0.88\\
30 & 4.00 & 0.90 & 0.89 & 0.89 & 0.91\\
\bottomrule
\end{tabular}
\end{minipage}
\caption{Simulation performance without regularization at different $K$ and varying degrees of $\sigma^2_{\theta}$. Performance of each method is relative to a country-specific model (i.e., a model fit only on the target study, $RMSE_{OEC^{\text{S}}} / RMSE_{SSM}$) Monte Carlo error was at most 0.007. 
}
\label{table:mortSims2_zero_full}
\end{table}

\begin{table}[H]
  \centering
\begin{minipage}[t]{0.75\linewidth} \centering
\begin{tabular}{rr|rrrr}
\toprule
\toprule
\multicolumn{2}{c}{\bfseries\footnotesize ~~~Parameters}  \vline  & 
     \multicolumn{1}{c}{\footnotesize\bfseries $\mathbf{OEC^{S}}$} 
     &
     \multicolumn{1}{c}{\footnotesize\bfseries $\mathbf{MSS^{\text{S}}}$} &
     \multicolumn{1}{c}{\footnotesize\bfseries $\mathbf{OEC^{\text{SN}}}$} 
      &
      \multicolumn{1}{c}{\footnotesize\bfseries $\mathbf{MSS^{\text{SN}}}$} \\
  \multicolumn{1}{c}{\footnotesize\bfseries $K$} &
   \multicolumn{1}{c}{\footnotesize\bfseries ~~~~~$\sigma^2_{\theta}$}  \vline &
\multicolumn{4}{c}{\footnotesize  ~~~~~vs. Study-Specific Model   ~~~~~} \\
   
\midrule
\midrule
2 & 0.25 & 0.91 & 0.94 & 0.81 & 0.84\\
2 & 0.50 & 0.91 & 0.96 & 0.79 & 0.83\\
2 & 1.00 & 0.94 & 0.98 & 0.90 & 0.99\\
2 & 2.00 & 0.98 & 0.99 & 1.03 & 1.13\\
2 & 4.00 & 0.96 & 1.00 & 1.02 & 1.23\\
\addlinespace
5 & 0.25 & 0.91 & 0.95 & 0.80 & 0.80\\
5 & 0.50 & 0.90 & 0.96 & 0.82 & 0.82\\
5 & 1.00 & 0.91 & 0.96 & 0.84 & 0.84\\
5 & 2.00 & 0.92 & 0.95 & 0.88 & 0.88\\
5 & 4.00 & 0.94 & 0.97 & 0.93 & 0.96\\
\addlinespace
8 & 0.25 & 0.85 & 0.90 & 0.77 & 0.77\\
8 & 0.50 & 0.85 & 0.91 & 0.79 & 0.80\\
8 & 1.00 & 0.86 & 0.92 & 0.82 & 0.83\\
8 & 2.00 & 0.90 & 0.93 & 0.86 & 0.91\\
8 & 4.00 & 0.95 & 0.94 & 0.93 & 0.98\\
\addlinespace
25 & 0.25 & 0.81 & 0.88 & 0.79 & 0.80\\
25 & 0.50 & 0.82 & 0.88 & 0.81 & 0.81\\
25 & 1.00 & 0.84 & 0.88 & 0.83 & 0.83\\
25 & 2.00 & 0.87 & 0.89 & 0.87 & 0.86\\
25 & 4.00 & 0.89 & 0.89 & 0.89 & 0.92\\
\addlinespace
30 & 0.25 & 0.80 & 0.86 & 0.79 & 0.79\\
30 & 0.50 & 0.81 & 0.87 & 0.80 & 0.81\\
30 & 1.00 & 0.83 & 0.86 & 0.83 & 0.84\\
30 & 2.00 & 0.85 & 0.87 & 0.85 & 0.87\\
30 & 4.00 & 0.90 & 0.89 & 0.90 & 0.90\\
\bottomrule
\end{tabular}
\end{minipage}
\caption{Simulation performance with regularization at different $K$ and varying degrees of $\sigma^2_{\theta}$. Performance of each method is relative to a country-specific model (i.e., a model fit only on the target study): $RMSE_{OEC^{\text{S}}} / RMSE_{SSM}$. Monte Carlo error was at most 0.007. 
}
\label{table:mortSims2_ridge_full}
\end{table}

\section{Supplementary Material: General Simulations}\label{sims23_supplement}


\begin{figure}[H]
	\centering
	\begin{subfigure}[t]{1\textwidth}
		\centering
		\includegraphics[width=1.0\linewidth]{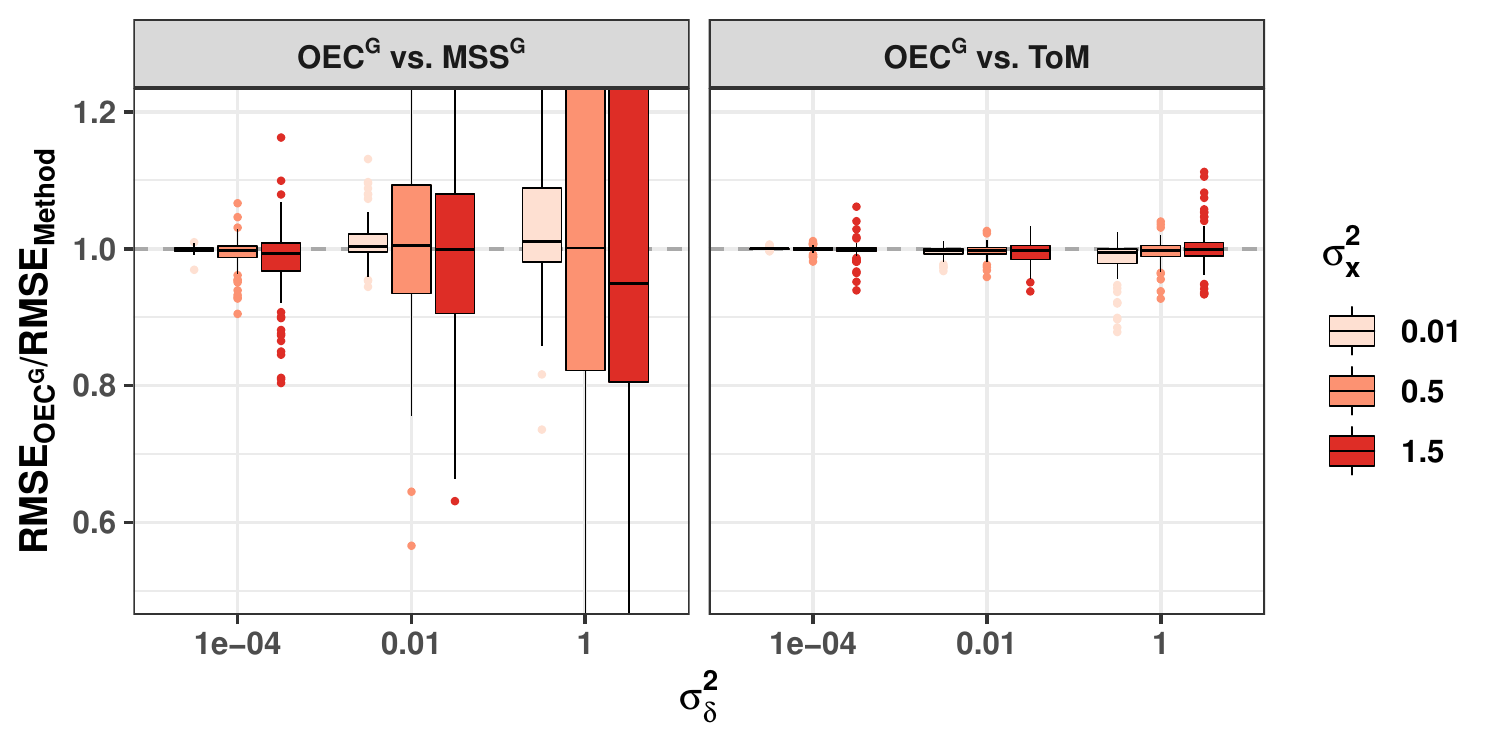}
		\caption{No clusters}
	\end{subfigure}
	\hfill
	\begin{subfigure}[t]{1\textwidth}
		\centering
		\includegraphics[width=1.0\linewidth]{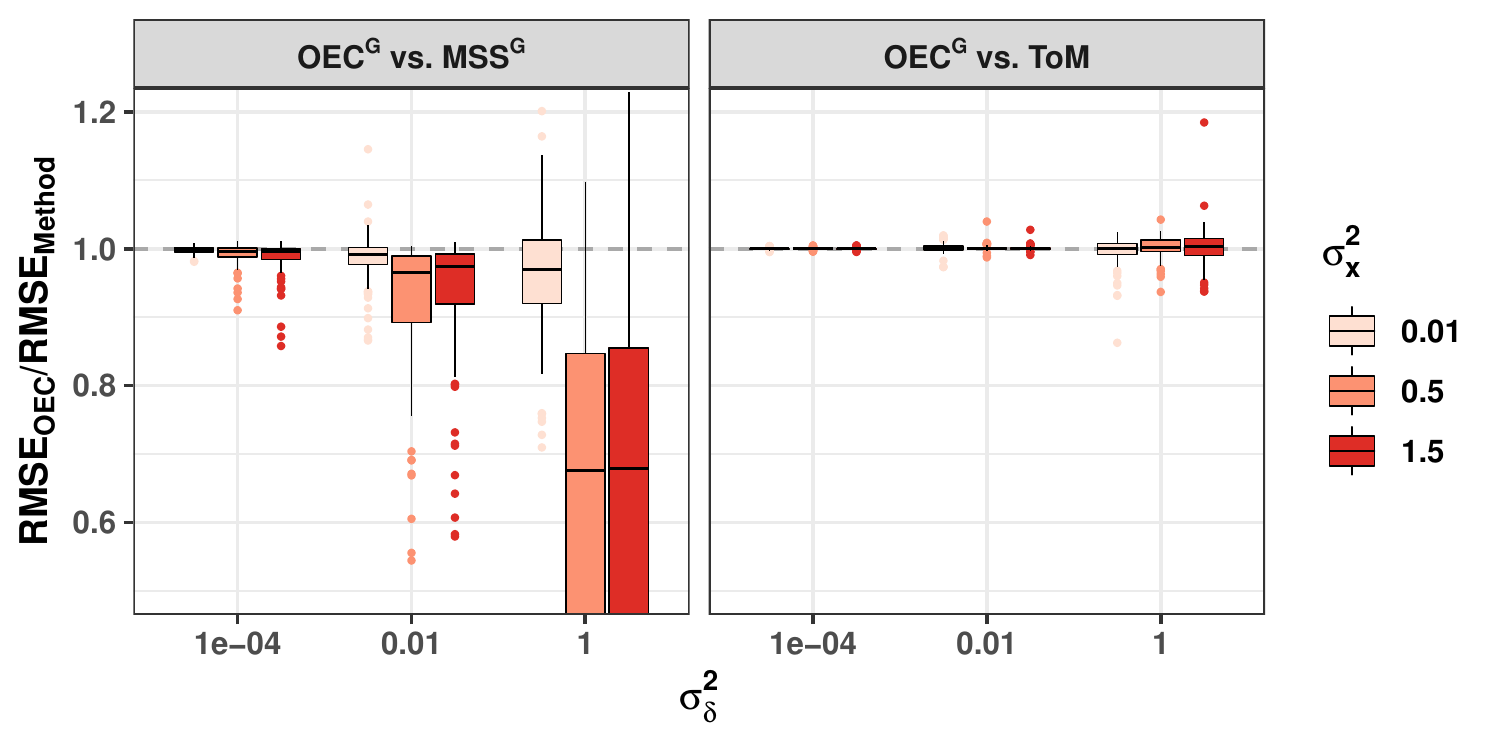}
		\caption{Clusters}
	\end{subfigure}
	\caption{Performance of the OEC$^{\text{G}}$ compared to the MSS$^{\text{G}}$ and the ToM algorithm using no study-specific Ridge penalties. Figures are zoomed in to easily visualize differences.}
	\label{fig:sims23_generalist_zero}
\end{figure}

\begin{figure}[H]
	\centering
	\begin{subfigure}[t]{1\textwidth}
		\centering
		\includegraphics[width=1.0\linewidth]{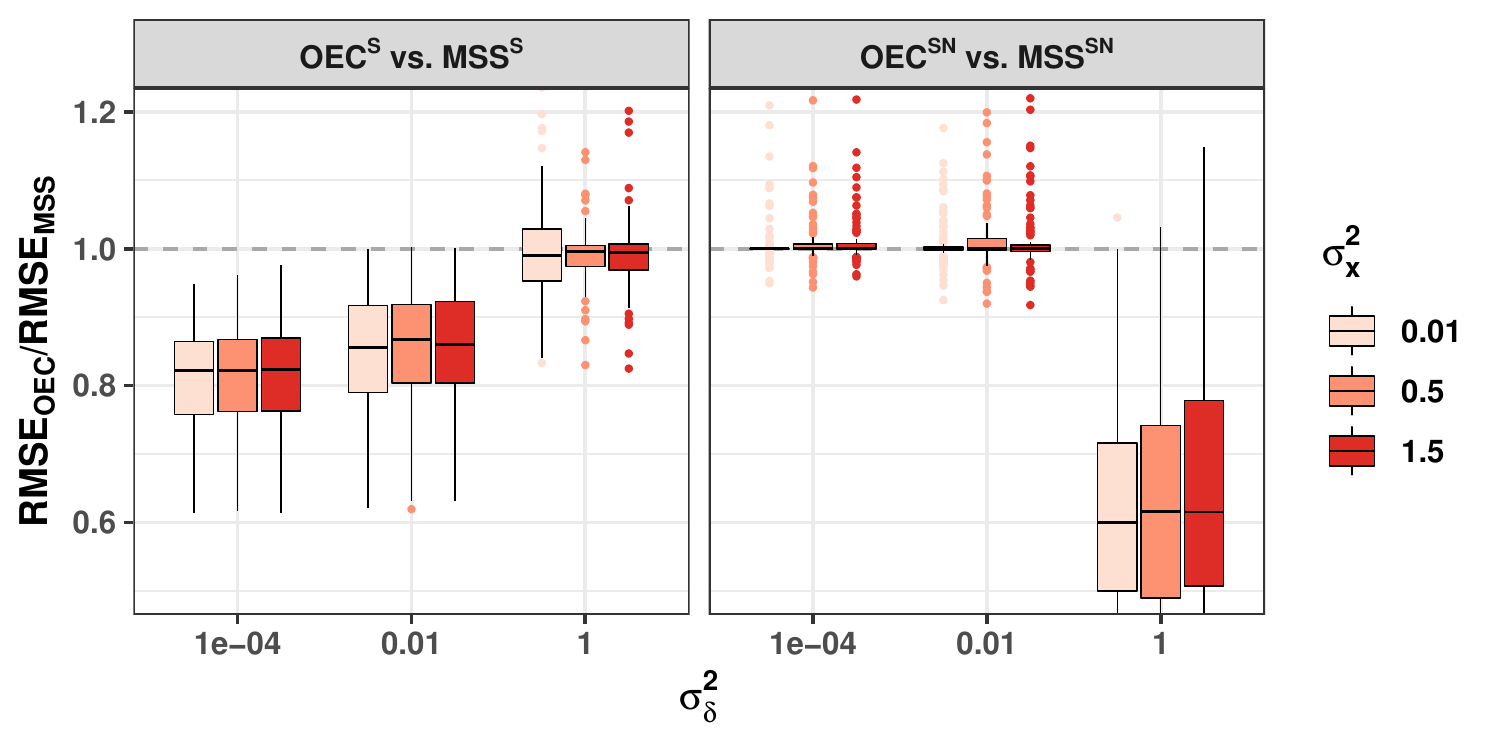}
		\caption{No clusters}
	\end{subfigure}
	\hfill
	\begin{subfigure}[t]{1\textwidth}
		\centering
		\includegraphics[width=1.0\linewidth]{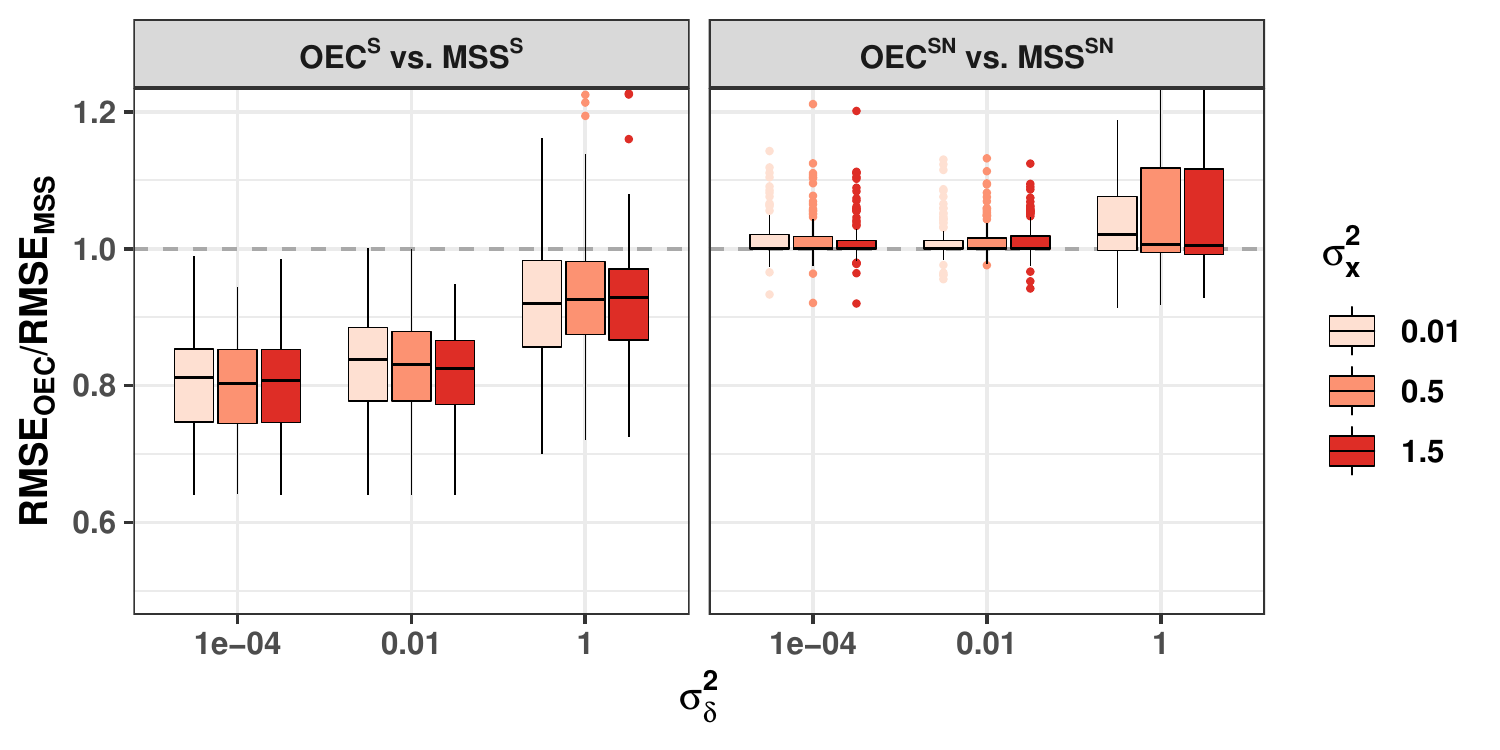}
		\caption{Clusters}
	\end{subfigure}
	\caption{Performance of the OEC$^{\text{S}}$ and OEC$^{\text{SN}}$ compared to their MSS counterparts using no study-specific Ridge penalties. Figures are zoomed in to easily visualize differences.}
	\label{fig:sims23_specialist_zero}
\end{figure}
\begin{figure}[H]
	\centering
	\begin{subfigure}[t]{1\textwidth}
		\centering
		\includegraphics[width=1.0\linewidth]{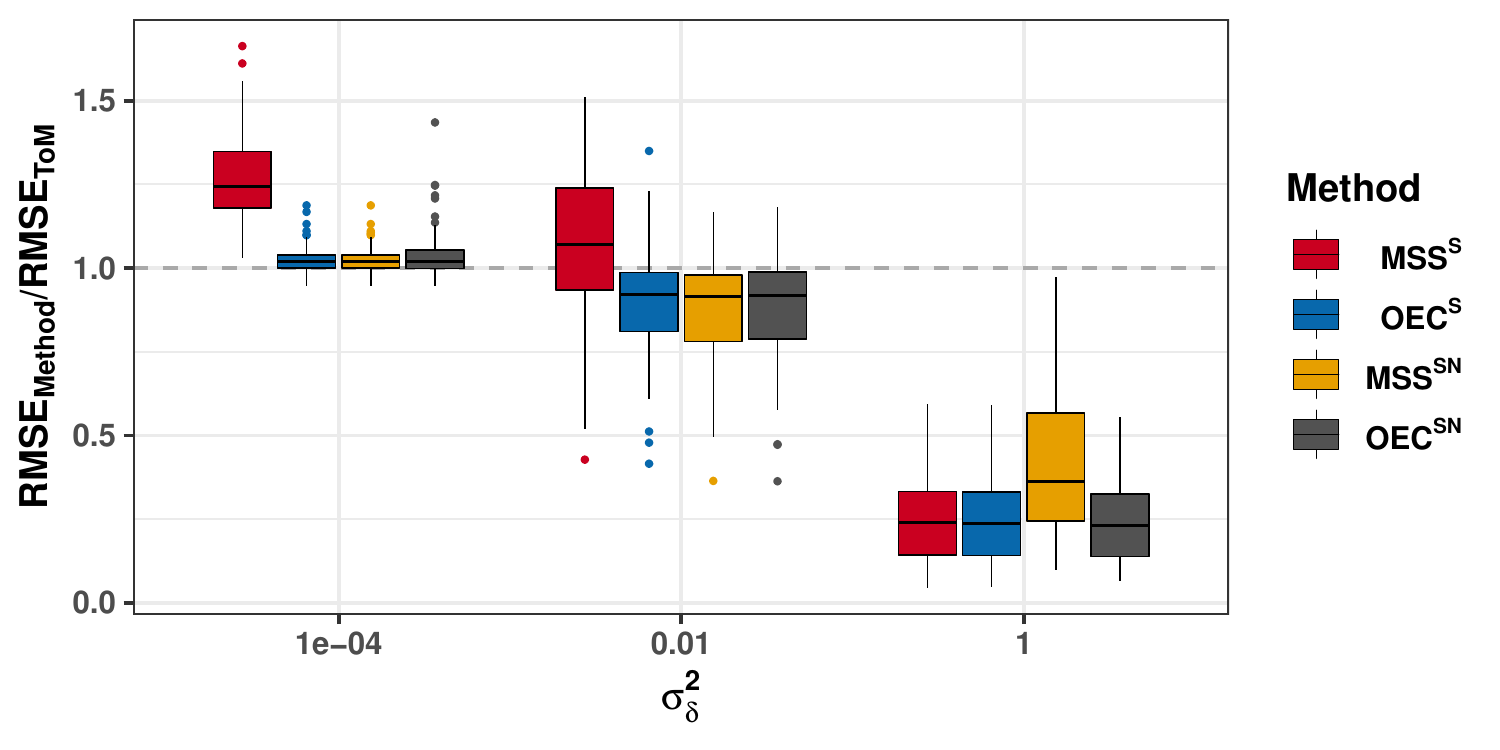}
		\caption{No clusters}
	\end{subfigure}
	\hfill
	\begin{subfigure}[t]{1\textwidth}
		\centering
		\includegraphics[width=1.0\linewidth]{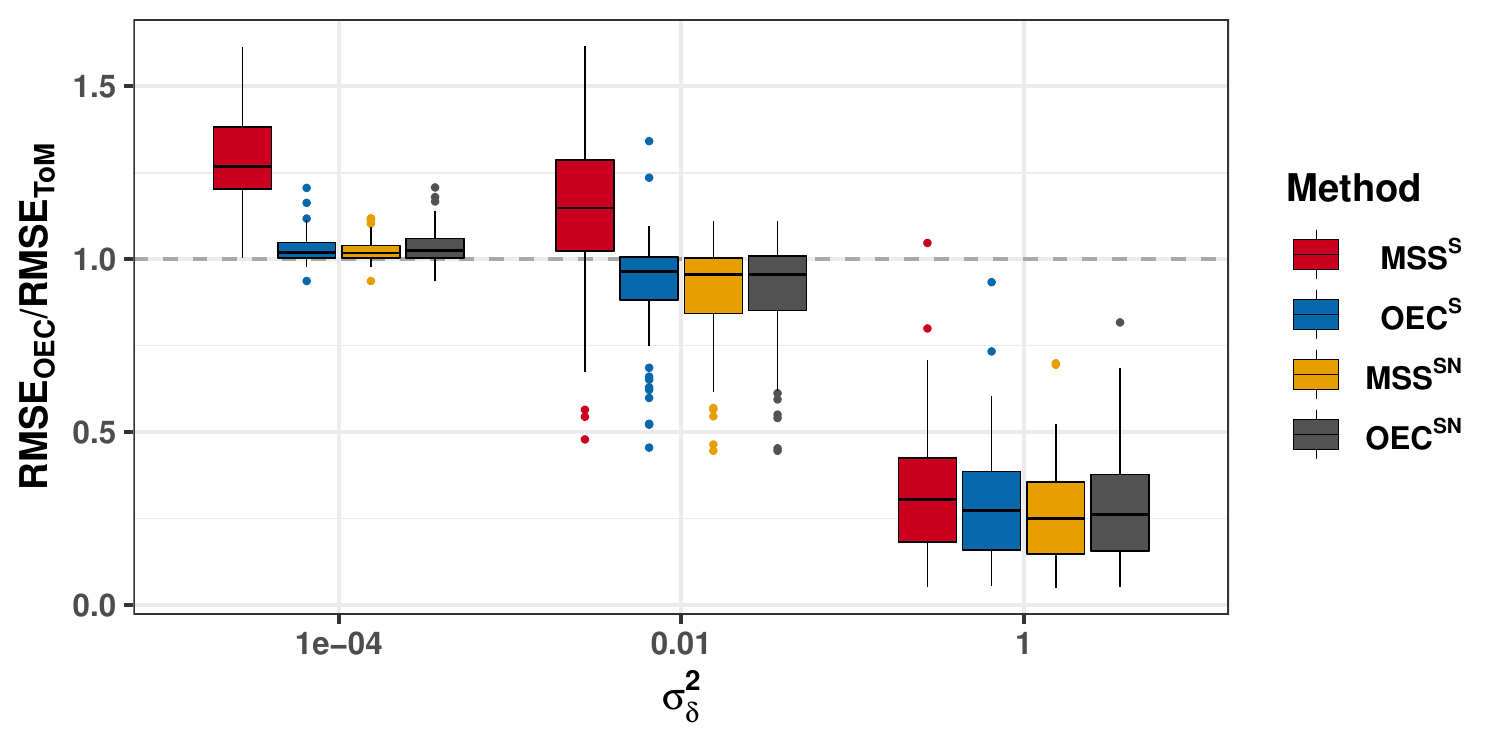}
		\caption{Clusters}
	\end{subfigure}
	\caption{Performance of all methods together ($\sigma^2_X = 0.01$) and the ToM algorithm without study-specific Ridge penalties. Figures are zoomed in to easily visualize differences.}
	\label{fig:sims23_generalistTogether_zero}
\end{figure}

\begin{figure}[H]
	\centering
	\begin{subfigure}[t]{1\textwidth}
		\centering
		\includegraphics[width=1.0\linewidth]{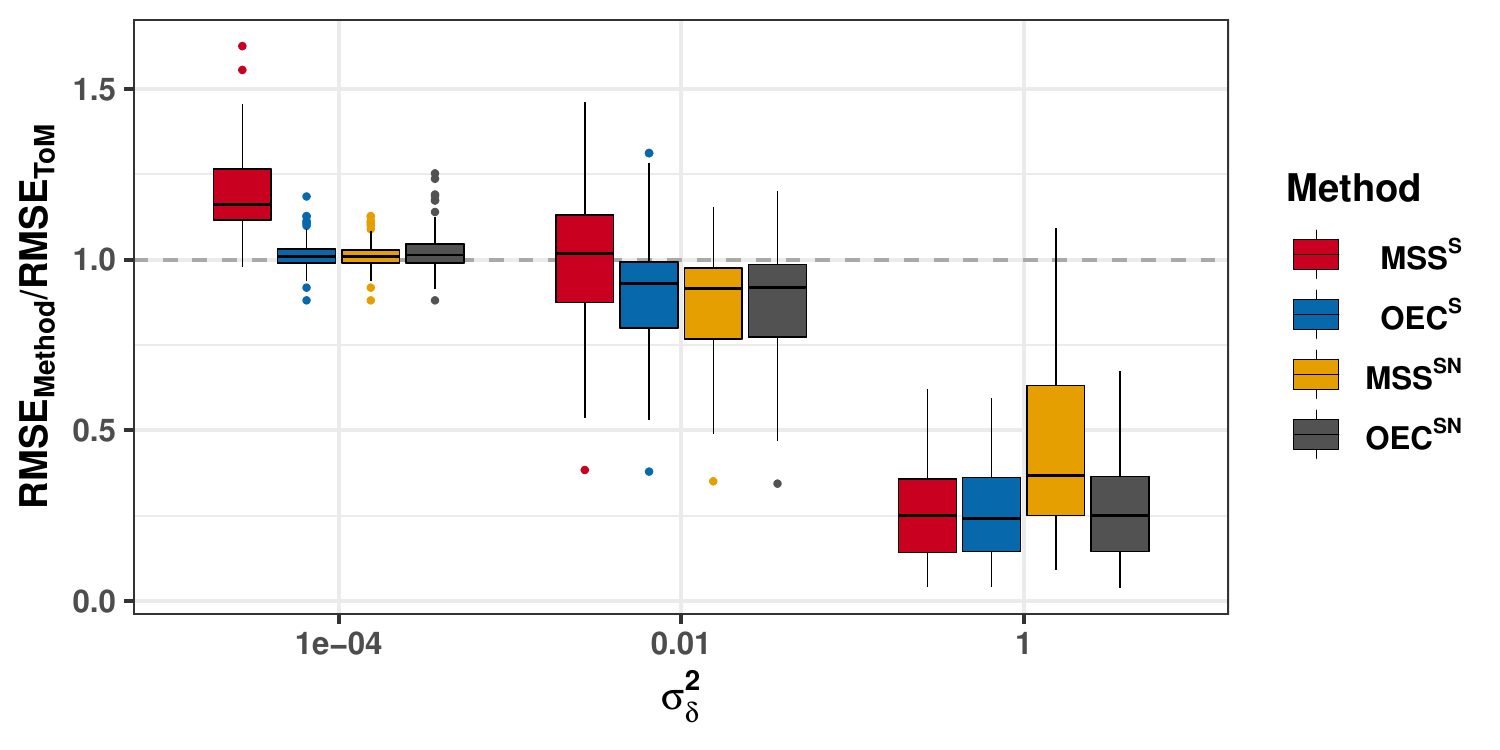}
		\caption{No clusters}
	\end{subfigure}
	\hfill
	\begin{subfigure}[t]{1\textwidth}
		\centering
		\includegraphics[width=1.0\linewidth]{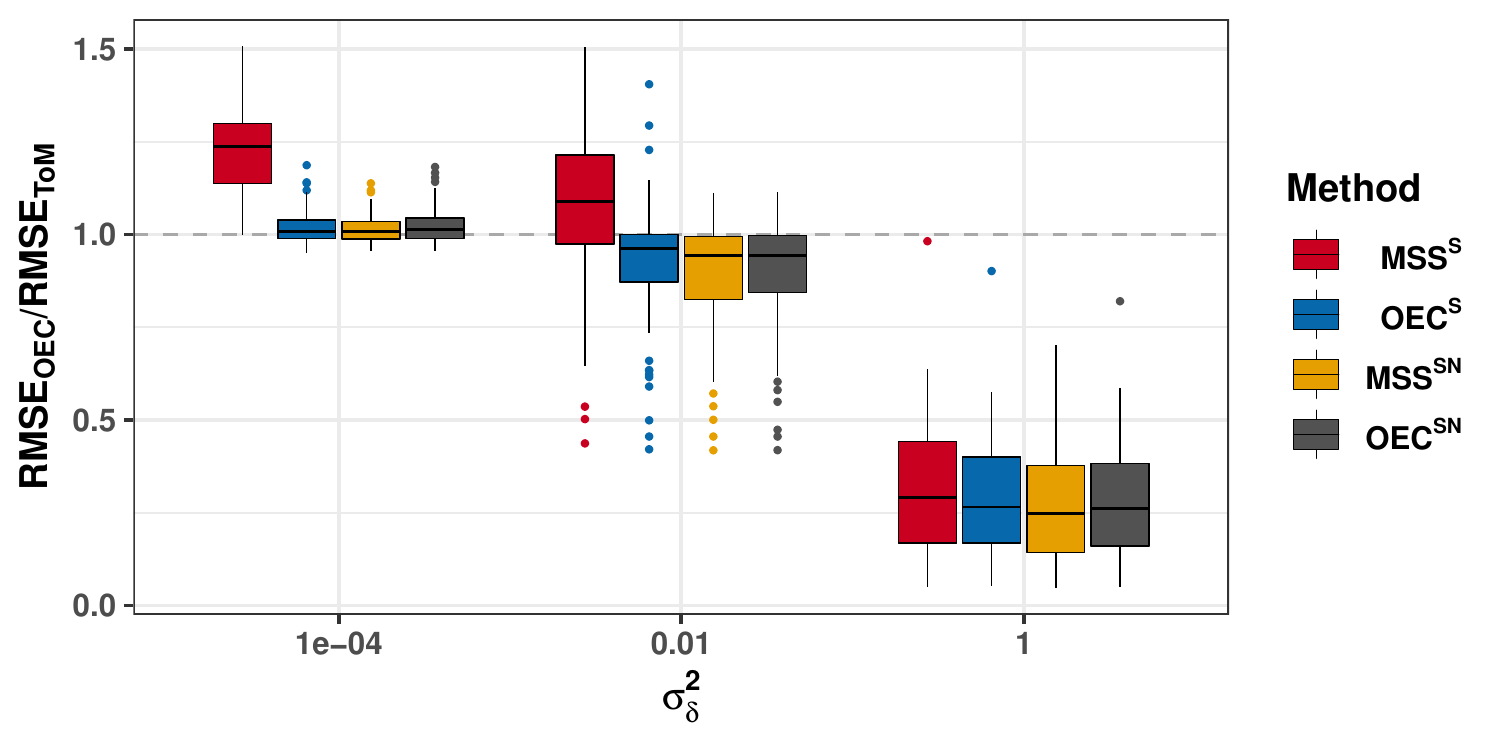}
		\caption{Clusters}
	\end{subfigure}
	\caption{Performance of all methods together ($\sigma^2_X = 0.01$) and the ToM algorithm with study-specific Ridge penalties. Figures are zoomed in to easily visualize differences.}
	\label{fig:sims23_generalistTogether_cvCF}
\end{figure}
\FloatBarrier
\begin{table}[H]
   \centering
\begin{minipage}[t]{1\linewidth}\centering
\begin{tabular}{rrr|rr|rr|r}
\toprule
\toprule
\multicolumn{3}{c}{\footnotesize\bfseries Parameters} \vline &

    \multicolumn{1}{c}{\footnotesize\bfseries OEC$^{\text{G}}$} &
     \multicolumn{1}{c}{\footnotesize\bfseries ToM} 
\vline &
     \multicolumn{1}{c}{\footnotesize\bfseries OEC$^{\text{S}}$} &      \multicolumn{1}{c}{\footnotesize\bfseries
     SSM} 
     \vline &
\multicolumn{1}{c}{\footnotesize\bfseries  OEC$^{\text{SN}}$} \\
 \multicolumn{1}{c}{ $C$} &
  \multicolumn{1}{c}{\footnotesize\bfseries $\sigma^2_X$} &
   \multicolumn{1}{c}{\footnotesize\bfseries $\sigma^2_{\delta}$} \vline  &
    \multicolumn{2}{c}{\footnotesize vs. MSS$^{\text{G}}$} \vline &  \multicolumn{2}{c}{\footnotesize vs. MSS$^{\text{S}}$} \vline  &

      \multicolumn{1}{c}{\footnotesize vs. MSS$^{\text{SN}}$ } \\
\midrule
\midrule
3 & 0.01 & 0.00 & 1.00 & 1.00 & 0.80 & 1.00 & 1.01\\
3 & 0.50 & 0.00 & 0.99 & 0.99 & 0.80 & 1.00 & 1.01\\
3 & 1.50 & 0.00 & 0.99 & 0.99 & 0.80 & 1.00 & 1.01\\
3 & 0.01 & 0.01 & 0.99 & 0.99 & 0.83 & 1.00 & 1.01\\
3 & 0.50 & 0.01 & 0.92 & 0.92 & 0.82 & 1.00 & 1.01\\
3 & 1.50 & 0.01 & 0.93 & 0.93 & 0.82 & 1.00 & 1.01\\
3 & 0.01 & 1.00 & 1.02 & 1.03 & 0.91 & 1.00 & 1.05\\
3 & 0.50 & 1.00 & 0.67 & 0.67 & 0.94 & 1.00 & 1.08\\
3 & 1.50 & 1.00 & 0.66 & 0.65 & 0.98 & 1.00 & 1.15\\
\addlinespace
6 & 0.01 & 0.00 & 1.00 & 1.00 & 0.81 & 1.00 & 1.01\\
6 & 0.50 & 0.00 & 0.99 & 0.99 & 0.81 & 1.00 & 1.01\\
6 & 1.50 & 0.00 & 0.98 & 0.98 & 0.82 & 1.00 & 1.01\\
6 & 0.01 & 0.01 & 1.01 & 1.01 & 0.85 & 1.00 & 1.01\\
6 & 0.50 & 0.01 & 1.02 & 1.03 & 0.86 & 1.00 & 1.01\\
6 & 1.50 & 0.01 & 1.02 & 1.02 & 0.85 & 1.00 & 1.01\\
6 & 0.01 & 1.00 & 1.09 & 1.11 & 1.00 & 1.00 & 0.62\\
6 & 0.50 & 1.00 & 1.12 & 1.15 & 1.01 & 1.00 & 0.63\\
6 & 1.50 & 1.00 & 1.11 & 1.13 & 1.06 & 1.00 & 0.66\\
\bottomrule
\end{tabular}
\end{minipage}
\caption{Simulation performance with ($C = 3$) and without ($C = 6$)  clustering at varying degrees of covariate-shift ($\sigma^2_{X}$) and variance of random effects ($\sigma^2_{\delta}$). No Ridge penalty was included. Each column indicates the performance of a method relative to a corresponding multi-study stacking method (e.g., OEC$^{\text{G}}$ vs. MSS$^{\text{G}}$ indicates $RMSE_{OEC^{\text{G}}} / RMSE_{MSS^{\text{G}}}$. Monte carlo error was at most 0.0024.}  
\label{table:generalSims_full_zero}
\end{table}
\begin{table}[H]
  \centering
\begin{minipage}[t]{1\linewidth} \centering
\begin{tabular}{rrr|rr|rrrr}
\toprule
\toprule
\multicolumn{3}{c}{\footnotesize\bfseries ~~~Parameters} \vline &
\multicolumn{1}{c}{\footnotesize\bfseries $\mathbf{OEC^{\text{G}}}$} &
    \multicolumn{1}{c}{\footnotesize\bfseries $\mathbf{MSS^{\text{G}}}$} \vline   &
     \multicolumn{1}{c}{\footnotesize\bfseries $\mathbf{OEC^{\text{S}}}$} 
     &
     \multicolumn{1}{c}{\footnotesize\bfseries $\mathbf{MSS^{\text{S}}}$} &
     \multicolumn{1}{c}{\footnotesize\bfseries
     $\mathbf{OEC^{\text{SN}}}$} 
      &
      \multicolumn{1}{c}{\footnotesize\bfseries $\mathbf{MSS^{\text{SN}}}$ } \\
 \multicolumn{1}{c}{ $C$} &
  \multicolumn{1}{c}{\footnotesize\bfseries $\sigma^2_X$} &
   \multicolumn{1}{c}{\footnotesize\bfseries $\sigma^2_{\delta}$}  \vline &
      \multicolumn{2}{c}{\footnotesize vs. ToM~~~~~}
\vline &
\multicolumn{4}{c}{\footnotesize  ~~~~~vs. Study-Specific Model   ~~~~~} \\

\midrule
\midrule
3 & 0.01 & 0.00 & 1.00 & 1.00 & 0.80 & 1.00 & 0.81 & 0.80\\
3 & 0.50 & 0.00 & 1.00 & 1.01 & 0.80 & 1.00 & 0.81 & 0.80\\
3 & 1.50 & 0.00 & 1.00 & 1.01 & 0.80 & 1.00 & 0.81 & 0.80\\
3 & 0.01 & 0.01 & 1.00 & 1.01 & 0.83 & 1.00 & 0.82 & 0.81\\
3 & 0.50 & 0.01 & 1.00 & 1.10 & 0.82 & 1.00 & 0.82 & 0.81\\
3 & 1.50 & 0.01 & 1.00 & 1.09 & 0.82 & 1.00 & 0.82 & 0.81\\
3 & 0.01 & 1.00 & 1.00 & 1.02 & 0.91 & 1.00 & 0.87 & 0.83\\
3 & 0.50 & 1.00 & 1.00 & 1.76 & 0.94 & 1.00 & 0.90 & 0.83\\
3 & 1.50 & 1.00 & 1.02 & 1.89 & 0.98 & 1.00 & 0.96 & 0.83\\
\addlinespace
6 & 0.01 & 0.00 & 1.00 & 1.00 & 0.81 & 1.00 & 0.82 & 0.81\\
6 & 0.50 & 0.00 & 1.00 & 1.01 & 0.81 & 1.00 & 0.82 & 0.81\\
6 & 1.50 & 0.00 & 1.00 & 1.02 & 0.82 & 1.00 & 0.82 & 0.81\\
6 & 0.01 & 0.01 & 1.00 & 0.99 & 0.85 & 1.00 & 0.84 & 0.83\\
6 & 0.50 & 0.01 & 1.00 & 1.00 & 0.86 & 1.00 & 0.84 & 0.83\\
6 & 1.50 & 0.01 & 0.99 & 1.02 & 0.85 & 1.00 & 0.84 & 0.83\\
6 & 0.01 & 1.00 & 0.98 & 0.94 & 1.00 & 1.00 & 1.00 & 1.72\\
6 & 0.50 & 1.00 & 0.99 & 1.06 & 1.01 & 1.00 & 1.01 & 1.72\\
6 & 1.50 & 1.00 & 1.00 & 1.12 & 1.06 & 1.00 & 1.07 & 1.72\\
\bottomrule
\end{tabular}
\end{minipage}
\caption{Simulation performance with ($C = 3$) and without ($C = 6$)  clustering at varying degrees of covariate-shift ($\sigma^2_{X}$) and variance of random effects ($\sigma^2_{\delta}$). No Ridge penalty was included. Each section indicates the performance of the method (e.g., OEC$^{\text{G}}$) relative to a baseline of the ToM algorithm or a study-specific model (e.g., $RMSE_{OEC^{\text{G}}} / RMSE_{ToM}$). Monte carlo error was at most 0.0033. 
}
\label{table:generalSims_v2_supplement_zero}
\end{table}
\FloatBarrier
\begin{table}[H]
   \centering
\begin{minipage}[t]{1\linewidth}\centering
\begin{tabular}{rrr|rr|rr|r}
\toprule
\toprule
\multicolumn{3}{c}{\footnotesize\bfseries Parameters} \vline &

    \multicolumn{1}{c}{\footnotesize\bfseries OEC$^{\text{G}}$} &
     \multicolumn{1}{c}{\footnotesize\bfseries ToM} 
\vline &
     \multicolumn{1}{c}{\footnotesize\bfseries OEC$^{\text{S}}$} &      \multicolumn{1}{c}{\footnotesize\bfseries
     SSM} 
     \vline &
\multicolumn{1}{c}{\footnotesize\bfseries  OEC$^{\text{SN}}$} \\
 \multicolumn{1}{c}{ $C$} &
  \multicolumn{1}{c}{\footnotesize\bfseries $\sigma^2_X$} &
   \multicolumn{1}{c}{\footnotesize\bfseries $\sigma^2_{\delta}$} \vline  &
    \multicolumn{2}{c}{\footnotesize vs. MSS$^{\text{G}}$} \vline &  \multicolumn{2}{c}{\footnotesize vs. MSS$^{\text{S}}$} \vline  &

      \multicolumn{1}{c}{\footnotesize vs. MSS$^{\text{SN}}$ } \\
\midrule
\midrule
3 & 0.01 & 0.00 & 1.00 & 1.01 & 0.84 & 1.01 & 1.01\\
3 & 0.50 & 0.00 & 0.99 & 1.00 & 0.84 & 1.01 & 1.01\\
3 & 1.50 & 0.00 & 0.98 & 1.02 & 0.85 & 1.01 & 1.02\\
3 & 0.01 & 0.01 & 0.99 & 1.00 & 0.88 & 1.01 & 1.01\\
3 & 0.50 & 0.01 & 0.91 & 0.93 & 0.87 & 1.01 & 1.01\\
3 & 1.50 & 0.01 & 0.92 & 0.95 & 0.86 & 1.01 & 1.02\\
3 & 0.01 & 1.00 & 1.03 & 1.04 & 0.95 & 1.01 & 1.05\\
3 & 0.50 & 1.00 & 0.67 & 0.69 & 0.97 & 1.01 & 1.07\\
3 & 1.50 & 1.00 & 0.66 & 0.66 & 1.02 & 1.01 & 1.16\\
\addlinespace
6 & 0.01 & 0.00 & 1.00 & 1.01 & 0.86 & 1.03 & 1.01\\
6 & 0.50 & 0.00 & 0.99 & 1.03 & 0.86 & 1.03 & 1.01\\
6 & 1.50 & 0.00 & 0.97 & 1.15 & 0.86 & 1.03 & 1.01\\
6 & 0.01 & 0.01 & 1.01 & 1.01 & 0.91 & 1.02 & 1.01\\
6 & 0.50 & 0.01 & 1.02 & 1.02 & 0.91 & 1.02 & 1.01\\
6 & 1.50 & 0.01 & 1.00 & 1.07 & 0.91 & 1.02 & 1.02\\
6 & 0.01 & 1.00 & 1.09 & 1.03 & 1.02 & 1.01 & 0.63\\
6 & 0.50 & 1.00 & 1.14 & 1.02 & 1.02 & 1.01 & 0.63\\
6 & 1.50 & 1.00 & 1.16 & 1.05 & 1.03 & 1.01 & 0.65\\
\bottomrule
\end{tabular}
\end{minipage}
\caption{Simulation performance with ($C = 3$) and without ($C = 6$)  clustering at varying degrees of covariate-shift ($\sigma^2_{X}$) and variance of random effects ($\sigma^2_{\delta}$). A Ridge penalty was included. Each column indicates the performance of a method relative to a corresponding multi-study stacking method (e.g., OEC$^{\text{G}}$ vs. MSS$^{\text{G}}$ indicates $RMSE_{OEC^{\text{G}}} / RMSE_{MSS^{\text{G}}}$. Monte carlo error was at most 0.003.}  
\label{table:generalSims_full_cvCF}
\end{table}
\begin{table}[H]
  \centering
\begin{minipage}[t]{1\linewidth} \centering
\begin{tabular}{rrr|rr|rrrr}
\toprule
\toprule
\multicolumn{3}{c}{\footnotesize\bfseries ~~~Parameters} \vline &
\multicolumn{1}{c}{\footnotesize\bfseries $\mathbf{OEC^{G}}$} &
    \multicolumn{1}{c}{\footnotesize\bfseries $\mathbf{Generalist}$} \vline   &
     \multicolumn{1}{c}{\footnotesize\bfseries $\mathbf{OEC^{S}}$} 
     &
     \multicolumn{1}{c}{\footnotesize\bfseries $\mathbf{Specialist}$} &
     \multicolumn{1}{c}{\footnotesize\bfseries
     $\mathbf{OEC^{S-NDR}}$} 
      &
      \multicolumn{1}{c}{\footnotesize\bfseries $\mathbf{Zero~Out}$ } \\
 \multicolumn{1}{c}{\footnotesize $C$} &
  \multicolumn{1}{c}{\footnotesize\bfseries $\sigma^2_X$} &
   \multicolumn{1}{c}{\footnotesize\bfseries $\sigma^2_{\beta}$}  \vline &
     \multicolumn{2}{c}{\footnotesize vs. ToM~~~~~}
\vline &
\multicolumn{4}{c}{\footnotesize  ~~~~~vs. Study-Specific Model   ~~~~~} \\
   
\midrule
\midrule
3 & 0.01 & 0.00 & 0.99 & 0.99 & 0.83 & 0.99 & 0.83 & 0.82\\
3 & 0.50 & 0.00 & 0.98 & 1.00 & 0.83 & 0.99 & 0.84 & 0.82\\
3 & 1.50 & 0.00 & 0.96 & 0.99 & 0.84 & 0.99 & 0.84 & 0.82\\
3 & 0.01 & 0.01 & 0.99 & 1.00 & 0.87 & 0.99 & 0.84 & 0.84\\
3 & 0.50 & 0.01 & 0.98 & 1.10 & 0.86 & 0.99 & 0.84 & 0.84\\
3 & 1.50 & 0.01 & 0.96 & 1.07 & 0.85 & 0.99 & 0.85 & 0.83\\
3 & 0.01 & 1.00 & 0.98 & 0.99 & 0.94 & 0.99 & 0.89 & 0.85\\
3 & 0.50 & 1.00 & 0.97 & 1.69 & 0.96 & 0.99 & 0.92 & 0.85\\
3 & 1.50 & 1.00 & 1.02 & 1.85 & 1.01 & 0.99 & 0.99 & 0.85\\
\addlinespace
6 & 0.01 & 0.00 & 0.99 & 0.99 & 0.84 & 0.98 & 0.84 & 0.83\\
6 & 0.50 & 0.00 & 0.96 & 0.97 & 0.84 & 0.98 & 0.84 & 0.83\\
6 & 1.50 & 0.00 & 0.87 & 0.91 & 0.84 & 0.98 & 0.84 & 0.83\\
6 & 0.01 & 0.01 & 1.00 & 0.99 & 0.89 & 0.98 & 0.86 & 0.85\\
6 & 0.50 & 0.01 & 1.01 & 1.00 & 0.89 & 0.98 & 0.86 & 0.85\\
6 & 1.50 & 0.01 & 0.97 & 1.00 & 0.89 & 0.98 & 0.86 & 0.85\\
6 & 0.01 & 1.00 & 1.06 & 0.99 & 1.01 & 0.99 & 1.03 & 1.75\\
6 & 0.50 & 1.00 & 1.13 & 1.10 & 1.01 & 0.99 & 1.03 & 1.75\\
6 & 1.50 & 1.00 & 1.12 & 1.15 & 1.03 & 0.99 & 1.06 & 1.74\\
\bottomrule
\end{tabular}
\end{minipage}
\label{table:g}
\caption{Simulation performance with ($C = 3$) and without ($C = 6$)  clustering at varying degrees of covariate-shift ($\sigma^2_{X}$) and variance of random effects ($\sigma^2_{\delta}$). A Ridge penalty was included. Each section indicates the performance of the method (e.g., OEC$^{\text{G}}$) relative to a baseline of the ToM algorithm or a study-specific model (e.g., $RMSE_{OEC^{\text{G}}} / RMSE_{ToM}$). Monte carlo error was at most 0.0031. 
}
\label{table:generalSims_v2_supplement_cvCF}
\end{table}

\section{Supplementary Material: Proofs}\label{proofs}
The proofs of Propositions~\ref{eta1_minus_prop} and \ref{eta0_prop} rely on the following lemma. 
\begin{lemma}\label{lemma1}
Consider the following optimization problem
$$\min_{\boldsymbol{\theta}} H_\delta(\boldsymbol{\theta}) := h_{1}(\boldsymbol{\theta}) + \delta h_{2}(\boldsymbol{\theta})$$
where, $h_{1}(\boldsymbol{\theta})$ and $h_{2}(\boldsymbol{\theta})$ are both non-negative functions on $\mathbb{R}^n$ and $\delta \geq 0$ is a tuning parameter. Assume that 
$\min_{\boldsymbol{\theta}} h_{i}(\boldsymbol{\theta}) < \infty$ exists (i.e., the minimum is attained) for each $i \in \{1, 2\}.$
Let $\hat{\boldsymbol{\theta}}_{\delta} \in \arg\min_{\boldsymbol{\theta}} H_{\delta}(\boldsymbol{\theta})$ for $\delta \geq 0.$
Let us define
$$ \boldsymbol{\theta}^0 \in \arg\min_{\boldsymbol{\theta}}~ h_{2}(\boldsymbol{\theta})~~~\text{s.t.}~~h_{1}(\boldsymbol{\theta}) = \min_{\boldsymbol{v}} h_{1}(\boldsymbol{v}).$$
Assume that $\boldsymbol{\theta}^0$ is bounded and $h_{2}(\boldsymbol{\theta}^0) < \infty$.
Then, as $\delta \rightarrow 0+$
$$\hat{\boldsymbol{\theta}}_{\delta} \rightarrow \arg\min_{\boldsymbol{\theta}} h_{2}(\boldsymbol{\theta}) ~~\text{s.t.} ~~ h_{1}(\boldsymbol{\theta}) = \min_{\boldsymbol{v}} h_{1}(\boldsymbol{v}).$$
\end{lemma}
\begin{proof} Note the following chain of inequalities:
\begin{align*}
\min_{\boldsymbol{v}}h_{1}(\boldsymbol{v}) 
\leq \min_{\boldsymbol{\theta}} \{ h_{1}(\boldsymbol{\theta}) + \delta h_{2}(\boldsymbol{\theta})\}
\leq h_{1}(\boldsymbol{\theta}^0) + \delta h_{2}(\boldsymbol{\theta}^0). 
\end{align*}
If we let $\delta \rightarrow 0+$, then, the left and right hand sides converge to the same limit. Therefore, the middle expression has the same limit. This completes the proof of the lemma.
\end{proof}


\begin{proof}[Proof of Proposition~1] 
The proof follows by taking $\delta=(1-\eta)/\eta$.  
\end{proof}

\begin{proof}[Proof of Proposition~2]
The proof follows by taking $\delta = \eta / (1 - \eta)$. 
\end{proof}

\clearpage 

\end{document}